\documentclass[11pt,leqno]{amsart}
\usepackage[utf8]{inputenc}
\usepackage[T1]{fontenc}
\usepackage[margin=1in]{geometry}
\usepackage{graphicx}
\usepackage[numbers,sort&compress]{natbib}
\usepackage{indentfirst,csquotes}
\usepackage{amssymb,amsthm,amsmath,mathtools}
\usepackage{xcolor,paralist,fancyhdr,etoolbox}
\usepackage{subcaption,booktabs,multirow,colortbl,enumitem}
\usepackage{algorithm,algorithmic}
\usepackage{aliascnt}
\usepackage{titletoc}
\usepackage{hyperref}
\hypersetup{colorlinks=true, linkcolor=black, citecolor=black, filecolor=black, urlcolor=black}
\usepackage[capitalize,noabbrev]{cleveref}

\theoremstyle{plain}
\newtheorem{theorem}{Theorem}[section]

\newaliascnt{proposition}{theorem}
\newtheorem{proposition}[proposition]{Proposition}
\aliascntresetthe{proposition}
\crefname{proposition}{Proposition}{Propositions}
\Crefname{proposition}{Proposition}{Propositions}

\newaliascnt{lemma}{theorem}
\newtheorem{lemma}[lemma]{Lemma}
\aliascntresetthe{lemma}
\crefname{lemma}{Lemma}{Lemmas}
\Crefname{lemma}{Lemma}{Lemmas}

\newaliascnt{corollary}{theorem}
\newtheorem{corollary}[corollary]{Corollary}
\aliascntresetthe{corollary}
\crefname{corollary}{Corollary}{Corollaries}
\Crefname{corollary}{Corollary}{Corollaries}

\theoremstyle{definition}
\newaliascnt{definition}{theorem}
\newtheorem{definition}[definition]{Definition}
\aliascntresetthe{definition}
\crefname{definition}{Definition}{Definitions}
\Crefname{definition}{Definition}{Definitions}

\newaliascnt{assumption}{theorem}
\newtheorem{assumption}[assumption]{Assumption}
\aliascntresetthe{assumption}
\crefname{assumption}{Assumption}{Assumptions}
\Crefname{assumption}{Assumption}{Assumptions}

\theoremstyle{remark}
\newaliascnt{remark}{theorem}
\newtheorem{remark}[remark]{Remark}
\aliascntresetthe{remark}
\crefname{remark}{Remark}{Remarks}
\Crefname{remark}{Remark}{Remarks}

\newcommand{\RR}{\mathbb{R}}
\DeclareMathOperator{\Ortho}{Ortho}
\DeclareMathOperator{\sign}{sign}
\DeclareMathOperator{\tr}{tr}
\DeclareMathOperator{\diag}{diag}
\DeclareMathOperator*{\argmax}{arg\,max}

\DeclareMathOperator{\rank}{rank}
\DeclareMathOperator{\Skew}{skew}
\DeclareMathOperator{\qf}{qf}
\newcommand{\inner}[2]{\langle #1,\, #2 \rangle}

\newcommand{\M}{\mathcal{M}}

\title[iMuon: Intrinsic Spectral Optimization]{iMuon: Intrinsic Spectral Optimization\\on Riemannian Matrix Manifolds}

\author[Y. Li]{Yibang Li $^{1*}$}
\thanks{$^1$ Nanyang Technological University, Singapore}
\thanks{$*$ Equal contribution}
\author[B. L. Pandey]{Bihari Lal Pandey $^{2*}$}
\author[R. Sah]{Ravi Sah $^2$}
\thanks{$^2$ Indian Institute of Technology Bombay, India}
\author[A. Han]{Andi Han $^3$}
\thanks{$^3$ The University of Sydney, Australia}
\author[C. Mostajeran]{Cyrus Mostajeran $^1$}
\author[P. Jawanpuria]{Pratik Jawanpuria $^2$}
\author[B. Mishra]{Bamdev Mishra $^4$}
\thanks{$^4$ Microsoft India}
\thanks{Email addresses: Yibang Li (\texttt{yibang001@e.ntu.edu.sg}), Bihari Lal Pandey (\texttt{bihari@iitb.ac.in}), Ravi Sah (\texttt{ravi.sah@iitb.ac.in}), Andi Han (\texttt{andi.han@sydney.edu.au}), Cyrus Mostajeran (\texttt{cyrussam.mostajeran@ntu.edu.sg}), Pratik Jawanpuria (\texttt{pratik.jawanpuria@iitb.ac.in}), Bamdev Mishra (\texttt{bamdevm@microsoft.com}).}

\begin{document}

\begin{abstract}
Muon and related norm-constrained optimizers, now central to large-scale learning, compute updates via a linear maximization oracle (LMO) over a matrix-norm ball in Euclidean space. However, they do not generalize cleanly to manifold-valued parameters such as low-rank factorizations, orthogonality constraints, or symmetric positive-definite (SPD) matrices. Naively restricting the Muon LMO to the tangent space (i) breaks quotient symmetries and (ii) couples the tangent-space constraint with an ambient norm bound, which obstructs closed-form solutions. We resolve both issues with a key observation: every Riemannian metric canonically lifts a unitarily invariant Euclidean norm to an intrinsic norm on each tangent space. The resulting intrinsic LMO is symmetry preserving and, under a structural condition on the scaled tangent space, reduces to a standard Euclidean LMO. Building on this, we introduce intrinsic Muon (iMuon), a unified framework with closed-form updates on the fixed-rank, SPD, Stiefel, and Grassmann manifolds for various unitarily invariant norms, including the spectral, Frobenius, and nuclear norms. We establish deterministic and stochastic convergence rates of iMuon with constants that depend only on the manifold dimension. Notably, on the fixed-rank manifold, the constant depends only on the rank, which makes the rate independent of factor conditioning and removes the runtime factor rescaling required by prior work. Experiments on LoRA fine-tuning of LLMs, image classification, and subspace learning illustrate the efficacy of the proposed approach.

\end{abstract}

\maketitle

\section{Introduction}\label{sec:intro}

A growing family of matrix-variate optimizers is unified by a single design principle: at each step, select the update direction that maximizes the inner product with the (stochastic) gradient over a matrix-norm ball, a linear maximization oracle (LMO)~\citep{pethick2025lmo,bernstein2025old,xie2026sso}. The choice of norm determines the optimizer: the Frobenius norm recovers SGD, the spectral norm yields Muon~\citep{jordan2024muon,bernstein2025old}, and a combined spectral-nuclear constraint gives NuMuon~\citep{dolatabadi2026numuon}. This perspective has driven recent advances in LLM pretraining~\citep{zhao2024galore,mo2025loro,wang2026taming,janson2026stabilizing} and low-rank adaptation (LoRA) fine-tuning~\citep{hu2022lora,schotthoefer2025geometric,park2025stiefel,gu2026mano}. Despite their empirical success, the design of such LMO-based optimizers has been mostly confined to unconstrained Euclidean spaces.

However, many modern representation-learning problems are intrinsically manifold-valued: LoRA fine-tuning operates on the fixed-rank manifold~\citep{mishra2014fixed,zhang24riemannianlora}, orthogonality constraints define the Stiefel manifold, and robust covariance estimation lives on the SPD manifold. Extending norm-constrained matrix optimizers to such Riemannian settings remains relatively unexplored~\citep{gu2026mano,bogachev2026riemannion}. The seemingly natural fix of restricting the LMO to the tangent space runs into two fundamental obstacles:

\begin{itemize}[leftmargin=0.3in]
    \item \textbf{(O1) Symmetry breaking on quotient manifolds.}
    The Euclidean spectral norm is not invariant under the symmetry group of a quotient manifold (i.e., manifold arising out of symmetries). For example, on the fixed-rank manifold, equivalent factorizations $X=BA$ that share the same product induce different factor-wise Muon updates, so both the direction and the magnitude of the update depend on an arbitrary representative~\citep{janson2026stabilizing}.
    \item \textbf{(O2) Coupling of the tangent space and norm constraints obstructs a closed-form solution.}
    Restricting the spectral-norm LMO to the tangent space couples two constraints, tangent space membership and an ambient norm bound, whose joint enforcement generally precludes the Euclidean polar-factor solution and admits no closed form on most manifolds of interest.
\end{itemize}
Existing literature addresses these issues only partially.
Spectron~\citep{janson2026stabilizing} bounds the ambient update by an iterate-dependent factor rescaling computed dynamically. Riemannion~\citep{bogachev2026riemannion} relies on an approximate tangent space orthogonalization. Manifold constrained steepest descent (MCSD) \citep{yang2026mcsd} drops the tangent constraint and projects back to the manifold instead. However, none of the existing approaches resolves~(O1)~and~(O2) jointly.

In this paper, we propose intrinsic Muon (\textbf{iMuon}), a principled and unified framework for spectral, and, more generally, unitarily invariant norm-constrained optimization on Riemannian matrix manifolds (whose points can be represented as matrices) that resolves both~(O1)~and~(O2) simultaneously. The key observation is that the Riemannian metric canonically induces an \textit{intrinsic} norm on each tangent space, with the Frobenius case recovering the standard Riemannian norm and the spectral case giving a manifold-aware analog of Muon. Our {main contributions} are:

\begin{itemize}[leftmargin=0.3in]
    \item \textbf{Intrinsic LMO with closed-form solutions on four manifolds (\Cref{sec:principle,sec:solutions}).} We identify two structural properties that together resolve~(O1)~and~(O2) and reduce the intrinsic LMO to a standard Euclidean problem: a left-right product form of the metric (\Cref{prop:symmetry}) together with either singular value invariance (\Cref{def:sv_invariant}) or $\varphi$-contractivity (\Cref{def:contractive}) of the scaled tangent space. This yields exact closed-form solutions (\Cref{prop:tangent_compat}) on the fixed-rank, SPD, Stiefel, and Grassmann manifolds, for various norms including the spectral, Frobenius, and nuclear norms.
    \item \textbf{Unifying viewpoint.} Beyond enabling efficient computation, instantiating these intrinsic LMOs across different norms and geometries provides a generic approach that unifies established algorithms and yields new ones. In the Euclidean special case~\citep{jordan2024muon,bernstein2025old}, our proposed LMO recovers standard SGD, Muon, and NuMuon~\citep{dolatabadi2026numuon} under different norm ball constraints.
    On the fixed-rank manifold equipped with quotient geometry, the Frobenius instance recovers Scaled GD~\citep{mishra2012geometry,tong2021scaledgd,zhang24riemannianlora}, while the spectral instance is, to our knowledge, the first quotient-Muon optimizer with an exact closed-form tangent space update.
    \item \textbf{Convergence rate (\Cref{sec:convergence}).} For both deterministic and stochastic iMuon, we prove an $O(1/\sqrt{T})$ rate in the dual intrinsic norm (up to a logarithmic factor and an additive noise term in the stochastic case) whose constant depends solely on the manifold dimensions (and not on the iterates). On the fixed-rank manifold with the spectral norm, this constant depends only on the rank, yielding a rate that is      {independent of factor conditioning}, and, for example, obviates the runtime factor rescaling of \citet{janson2026stabilizing}.
    \item \textbf{Empirical validation (\Cref{sec:experiments}).} {Experiments across LoRA fine-tuning, image classification, and subspace learning confirm the merits of the intrinsic approach in various settings.}
\end{itemize}

\section{Background and related work}\label{sec:setup}
We briefly review concepts of Riemannian geometry for matrix manifolds (\Cref{subsec:riemannian_primer}), Muon and norm-constrained Euclidean LMO (\Cref{subsec:muon_background}), and recent attempts at generalizing the latter to low-rank and matrix manifold settings (\Cref{subsec:related_work}). {Additional background and related works are discussed in \Cref{app:background,app:related_work_extended} and detailed comparisons are provided in  \Cref{app:comparisons}.}

\subsection{Riemannian optimization on matrix manifolds}\label{subsec:riemannian_primer}
We consider the problem $\min_{x\in\M}f(x)$ on a Riemannian matrix manifold $\M$ with smooth $f:\M\to\RR$.
The metric $g_x:T_x\M\times T_x\M\to\RR$ on the tangent space $T_x\M$ at $x\in \M$ admits a self-adjoint, positive-definite operator form $g_x(\xi,\zeta)=\langle G_x\xi,\zeta\rangle$, with induced Riemannian norm $\|\xi\|_x=\|G_x^{1/2}\xi\|_F$, where $\|\cdot \|_F$ is the Frobenius norm. The Riemannian gradient $\textnormal{grad}\,f(x)\in T_x\M$ satisfies $g_x(\textnormal{grad}\,f,\zeta)=\langle\nabla f(x),\zeta\rangle$ for all $\zeta\in T_x\M$.
To update parameters, a retraction operator $R_x : T_x\M \to \M$ generalizes the Euclidean step $x+\eta\xi$ with step size $\eta$. On an \textit{embedded submanifold} $\M\subset\RR^{p\times q}$, $T_x\M$ is a linear subspace of the ambient space. On a \textit{quotient manifold} $\M=\overline{\M}/\mathcal{G}$, obtained by symmetry invariance $\mathcal{G}$ in the total space $\overline{\M}$, the metric must be $\mathcal{G}$-invariant for the Riemannian gradient to be well-defined independently of the representative. {Please refer to Appendix~\ref{app:related_work_extended} and standard references~\citep{absil2008optimization,boumal2023intromanifolds} for more details.}

\subsection{Muon and norm-constrained Euclidean optimizers}\label{subsec:muon_background}

For an unconstrained matrix parameter $X\in\RR^{p\times q}$ and gradient (or momentum) $M\in\RR^{p\times q}$,  Muon~\citep{jordan2024muon,bernstein2025old,carlson2015psd} has the update $X_{+}=X-\eta\,O$, where $O$ is the solution to:
\begin{equation}\label{eqn:muon_lmo}
\text{Spectral-norm LMO} \qquad    O \;=\; \argmax_{\|O\|_2\le\tau}\;\langle O, M\rangle \;=\; \tau\,\Ortho(M),
\end{equation}
with $\Ortho(M)=UV^\top$ being the polar factor of $M$, where $M=U\Sigma V^\top$ is the compact singular value decomposition with $\Sigma\succ 0$. \citet{bernstein2025old} formalize~\eqref{eqn:muon_lmo} as steepest descent under the spectral norm. Subsequent works generalize the LMO viewpoint to broader operator norms~\citep{pethick2025lmo}, modular norms across the network~\citep{bernstein2025modular,large2025modular}, additional nuclear-norm budget~\citep{dolatabadi2026numuon}, per-layer and product norms~\citep{crawshaw2026exploration}, and analyze convergence, preconditioning, and Newton--Schulz approximations of $\Ortho(\cdot)$~\citep{shen2025muon,ma2026preconditioning,sato2025convergence,kim2026convergence,amsel2026polar,riabinin2025gluon,wen2026fantastic}.
All of these methods reside in unconstrained Euclidean space and our intrinsic LMO~\eqref{eqn:main_lmo}  reduces to them in that special case (\Cref{app:connections}).

\subsection{Muon variants for structured and manifold settings}\label{subsec:related_work}\label{subsec:naive_extension}

A natural attempt to use~\eqref{eqn:muon_lmo} on a Riemannian manifold $\M$ is to additionally restrict the search to the tangent space $T_x \M$ at $x\in \M$:
\begin{equation}\label{eqn:naive_manifold_lmo}
\xi^*_{\text{naive}} \;=\; \argmax_{\xi\in T_x\M,\;\|\xi\|_2\le\tau}\;\langle\xi,\,\nabla f(x)\rangle.
\end{equation}
This naive lift is the implicit template behind several recent manifold/structured Muon proposals but it inherits the two {challenges} (O1) and (O2) introduced in \Cref{sec:intro}.

\textbf{LoRA and the fixed-rank manifold.}  For LoRA-style fine-tuning $X=BA$ with $B\in\RR^{m\times r},A\in\RR^{r\times n}$, $X$ lies on the fixed-rank manifold $\M_r$, equipped with the $\textnormal{GL}(r)$ symmetry $(B,A)\sim(BN^{-1},NA)$, where $N \in \textnormal{GL}(r)$ the set of $r\times r$ non-singular matrices.
Existing works \citep{mishra2012geometry,tong2021scaledgd,zhang24riemannianlora} propose Riemannian preconditioned gradient descent for such settings, which corresponds to the Frobenius instance of our framework on $\M_r$. Recently,
\citep{janson2026stabilizing,kang2026spectral} study the following factor-wise Muon problem, which independently solves a Euclidean spectral-norm LMO on each factor, i.e.,
\begin{equation}\label{eqn:factorwise_muon}
\dot{B}^*=\argmax_{\|\dot{B}\|_2\le\tau}\inner{\dot{B}}{\nabla_X f\,A^\top},\quad \dot{A}^*=\argmax_{\|\dot{A}\|_2\le\tau}\inner{\dot{A}}{B^\top\nabla_X f}.
\end{equation}
Solving (\ref{eqn:factorwise_muon}) requires applying $\Ortho(\cdot)$ to each factor's gradient (or momentum). However, the resulting ambient update is not invariant under the $\textnormal{GL}(r)$ symmetry, exhibiting (O1).
\citet{kang2026spectral} prove a continuous-time convergence for the momentum-free limit, conditional on factor-norm boundedness. On the other hand, \citep{janson2026stabilizing} bounds the ambient spectral norm of the update by an iterate-dependent rescaling $\rho=\eta/(\|B\|_2+\|A\|_2+1)$ computed at runtime.
While this circumvents (O1) numerically, our approach (Section~\ref{sec:principle}) handles (O1) algebraically. Please also refer to  \Cref{cor:lowrank,app:compare_lowrank} for more details.

\textbf{Embedded manifolds and the coupled tangent space LMO.}\ \ Restricting~\eqref{eqn:muon_lmo} to $T_x\M$ couples the spectral-norm constraint to a tangent-space constraint. The resulting LMO~\eqref{eqn:naive_manifold_lmo} generally has no closed form, which is (O2) described earlier.
Riemannion~\citep{bogachev2026riemannion}  approximately solves the tangent-constrained spectral LMO on the embedded fixed-rank manifold,
while MCSD/SPEL~\citep{yang2026mcsd} drops the tangent-space constraint on the Stiefel manifold and projects the ambient solution onto the Stiefel manifold. On the Stiefel manifold, the spectral-norm variant of~\eqref{eqn:naive_manifold_lmo} has also been studied in the context of manifold Muon~\citep{bernstein2025manifolds,cesista2025stiefel,su2025stiefel}, and in a concurrent work~\citep{solonko2026skewon}, which we discuss in \Cref{app:compare_stiefel}.

In the next section, we propose an intrinsic LMO that preserves quotient symmetries and identify conditions under which it admits closed-form solutions. In \Cref{sec:solutions}, we derive these solutions on the fixed-rank, SPD, Stiefel, and Grassmann manifolds.

\section{The intrinsic LMO}\label{sec:principle}

We now generalize norm-constrained Euclidean LMOs, e.g.,~(\ref{eqn:muon_lmo}), to the geometry-aware manifold setting. We begin by discussing the key idea that the Riemannian metric itself canonically induces the appropriate notion of matrix norm on the tangent space. Later, we propose a novel family of intrinsic norm-constrained LMOs on the tangent space of manifolds and analyze it with respect to the challenges (O1) and (O2).

\textbf{Intrinsic norms on the tangent space.\ } The metric operator $G_x\succ 0$ induces the Riemannian norm $\|\xi\|_x\coloneqq \|G_x^{1/2}\xi\|_F$, where $\xi\in T_x\M$ and $\|\cdot\|_F$ is the Frobenius norm.
This suggests a canonical extension from the Frobenius norm to any unitarily invariant norm \(\varphi\): for \(\xi\in T_x\M\), we define the intrinsic \(\varphi\)-norm  by $\|\xi\|_{\varphi,x}:=\varphi(G_x^{1/2}\xi)$.
In particular, choosing $\varphi = \|\cdot\|_2$, $\|\cdot\|_F$, or $\|\cdot\|_*$ yields the intrinsic spectral, Frobenius, and nuclear norms, respectively.
We note that $\|\cdot\|_{\varphi,x}$ is intrinsic in the sense that it is determined entirely by the Riemannian metric. It reduces to the standard matrix norm in the Euclidean case $G_x = I$, while for $\varphi = \|\cdot\|_F$ it recovers the Riemannian norm itself.

\subsection{Proposed LMO formulation}
Given the intrinsic norm $\|\cdot\|_{\varphi,x}$ introduced above, we propose the update direction at $x \in \M$ for optimizing $\min_{x\in\M}f(x)$ through the following LMO:
\begin{equation}\label{eqn:main_lmo}
{
    \xi^* = \argmax_{\xi\in T_x\M,\;\varphi(G_x^{1/2}\xi)\le\tau}\;g_x(\xi,\,\text{grad}\,f),
}
\end{equation}
where \(\varphi\) is any unitarily invariant norm and $g_x$ is the metric defined on the tangent space $T_x\M$ at $x\in\M$. The iterate is then updated by retracting along the $-\xi^*$ direction: $x_+ = R_x(-\eta\,\xi^{*})$ with step size $\eta>0$.
Compared with the naive tangent space LMO~\eqref{eqn:naive_manifold_lmo}, both the objective $g_x(\xi,\text{grad}\,f)$ and the feasible set $\{\xi \in T_x\mathcal{M} : \|\xi\|_{\varphi,x}\leq \tau\}$ in the proposed LMO \eqref{eqn:main_lmo} are now intrinsic to the manifold geometry.

Our formulation recovers several familiar special cases. If $\mathcal{M}$ is a Euclidean space with the standard metric, then  \eqref{eqn:main_lmo} reduces exactly to the usual Euclidean norm-constrained LMO. If $\varphi = \|\cdot\|_F$ and $\tau = 1$, then \eqref{eqn:main_lmo} returns the normalized Riemannian gradient direction \citep{absil2008optimization,boumal2023intromanifolds}.
In the next sections, we show that our LMO formulation simultaneously preserves the quotient symmetries (resolving (O1)) and admits efficient closed-form solutions (resolving (O2)).

\subsection{Symmetry preservation}\label{subsec:symmetry_preservation}

We next analyze the proposed intrinsic constraint in \eqref{eqn:main_lmo} under a group $\mathcal{G}$ of symmetries of the metric. This covers quotient manifolds $\mathcal{M}=\bar{\mathcal{M}}/\mathcal{G}$ ($\bar{\mathcal{M}}$ being the total space) with a $\mathcal{G}$-invariant Riemannian metric $g_x$, where $\mathcal{G}$ acts on representatives, as well as groups acting on $\mathcal{M}$ by isometries, such as affine-invariance in the context of symmetric positive-definite matrices.
In such settings, we show that our intrinsic constraint respects the underlying symmetry if the metric admits a {left-right form}: $g_x(\xi,\zeta)=\tr(\xi^\top P_x \zeta Q_x)$, or equivalently, the metric operator factors as $G_x \xi=P_x \xi Q_x$, where $P_x$ and $Q_x$ are symmetric positive-definite matrices depending smoothly on $x$. Then, $G_x^{1/2} \xi=P_x^{1/2}\xi Q_x^{1/2}$ and the intrinsic norm reads $\|\xi\|_{\varphi,x} = \varphi(P_x^{1/2}\xi Q_x^{1/2})$.

In \Cref{sec:solutions}, we discuss popular manifolds where the metric has the left-right form.
It should be noted that the intrinsic norm itself is well-defined for any positive self-adjoint $G_x$ and does not require the left-right factorization. However, under such symmetries, this factorization makes the metric scaling explicit and the symmetry behavior of the intrinsic norm transparent. Our next result shows that, under such metrics and a compatible transformation of the factors $P_x,Q_x$ under $\mathcal{G}$, the intrinsic norm constraint is $\mathcal{G}$-invariant.

\begin{proposition}[Symmetry invariance of the intrinsic constraint]\label{prop:symmetry}
Let a group $\mathcal{G}$ act on $\M$ (on representatives in $\overline{\M}$ when $\M=\overline{\M}/\mathcal{G}$), with $\tilde{x}=h\cdot x$ and induced tangent action $\xi\mapsto\tilde{\xi}=L_h\,\xi\,R_h^\top$. Suppose the metric operator on $T_x\M$ is $G_x\xi=P_x\,\xi\,Q_x$ with $P_x,Q_x\succ 0$, and $P_{\tilde{x}}=L_h^{-\top}P_xL_h^{-1}$, $Q_{\tilde{x}}=R_h^{-\top}Q_xR_h^{-1}$. Then $g$ is $\mathcal{G}$-invariant and $\varphi(G_{\tilde{x}}^{1/2}\tilde{\xi})=\varphi(G_x^{1/2}\xi)$ for all $\xi\in T_x\M$ and every unitarily invariant norm $\varphi$. This holds on the manifolds considered in \Cref{sec:solutions}: the fixed-rank manifold $\M_r$ with $\textnormal{GL}(r)$ (block-wise), the Grassmann manifold $\textnormal{Gr}(m,r)=\textnormal{St}(m,r)/O(r)$, and the manifold $S_{++}^n$ of symmetric positive-definite matrices with $\textnormal{GL}(n)$ acting by congruence (an isometry group).
\end{proposition}

The proof is given in \Cref{app:proof_symmetry}.
With respect to the proposed LMO problem (\ref{eqn:main_lmo}), the above result shows that the feasible set $\{\xi:\varphi(G_x^{1/2}\xi)\le\tau\}$ is invariant under the symmetry group $\mathcal{G}$. When $f$ is $\mathcal{G}$-invariant, as on quotient manifolds, the objective $g_x(\xi,\text{grad}\,f)$ is preserved as well, so the optimal update is equivariant, i.e., $\tilde{\xi}^*=L_h\,\xi^*R_h^\top$, and the resulting update does not depend on the chosen representative. Thus, our result shows that our intrinsic norm generalizes the Frobenius-norm invariance (that makes the Riemannian gradient well-defined on  quotients~\citep{absil2008optimization,mishra2016riemannian, boumal2023intromanifolds}) to all unitarily invariant norms.

\subsection{Efficient and optimal solution to the coupled norm and tangent-space constraints}\label{subsec:reduction}

A key computational challenge in solving (\ref{eqn:main_lmo}) efficiently is the simultaneous presence of two constraint sets: $\{\xi:\xi\in T_x\M\}$ and $\{\xi:\varphi(G_x^{1/2}\xi)\le\tau\}$.
We now show that, under certain conditions, these two constraints can be simplified, which allows for obtaining closed-form solutions of (\ref{eqn:main_lmo}).

In this regard, we first define the scaled tangent space $\mathcal{Z}_x\coloneqq\{G_x^{1/2}\eta:\eta\in T_x\M\}$. Since $G_x^{1/2}$ is invertible on $T_x\mathcal{M}$, the map $\xi\mapsto Z \coloneqq G_x^{1/2}\xi$ is an isomorphism between $T_x\mathcal{M}$ and $\mathcal{Z}_x$. The intrinsic norm may now be rewritten as $\varphi(G_x^{1/2}\xi) = \varphi(Z)$.
Moreover, if we define the scaled gradient $H \coloneqq G_x^{1/2}\text{grad} f$, then the objective of (\ref{eqn:main_lmo}) satisfies $g_x(\xi,\text{grad} f) = \langle Z,H\rangle$.
To formalize our next result, we also require the following notions.
\begin{definition}\label{def:sv_invariant}
A linear subspace $\mathcal{S}\subseteq\RR^{p\times q}$ is \emph{singular value (SV) invariant} if, for every $Z\in\mathcal{S}$ with compact SVD $Z=U\Sigma V^\top$ ($\Sigma\succ 0$ diagonal), we have $UDV^\top\in\mathcal{S}$ for every nonnegative diagonal $D$ satisfying $D_{ii}=D_{jj}$ whenever $\Sigma_{ii}=\Sigma_{jj}$.
\end{definition}
\begin{definition}\label{def:contractive}
Let $\varphi$ be a unitarily invariant norm. A linear subspace $\mathcal{S}\subseteq\RR^{p\times q}$ is \emph{$\varphi$-contractive} if the orthogonal projection $P_{\mathcal{S}}$ onto $\mathcal{S}$ satisfies $\varphi(P_{\mathcal{S}}(W))\le\varphi(W)$ for all $W\in\RR^{p\times q}$.
\end{definition}
For example, every linear subspace is $\|\cdot\|_F$-contractive, and $\text{Skew}(r)$ is $\varphi$-contractive for every unitarily invariant $\varphi$ (\Cref{lem:skew_compat}).

\begin{proposition}[Closed-form solution]\label{prop:tangent_compat}
Let $\mathcal{Z}_x\coloneqq\{G_x^{1/2}\eta:\eta\in T_x\M\}$, let $\varphi$ be a unitarily invariant norm, and suppose $\mathcal{Z}_x$ is SV-invariant or $\varphi$-contractive. Then the solution of (\ref{eqn:main_lmo}) is given by
\begin{equation}
    \xi^*=G_x^{-1/2}P_{\mathcal{Z}_x}(Z^{*}),\ \textup{where } Z^{*} = \argmax_{\varphi(Z)\le\tau}\inner{Z}{H},
\end{equation}
$H:=G_x^{1/2}(\textnormal{grad}\,f)$, and $P_{\mathcal{Z}_x}$ is the orthogonal projection onto $\mathcal{Z}_x$.
Moreover, $Z^*$ can be computed as $Z^*= U\diag(z^*)V^\top$, where $H=U\diag(\sigma) V^\top$ is the compact SVD of $H$ and $z^*$ solves the vector problem $\max_{z:\varphi(z)\le \tau }\inner{z}{\sigma}= \tau \varphi^\circ(\sigma)$ with $z^*_i=z^*_j$ whenever $\sigma_i=\sigma_j$. With this choice, $P_{\mathcal{Z}_x}(Z^*)=Z^*$ when $\mathcal{Z}_x$ is SV-invariant.
\end{proposition}

The proof is given in \Cref{app:proof_reduction}. We note that the vector problem ($\argmax_{z:\varphi(z)\le \tau }\inner{z}{\sigma}$) admits a closed-form solution for various unitarily invariant norms, including the spectral, Frobenius, and nuclear norms (\Cref{app:primal_dual}).
In Section~\ref{sec:solutions}, we discuss several manifolds of interest where Proposition~\ref{prop:tangent_compat} is useful.
In \Cref{app:other_manifolds}, we also present a broader discussion including cases where SV invariance of $\mathcal{Z}_x$ or the left-right form of the metric fail.

\section{Application of intrinsic LMO (\ref{eqn:main_lmo}) on matrix manifolds}\label{sec:solutions}

In this section, we discuss the closed-form solutions of the proposed LMO (\ref{eqn:main_lmo}) for various manifolds and intrinsic spectral/Frobenius/nuclear norms.

\paragraph{Product of tangent spaces.}
When the tangent space decomposes as a product $T_x\M=\RR^{p_1\times q_1}\times\cdots\times\RR^{p_k\times q_k}$ and $G_x$ has the corresponding block structure, then  $G_x^{1/2}$ acts component-wise. In such settings, we may choose $\varphi(Z_1,\ldots,Z_k)$ as the $\ell_\infty$-product norm, i.e., $\varphi(Z_1,\ldots,Z_k)\coloneqq\max_i\varphi(Z_i)$, as it decouples the intrinsic LMO into $k$ independent per-block subproblems:
\begin{equation}\label{eqn:whitened_lmo_product}
\{Z_i^*\}_{i=1}^k =  \argmax_{\varphi(Z_1,\ldots,Z_k)\le\tau} \sum\nolimits_i \inner{Z_i}{H_i} \ \equiv \
Z_i^* = \argmax_{\varphi(Z_i)\le\tau}\;\inner{Z_i}{H_i}, \qquad i=1,\ldots,k,
\end{equation}
each solvable in closed form via \Cref{prop:tangent_compat}, since the constraint $\max_i\varphi(Z_i)\le\tau$ separates into $\varphi(Z_i)\le\tau$ for each block independently.
Interestingly, \citep{crawshaw2026exploration} has also studied $\ell_\infty$-based product norms in the context of combining norms across neural network layers.

\textbf{Fixed-rank manifold.\ }\label{subsec:lowrank}
We now consider the set $\M_r$ of rank-$r$ matrices parameterized as $X=BA$ with $B\in\RR^{m\times r}_*$, $A\in\RR^{r\times n}_*$ (where $\RR^{\cdot}_*$ denotes full-rank matrices). We note that the factorization $X=BA$ has the quotient symmetry: $(B,A)\sim(BN^{-1},N A)$ for $N\in\text{GL}(r)$, where $\text{GL}(r)$ denotes the group of $r\times r$ invertible matrices. This results in $\M_r = (\RR^{m\times r}_*\times\RR^{r\times n}_*)/\text{GL}(r)$ and $T_x \M_r = \mathbb{R}^{m\times r} \times \mathbb{R}^{r\times n}$~\citep{absil2008optimization,mishra2014fixed}. If $x=(B,A)\in\M_r$ and $\xi_i=(\dot{B}_i,\dot{A}_i)\in T_x \M_r$ for $i=1,2$, then
the coupled metric is
$g_x(\xi_1,\xi_2)=\tr(\dot{B}_1^\top\dot{B}_2\,AA^\top)+\tr(B^\top B\,\dot{A}_1\dot{A}_2^\top)$~\citep{mishra2012geometry,mishra2016riemannian,zhang24riemannianlora}. We observe that the metric is $\text{GL}(r)$-invariant with the left-right form (\Cref{prop:symmetry}) and has a block structure, with
$G_x^{1/2}\xi_i=(\dot{B}_i(A A^\top)^{1/2},\;(B^\top B)^{1/2}\dot{A}_i)$. Thus, we decouple the LMO into $\dot{B}$ and $\dot{A}$ subproblems~(\ref{eqn:whitened_lmo_product}).

\begin{corollary}[Intrinsic LMO (\ref{eqn:main_lmo}) for fixed-rank manifold]\label{cor:fixed_rank} Let $\M=\M_r$ and at $X=BA\in\M_r$, let $\nabla_{X} f$ denote the Euclidean gradient w.r.t. $X$. Then, with the above discussed $\textnormal{GL}(r)$-invariant metric and  $\varphi=\|\cdot\|_2$, an optimal solution of the proposed LMO (\ref{eqn:main_lmo}) under the product constraint $\max\{\|\dot{B}(AA^\top)^{1/2}\|_2,\|(B^\top B)^{1/2}\dot{A}\|_2\}\le\tau$ is $\xi^{*}=(\dot{B}^{*},\dot{A}^{*})$, where $\dot{B}^{*} = \tau\Ortho((\nabla_{B}f)(AA^\top)^{-1/2})(AA^\top)^{-1/2}$, $\dot{A}^{*}=\tau(B^\top B)^{-1/2}\Ortho((B^\top B)^{-1/2}\nabla_{A}f)$,  $\nabla_{B}f=\nabla_X fA^\top$, and $\nabla_{A}f=B^\top \nabla_X f$.
\end{corollary}
The proof is given in \Cref{app:proof_fixed_rank}.
We note that the intrinsic spectral LMO in \Cref{cor:fixed_rank} differs from prior $\M_r$-Muon constructions in two ways: (i) Factor-wise Muon~\citep{janson2026stabilizing,kang2026spectral} solves~\eqref{eqn:factorwise_muon} by applying $\Ortho(\cdot)$ to each factor gradient. The resulting ambient update is not $\textnormal{GL}(r)$-invariant and its spectral norm scales with $\|A\|_2+\|B\|_2$. Hence, Spectron~\citep{janson2026stabilizing} rescales the solution of~\eqref{eqn:factorwise_muon} during runtime to bound the ambient spectral norm. (ii) Riemannion~\citep{bogachev2026riemannion} works in the embedded view with $G_X=I$  and approximates the tangent-constrained LMO by a rank-$2r$  truncated polar followed by tangent space re-projection. Therefore, the spectral-norm bound is satisfied only approximately.
\Cref{cor:fixed_rank}, in contrast, is the exact maximizer of~\eqref{eqn:main_lmo} on each factor block, and its Frobenius instance recovers Riemannian preconditioned gradient for fixed-rank matrices~\citep{mishra2016riemannian,tong2021scaledgd,zhang24riemannianlora}. \Cref{app:compare_lowrank} provides detailed update-rule and norm-bound comparisons.

\textbf{SPD manifold.\ }\label{subsec:other_manifolds}
The SPD manifold $S_{++}^n$ is the space of all $n\times n$ symmetric positive-definite matrices. At any $X\in S_{++}^n$, $T_X S_{++}^n = \text{Sym}(n)$, the space of $n\times n$ symmetric matrices. Since $f$ is defined on symmetric matrices, the Euclidean gradient $\nabla_X f\in\text{Sym}(n)$ is itself symmetric. Consider $S_{++}^n$ with the affine-invariant metric: $g_X(\xi,\zeta)=\tr(X^{-1}\xi X^{-1}\zeta)$ \citep{bhatia2009positive}. Both Propositions \ref{prop:symmetry} and \ref{prop:tangent_compat} are satisfied with $G_X^{1/2}\xi=X^{-1/2}\xi X^{-1/2}$ and the Riemannian gradient is $\text{grad}\,f(X) = X(\nabla_X f)X$. The solution of the proposed LMO (\ref{eqn:main_lmo}) with $\varphi=\|\cdot\|_2$ is $\xi^* = \tau\,X^{1/2}\,\sign(H)\,X^{1/2}$, where $H=G_X^{1/2}(\text{grad}\,f)=X^{1/2}(\nabla_X f)X^{1/2}$ and $\sign(H)=Q\diag(\sign(\lambda_i))Q^\top$ \citep{higham2008functions} with $Q$ and $\{\lambda_i\}_{i}$ being eigenvectors and eigenvalues of $H$, respectively.
We also note that the intrinsic Frobenius norm constraint in this setting recovers the natural gradient~\citep{amari1998natural}.
Alternative SPD metrics are discussed in~\Cref{app:spd_metrics}.

\textbf{Stiefel manifold.\ }
The Stiefel manifold $\text{St}(m,r)=\{X\in\RR^{m\times r}:X^\top X=I_r\}$ is the set of $m\times r$ matrices with orthonormal columns, and its tangent space at $X$ is $T_X\text{St}=\{\xi\in\RR^{m\times r}:X^\top\xi+\xi^\top X=0\}$. We equip $\text{St}(m,r)$ with the Euclidean metric, i.e., $G_X=I$. Any tangent vector $\xi\in T_X\text{St}$ can be uniquely parametrized as $\xi=[X,X_\perp][A;B]$, where $A\in\text{Skew}(r)$ (set of $r\times r$ skew-symmetric matrices), $B\in\RR^{(m-r)\times r}$, $[A;B] \coloneqq [A^\top, B^\top]^\top$, and $X_\perp\in\RR^{m\times(m-r)}$ is an orthonormal basis for the orthogonal complement of $\text{col}(X)$ \citep{edelman1998geometry,absil2008optimization, boumal2023intromanifolds}. Since $G_X=I$, the space $\mathcal{Z}_X=\{G_X^{1/2}\eta:\eta\in T_X\M\}$ is the same as $T_X\M$. As $\mathcal{Z}_X$ is not SV-invariant, \Cref{prop:tangent_compat} cannot be directly applied to $\text{St}(m,r)$.

However, we note that $\varphi(\xi)=\varphi([A;B])$ and $g_X(\xi,\text{grad}\,f)=\inner{A}{X^\top\nabla_X f}+\inner{B}{X_\perp^\top\nabla_X f}$. Furthermore, we employ the $\ell_\infty$-product norm $\max\{\varphi(A),\varphi(B)\}\le\tau$ to decouple the LMO into a skew and a normal subproblem as in~\eqref{eqn:whitened_lmo_product}, where the skew subproblem retains the constraint $A\in\text{Skew}(r)$. In the spectral case, this relaxes the joint norm, since $\max\{\|A\|_2,\|B\|_2\}\le\|[A;B]\|_2\le\sqrt{2}\max\{\|A\|_2,\|B\|_2\}$. Both blocks are SV-invariant (\Cref{app:sv_invariance}), so \Cref{prop:tangent_compat} gives the following exact solution for this product-norm formulation.
\begin{lemma}[Intrinsic LMO (\ref{eqn:main_lmo}) for Stiefel manifold]\label{lemma:stiefel_spectral}
Let $\M=\textnormal{St}(m,r)$ and at $X\in\M$, let $G_X=I$. Let $\nabla_{X} f$ denote the Euclidean gradient w.r.t. $X$. Then, with $\varphi=\|\cdot\|_2$, an optimal solution of the proposed LMO (\ref{eqn:main_lmo}) under the product constraint $\max\{\|A\|_2,\|B\|_2\}\le\tau$ is
$\xi^{*}=\tau( X\Skew(\Ortho(\Skew(X^\top\nabla_X f))) + X_\perp \Ortho(X_\perp^\top\nabla_X f))$.
\end{lemma}
The proof is given in \Cref{app:proof_stiefel}.
Recently, \citet{yang2026mcsd} studied LMO-based updates in the context of optimization over the Stiefel manifold.  However, they sidestep the coupled LMO problem by solving the LMO in the ambient space and projecting the result onto the Stiefel manifold. Thus, the update direction need not lie in $T_X \M$.
Other recent approaches to the joint-norm formulation are discussed in \Cref{app:compare_stiefel}.

\textbf{Grassmann manifold.} The Grassmann manifold $\text{Gr}(m,r)$ is the set of $r$-dimensional linear subspaces of $\RR^m$, represented by an orthonormal basis $X\in\RR^{m\times r}$ with $X^\top X=I_r$ up to right multiplication by $O(r)$, i.e., $\text{Gr}(m,r)=\text{St}(m,r)/O(r)$ \citep{edelman1998geometry,boumal2023intromanifolds}. However, unlike the Stiefel setting, the tangent space (horizontal space) $T_X \text{Gr}=\{\xi:X^\top\xi=0\}$ is SV-invariant. Using \Cref{prop:tangent_compat}, we get the optimal solution of our intrinsic LMO (\ref{eqn:main_lmo}) for $\varphi=\|\cdot\|_2$ as $\xi^*=\tau\,\Ortho((I-XX^\top)\nabla_X f)$.

\textbf{Overall.}  A salient feature of our above analysis across various manifolds is that computing the optimal solution of the proposed problem (\ref{eqn:main_lmo}) does not require the computation of the Riemannian gradient. Our closed-form optimal expressions rely only on Euclidean gradients, which are available via back-propagation. Please refer to \Cref{app:implementation} for proofs and additional results. In particular, \Cref{tab:per_manifold} in \Cref{app:summary_table} states, for each manifold, whether the update solves the joint LMO \Cref{eqn:main_lmo} or the product LMO \Cref{eqn:whitened_lmo_product}.

\begin{algorithm}[t]
\caption{Proposed Intrinsic Muon (iMuon) algorithm}
\label{alg:lmo}
\begin{algorithmic}[1]
\REQUIRE Initial point $x_0\in\M$, step size $\eta>0$, norm bound $\tau>0$, norm $\varphi$, iterations $T$
\FOR{$t=0,1,\ldots,T-1$}
    \STATE Compute Euclidean gradient $\nabla f(x_t)$
    \STATE Obtain $\xi_t^*$ by solving intrinsic LMO (\ref{eqn:main_lmo}) using $\nabla f(x_t)$ (refer to Section \ref{sec:solutions} for specific cases)
    \STATE Retract: $x_{t+1} = R_{x_t}(-\eta\,\xi_t^*)$
\ENDFOR
\RETURN $x_T$
\end{algorithmic}
\end{algorithm}

\section{Convergence}\label{sec:convergence}

In this section, we discuss the convergence analysis of our algorithm {intrinsic Muon} (\textbf{iMuon}), \Cref{alg:lmo} in both deterministic and stochastic settings.
Our convergence analysis is parametrized by a  geometric quantity $C_\varphi$, which captures the \textit{size} of the constraint set in (\ref{eqn:main_lmo}), defined as follows.
\begin{definition}\label{def:effdim}
The squared Riemannian radius of the $\varphi$-ball on $T_x\M$ is
$C_\varphi(x)\coloneqq\max\{\|\xi\|_x^2: \xi\in T_x\M, \varphi(G_x^{1/2}\xi)\leq 1\}$. Also define $C_{\varphi} \coloneqq \sup_{x \in \M} C_{\varphi}(x)$.
\end{definition}
\noindent For $\varphi=\|\cdot\|_2$: $C_\varphi=2r$ on $\M_r$, $n$ on $S_{++}^n$, $2\lfloor r/2\rfloor+\min(m{-}r,r)$ on $\text{St}(m,r)$, $\min(m{-}r,r)$ on $\text{Gr}(m,r)$. For $\varphi\in\{\|\cdot\|_F,\|\cdot\|_*\}$, $C_\varphi=O(1)$ on all manifolds.
On all four manifolds, $C_\varphi$ depends only on the manifold dimensions, not on the iterates.
\begin{assumption}[Retraction-$L$-smoothness~{\citep[Assumption 4.3]{boumal2023intromanifolds}}]\label{assum:smooth}
There exist constants $L>0$ and $\overline{\eta}>0$ such that for all $x\in\M$, all $\xi\in T_x\M$, and all $\eta\in(0,\overline{\eta}]$, $f(R_x(-\eta\xi))\le f(x)-\eta\,g_x(\xi,\textnormal{grad}\,f)+\frac{L\eta^2}{2}\|\xi\|_{x}^2$.
\end{assumption}

The next theorem shows that our algorithm iMuon converges at rate $O(1/\sqrt{T})$ with the constant determined by $C_{\varphi}$.

\begin{theorem}[Deterministic iMuon]\label{thm:general}
Under \Cref{assum:smooth}, with constant step size $\eta=c/(\tau \sqrt{T})$ where $c=\sqrt{2\Delta_0/(LC_\varphi)}$ (assuming $\eta\le\overline{\eta}$), our Algorithm~\ref{alg:lmo} achieves
$\min_{t\le T}\;\varphi^\circ(H_t) \;\le\; \sqrt{{2LC_\varphi\Delta_0}/{T}}$,
where $\Delta_0=f(x_0)-f^*$, $\varphi^\circ$ is the dual norm of $\varphi$, and $H_t=G_{x_t}^{1/2}(\textnormal{grad}\,f(x_t))$.
With a horizon-free decaying schedule defined as $\eta_t=c/(\tau\sqrt{t+1})$, Algorithm~\ref{alg:lmo} achieves: $\min_{t\le T}\varphi^\circ(H_t) \le \frac{\Delta_0(2+\ln T)}{2c(\sqrt{T}-1)}$.
\end{theorem}
The proof is provided in  \Cref{app:deterministic_proof}. We note that the convergence criterion $\varphi^\circ(H_t)\to 0$ implies $\|\text{grad}\,f(x_t)\|_{x_t}\to 0$ (first-order stationarity) as $H_t=G_{x_t}^{1/2}(\text{grad}\,f(x_t))$ and $G_{x_t}^{1/2}\succ 0$.
For $\varphi=\|\cdot\|_F$, the rate recovers standard Riemannian gradient descent with $C_\varphi=1$.
Our key technical ingredient is the uniform bound $\|\xi_t^*\|_{{x_t}}^2\le C_\varphi\tau^2$ (\Cref{lem:lmo_dual}). We also note manifold specific convergence rates below for constant step size setting.

\begin{remark}[Fixed-rank]\label{cor:lowrank}
On $\M_r$ with $\varphi=\|\cdot\|_2$, our rate is $O(\sqrt{Lr/T})$, depending only on the rank $r$ and independent of the factor norms $\|B\|,\|A\|$. In contrast, \citet{kang2026spectral} show in continuous time that, from almost every initialization, the trajectory either converges to a global minimum or its factor norms diverge, and they obtain unconditional global convergence only with $\ell_2$ regularization. \citet{janson2026stabilizing} instead apply an iterate-dependent spectral renormalization at runtime.
\end{remark}

\begin{remark}[SPD, Stiefel, Grassmann]\label{cor:others} Our convergence rates are $O(\sqrt{Ln/T})$ on $S_{++}^n$,
$O(\sqrt{L(r+\min(m{-}r,r))/T})$ on $\text{St}(m,r)$, and
$O(\sqrt{L\min(m{-}r,r)/T})$ on $\text{Gr}(m,r)$.
\end{remark}
In the stochastic setting, stochastic gradients ($\widetilde{\nabla f}$) are available and we appropriately employ the Riemannian stochastic gradient ($\widetilde{\text{grad}}\,f$) in the proposed LMO~\eqref{eqn:main_lmo} to obtain the update direction  $\tilde{\xi}_t^*$.

\begin{assumption}\label{assum:noise} There exists $\sigma_\varphi > 0$ such that the stochastic scaled gradient $\tilde{H}_t:=G_{x_t}^{1/2}(\widetilde{\textnormal{grad}}\,f_t)$ satisfies
$\mathbb{E}[\varphi^\circ(\tilde{H}_t - H_t)^2 \,|\, x_0,\ldots,x_t] \le \sigma_\varphi^2$,
where $H_t:=G_{x_t}^{1/2}(\textnormal{grad}\,f_t)$ is the true scaled gradient.
\end{assumption}
For $\varphi=\|\cdot\|_F$, this reduces to the standard Riemannian bounded-variance condition~\citep{boumal2023intromanifolds,mokhtari2020stochastic}.

\begin{theorem}[Stochastic iMuon]\label{thm:stochastic}
Under Assumptions~\ref{assum:smooth}~\&~\ref{assum:noise}, with step size $\eta_t=c/(\tau\sqrt{t+1})$ where $c=\sqrt{2\Delta_0/(LC_\varphi)}$ (assuming $\eta_t\le\overline{\eta}$), our Algorithm~\ref{alg:lmo} achieves $\min_{t\le T}\;\mathbb{E}[\varphi^\circ(H_t)]\;\le\; \sqrt{LC_\varphi\Delta_0/2}(2+\ln T)/\sqrt{T} + 2\sigma_\varphi$,
recovering \Cref{thm:general} (up to the $\ln T$ factor from the decaying schedule) when $\sigma_\varphi=0$.
\end{theorem}
The proof is provided in  \Cref{app:stochastic_proof}. The additive $2\sigma_\varphi$ arises from the lack of variance reduction~\citep{zhang2016riemannian,kasai2018riemannian,sato2019riemannian} in the gradient estimator and can be removed with gradient averaging or momentum techniques~\citep{mokhtari2020stochastic,shen2019complexities}, which we leave to future work.

Several prior results arise as instances of our framework (see \Cref{app:connections} for details): Euclidean spectral steepest descent~\citep{bernstein2025old} ($G_x=I$, $C_\varphi=\min(p,q)$), standard Riemannian GD~\citep{boumal2023intromanifolds} ($\varphi=\|\cdot\|_F$, $C_\varphi=1$), factor-wise Muon~\citep{kang2026spectral} (the coupled metric removes its factor-norm dependence, $C_\varphi=2r$), the rank-$1$ special case of NuMuon~\citep{dolatabadi2026numuon} ($\varphi=\|\cdot\|_*$, $C_\varphi=1$), and general Euclidean LMOs~\citep{pethick2025lmo} ($G_x=I$).

\section{Experiments}\label{sec:experiments}

We evaluate the proposed intrinsic LMO on two manifolds:
(i) the fixed-rank manifold $\M_r$ via LoRA fine-tuning of GPT-2 Medium on E2E NLG and Mistral-7B on GLUE, and (ii) the SPD manifold $S_{++}^n$ via SPD covariance classification on CIFAR-100. In LoRA fine-tuning, we compare iMuon with factor-wise Euclidean Muon~\citep{jordan2024muon,bernstein2025old} (as in~\eqref{eqn:factorwise_muon}), SGD (factor-wise with Frobenius norm), Scaled GD~\citep{zhang24riemannianlora} (preconditioned), Spectron~\citep{janson2026stabilizing}, and Riemannion~\citep{bogachev2026riemannion}. With momentum, SGD and Scaled GD become AdamW and Scaled AdamW. In the SPD setting, we include the Frobenius and nuclear instances of our framework (RGD, iMuon-Nu) and their Euclidean counterparts (EGD, NuMuon). Full details and additional experiments including those on Stiefel and Grassmann manifolds are deferred to \Cref{app:experimental_details,app:fixed_rank_experiments,app:spd_experiments,app:grassmann_experiments,app:stiefel_cifar_subcenter}. The code to reproduce our experiments is available at \url{https://github.com/1bang118/manifold-intrinsic-muon}.

\paragraph{LoRA fine-tuning ($\M_r$).}
We fine-tune GPT-2 Medium on E2E NLG~\citep{novikova2017e2e} with LoRA rank $r=4$, and 4-bit quantized Mistral-7B on GLUE~\citep{wang2019glue} with LoRA rank $r=16$. We follow the LoRA setup of~\citet{hu2022lora}, tuning the learning rate per method via grid search. We also evaluate a momentum variant of iMuon that transports the momentum buffer by horizontal projection (\Cref{app:momentum}).

\Cref{tab:exp_e2e_main,tab:exp_glue_main} report the five standard E2E metrics and the GLUE per-task and average scores, separately for methods without and with momentum. The momentum-free block isolates the LMO direction itself, since momentum schemes differ across methods and accumulate gradients over iterations. On E2E NLG, iMuon without momentum is best on BLEU, NIST, and ROUGE-L and ties factor-wise Muon for the best CIDEr, and with momentum it is best or second best on all five metrics. Momentum improves its BLEU and METEOR and leaves the other metrics nearly unchanged. Across both blocks, iMuon attains the best BLEU, NIST, ROUGE-L, and (tied) CIDEr. On GLUE, the advantage without momentum is more pronounced: iMuon has the best average. With momentum, its average ($88.76$) is second only to Scaled AdamW ($89.01$). Across both blocks, iMuon attains the best SST-2 and STS-B scores and ties for the best WNLI.
We attribute iMuon's edge over Spectron in both tables (e.g., GLUE average $88.15$ vs.\ $87.53$ without and $88.76$ vs.\ $87.58$ with momentum) to geometry: Spectron's scalar rescaling bounds the norm of the ambient update, but the ambient update still changes under the factor rescaling $(B,A)\mapsto(cB,A/c)$ that leaves $X=BA$ unchanged (\Cref{app:compare_lowrank}), whereas iMuon's $\textnormal{GL}(r)$-invariant metric makes its ambient update independent of the factorization (\Cref{prop:symmetry}).

\begin{table}[t]
\centering
\caption{E2E NLG results with GPT-2 Medium and LoRA rank $r=4$, reported separately for methods without and with optimizer momentum. For each method, we report the run with the best test-set BLEU. The remaining metrics are computed from the same run. Within each block, the best value in each column is in bold and the second best is underlined.}
\label{tab:e2e-main}\label{tab:exp_e2e_main}
\begin{tabular}{lccccc}
\toprule
Method & BLEU & NIST & METEOR & ROUGE-L & CIDEr \\
\midrule
\multicolumn{6}{c}{\textit{Without momentum}}\\
SGD & 66.60 & 8.54 & 44.20 & 68.20 & 2.32 \\
Scaled GD & 69.20 & 8.71 & 46.30 & 70.90 & 2.48 \\
Factor-wise Muon & 70.02 & \underline{8.81} & 46.77 & 71.68 & \textbf{2.53} \\
Riemannion (SGD) & 70.02 & 8.78 & \underline{46.79} & \underline{71.99} & \underline{2.52} \\
Spectron & \underline{70.25} & \underline{8.81} & \textbf{46.85} & 71.81 & 2.51 \\
iMuon & \textbf{70.30} & \textbf{8.85} & 46.62 & \textbf{72.03} & \textbf{2.53} \\
\midrule
\multicolumn{6}{c}{\textit{With momentum}}\\
AdamW & 68.90 & 8.69 & 46.50 & 71.30 & 2.51 \\
Scaled AdamW & 69.60 & \underline{8.77} & 46.60 & 71.80 & \underline{2.52} \\
Factor-wise Muon & \underline{70.26} & \textbf{8.83} & 46.76 & 71.75 & \underline{2.52} \\
Riemannion (Adam) & 70.21 & \textbf{8.83} & 46.55 & \underline{71.87} & \textbf{2.53} \\
Spectron & 69.84 & \underline{8.77} & \textbf{46.98} & 71.84 & 2.50 \\
iMuon & \textbf{70.41} & \textbf{8.83} & \underline{46.84} & \textbf{71.96} & \underline{2.52} \\
\bottomrule
\end{tabular}
\end{table}

\begin{table*}[t]
\centering
\caption{GLUE results with 4-bit NF4-quantized Mistral-7B and LoRA rank $r=16$, reported separately without and with optimizer momentum. Avg.\ is the unweighted average over the nine tasks. Within each block, the best value in each column is in bold and the second best is underlined.}
\label{tab:glue-main}\label{tab:exp_glue_main}
\resizebox{\textwidth}{!}{\begin{tabular}{lcccccccccc}
\toprule
Method & MNLI & SST-2 & MRPC & CoLA & QNLI & QQP & RTE & STS-B & WNLI & Avg. \\
\midrule
\multicolumn{11}{c}{\textit{Without momentum}}\\
SGD & 88.15 & 96.10 & 70.10 & 55.89 & 94.22 & 88.59 & 50.90 & 47.64 & 49.30 & 71.21 \\
Scaled GD & \underline{90.21} & \underline{96.90} & 81.62 & 68.17 & 94.40 & \textbf{91.15} & 54.15 & 90.31 & 56.34 & 80.36 \\
Factor-wise Muon & 89.77 & \textbf{97.25} & 82.35 & \underline{69.32} & 94.47 & 88.42 & 84.12 & 84.39 & 45.07 & 81.68 \\
Riemannion (SGD) & \textbf{91.55} & \underline{96.90} & 82.60 & 68.90 & \textbf{94.93} & 87.77 & 82.67 & 77.99 & 60.56 & 82.65 \\
Spectron & 89.97 & 96.79 & \textbf{90.69} & 66.62 & 94.34 & \underline{88.84} & \underline{87.73} & \underline{91.12} & \textbf{81.69} & \underline{87.53} \\
iMuon & 90.01 & \underline{96.90} & \underline{89.71} & \textbf{71.49} & \underline{94.67} & 88.31 & \textbf{89.89} & \textbf{92.05} & \underline{80.28} & \textbf{88.15} \\
\midrule
\multicolumn{11}{c}{\textit{With momentum}}\\
AdamW & 89.86 & 96.79 & 88.48 & 71.05 & 94.42 & 91.24 & \underline{90.61} & 90.42 & \underline{81.69} & 88.28 \\
Scaled AdamW & 90.68 & \underline{97.25} & \underline{89.46} & 71.30 & 94.67 & \textbf{92.22} & \textbf{91.34} & 91.10 & \textbf{83.10} & \textbf{89.01} \\
Factor-wise Muon & 91.29 & 96.90 & \textbf{89.71} & \underline{71.45} & \underline{94.97} & \underline{91.74} & 89.89 & \underline{92.18} & 78.87 & 88.56 \\
Riemannion (Adam) & 91.21 & 96.44 & 79.90 & 67.00 & 94.55 & 90.36 & 80.87 & 62.13 & \textbf{83.10} & 82.84 \\
Spectron & \textbf{91.60} & 97.13 & 87.25 & \textbf{72.30} & \textbf{95.28} & 91.43 & 88.81 & 91.22 & 73.24 & 87.58 \\
iMuon & \underline{91.50} & \textbf{97.59} & 88.48 & 71.18 & 94.71 & 91.65 & 88.09 & \textbf{92.52} & \textbf{83.10} & \underline{88.76} \\
\bottomrule
\end{tabular}}
\end{table*}

\paragraph{SPD prototype classification ($S_{++}^n$).}

Covariance matrices of local image features have been used as SPD descriptors since region covariances~\citep{tuzel2006region}, e.g., for pedestrian detection~\citep{tuzel2008pedestrian} and in SPD neural networks~\citep{huang2017riemannian}. We classify CIFAR-100 images into the $20$ coarse classes~\citep{krizhevsky2009learning} using shrinkage covariance descriptors $C_i\in S_{++}^{32}$~\citep{ledoit2004well} of frozen ResNet-18 \texttt{layer3} features~\citep{he2016deep}. As a Riemannian analog of nearest-class-mean classification, we learn one SPD prototype $X_c$ per class under the affine-invariant (AI) metric~\citep{pennec2006riemannian,bhatia2009positive} by minimizing the cross-entropy of the logits $-\beta\,d_{\mathrm{AI}}(C_i,X_c)^2$, regularized toward the log-Euclidean~\citep{arsigny2007geometric} class-mean initialization (\Cref{app:spd_experiments}). Since the AI metric is not Euclidean ($G_X\neq I$), the intrinsic and Euclidean LMOs differ for every norm, so we compare all three pairs: Frobenius (RGD vs.\ EGD), spectral (iMuon vs.\ Muon), and nuclear (iMuon-Nu vs.\ NuMuon, here the pure nuclear-norm LMO~\citep{dolatabadi2026numuon}, i.e., \eqref{eqn:muon_lmo} with the nuclear norm in place of the spectral norm).

\Cref{fig:exp_spd_trajectories} shows the test-accuracy trajectories. At the end of training, the intrinsic method improves the mean test accuracy in every pair: RGD $0.548\pm 0.002$ vs.\ EGD $0.523\pm 0.026$ (Frobenius), iMuon $0.475\pm 0.009$ vs.\ Muon $0.407\pm 0.016$ (spectral), and iMuon-Nu $0.546\pm 0.000$ (std below $0.0005$) vs.\ NuMuon $0.422\pm 0.032$ (nuclear), over three seeds at the validation-selected learning rate. For the Frobenius pair, this is consistent with the known advantages of the AI geometry~\citep{pennec2006riemannian,bhatia2009positive}. The spectral and nuclear pairs show that these advantages carry over to LMO directions. The gain grows from Frobenius ($+0.025$) to spectral ($+0.068$) to nuclear ($+0.124$). One possible reason for the large nuclear gain is that the rank-$1$ Euclidean step places the whole step on a single direction regardless of the scale of $X$ along it, whereas the metric scaling $X^{1/2}(\cdot)X^{1/2}$ adapts the step to $X$.

\begin{figure}[t]
    \centering
    \includegraphics[scale=0.50]{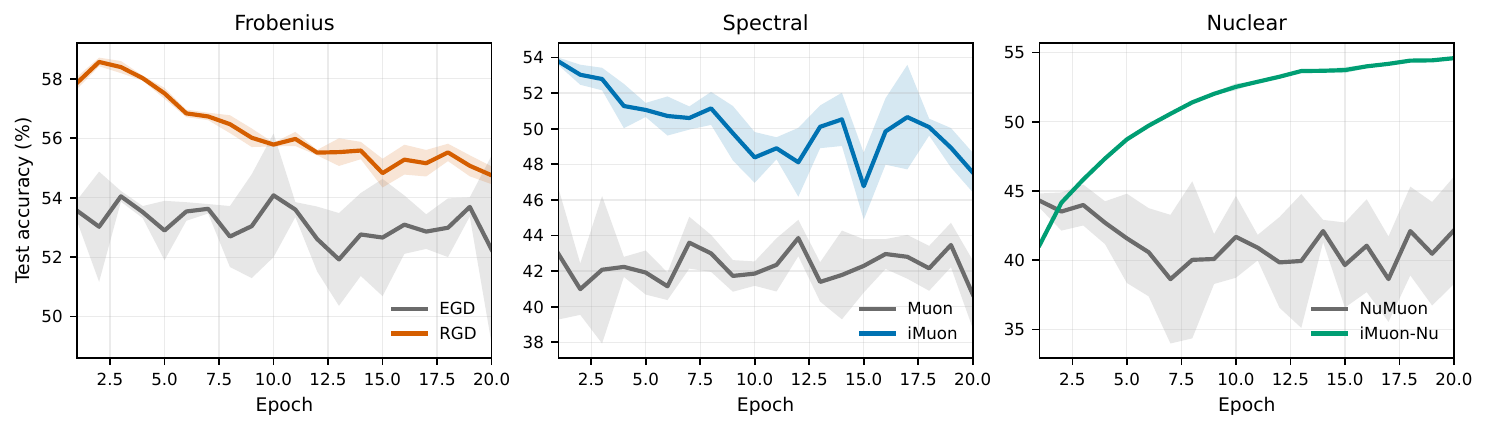}
    \caption{SPD classification on CIFAR-100 ($S_{++}^{32}$, $20$ coarse classes). Each panel pairs a Euclidean LMO with its intrinsic counterpart under a common norm: Frobenius (EGD vs.\ RGD), spectral (Muon vs.\ iMuon), and nuclear (NuMuon vs.\ iMuon-Nu). Curves show mean test accuracy with $\pm 1$ std bands over three seeds at the validation-selected learning rate. From the second epoch on, the intrinsic method is above its Euclidean counterpart in every pair, and the final gap widens from Frobenius to spectral to nuclear.}
    \label{fig:exp_spd_trajectories}
\end{figure}

\section{Conclusion}
We introduced iMuon, a unified framework for unitarily invariant norm-constrained optimization on Riemannian matrix manifolds. By lifting the ambient norm through the Riemannian metric, iMuon resolves the symmetry breaking (O1) and tangent space and norm constraints coupling (O2) that obstruct prior manifold extensions of Muon, and admits closed-form LMO solutions on the fixed-rank, SPD, Stiefel, and Grassmann manifolds. The constant in the resulting convergence rate depends only on the manifold dimension. On the fixed-rank manifold this is independent of factor conditioning, removing the bounded-factor-norm assumption and runtime rescaling required by prior work. Empirically, iMuon performs favorably across LoRA fine-tuning and other learning problems on the fixed-rank, SPD, Stiefel, and Grassmann manifolds, with the benefit of the intrinsic geometry most pronounced when the underlying metric is non-trivial. We note that this advantage is established in the momentum-free setting, which isolates the update direction. With momentum, iMuon is competitive with existing momentum-based optimizers on LoRA fine-tuning but does not uniformly outperform adaptive baselines such as Scaled AdamW. Developing momentum tailored to the intrinsic LMO, along with variance-reduced variants and adaptive choice of $\varphi$, remains future work. Limitations are discussed in \Cref{app:limitations}.

\section*{Acknowledgments and Disclosure of Funding}
PJ acknowledges the SBI Foundation Hub for Data \& Analytics grant and IIT Bombay seed grant.
CM acknowledges the support of National Research Foundation of Singapore under its NRF Fellowship Programme Award No.: NRF-NRFF18-2026-0010.

\bibliographystyle{plainnat}
\bibliography{refs}

\newpage

\appendix

\begin{center}
    \Huge{Appendix}
\end{center}

\[\qquad\]

\renewcommand{\contentsname}{Appendix Table of Contents}
\startcontents[appendix]
\printcontents[appendix]{}{1}{\setcounter{tocdepth}{2}}

\section{Broader Societal Impact}\label{app:broader_impact}

This paper presents foundational research on optimization algorithms on Riemannian matrix manifolds. The contribution is methodological and theoretical. The framework is not tied to a particular application or deployment, and our empirical evaluation uses standard publicly available benchmarks. As with any optimizer, the downstream impact depends on the objective and the data it is applied to. We do not identify a direct path to negative societal impact beyond the general-purpose effect of more efficient training.

\section{Limitations}\label{app:limitations}
Our closed-form solutions rely on two structural properties. The first is the left-right form of the metric, which gives symmetry preservation (\Cref{prop:symmetry}). The second is singular value invariance or $\varphi$-contractivity of the scaled tangent space, which gives the closed-form reduction (\Cref{prop:tangent_compat}). Both properties hold on the four manifolds we treat but not in general. \Cref{app:other_manifolds} discusses cases where one of them fails. The LMO remains a convex problem in such cases but requires an iterative solver. On $\textnormal{St}(m,r)$, we further use an $\ell_\infty$-product norm relaxation of the joint norm, which enlarges $C_\varphi$ (\Cref{app:compare_stiefel}). On $\M_r$, we use the $\ell_\infty$-product norm over the two factor blocks, for which no joint unitarily invariant norm is defined (\Cref{tab:per_manifold}). Extending the framework to tensor manifolds, such as the Tucker manifold~\citep{kasai2016low}, is left to future work. Our convergence analysis does not incorporate momentum or acceleration~\citep{zhang2018estimate,han2023riemannian}. The stochastic rate carries an additive $2\sigma_\varphi$ term, and the momentum variant used in our experiments (\Cref{app:momentum}) is a practical construction without a convergence guarantee.

Empirically, the advantage of iMuon is pronounced in the no-momentum setting, while with momentum, iMuon is competitive  (\Cref{tab:exp_e2e_main,tab:exp_glue_main}). Our experiments cover LoRA fine-tuning of GPT-2 Medium and 4-bit Mistral-7B, and learning problems on the fixed-rank, SPD, Stiefel, and Grassmann manifolds. Frontier-scale LLM pretraining, where Muon is most heavily used, is not explored in this work.

\section{Closed-Form LMO Solutions via Norm Duality}\label{app:primal_dual}

We derive the closed-form solution of the intrinsic LMO for a general unitarily invariant norm $\varphi$.

\subsection{Reduction to a vector problem}

The key tool is von Neumann's trace inequality, which reduces the matrix LMO to a vector problem over singular values.

\begin{lemma}\label{lem:vonneumann}
Let $H=U\diag(\sigma_1,\ldots,\sigma_{\min(p,q)})V^\top$ with $\sigma_1\ge\cdots\ge\sigma_{\min(p,q)}\ge 0$, and for a vector $z$ write $\varphi(z):=\varphi(\diag(z))$, with $\diag(z)$ the $p\times q$ rectangular diagonal matrix.
Then $\max_{\varphi(Z)\le\tau}\inner{Z}{H} = \tau\,\varphi^\circ(\sigma)$,
where $\varphi^\circ$ is the dual norm and $\sigma=(\sigma_1,\ldots,\sigma_{\min(p,q)})$.
A maximizer has the form $Z^*=U\diag(z_1^*,\ldots,z_{\min(p,q)}^*)V^\top$, where $z^*$ solves the vector problem $\max_{\varphi(z)\le\tau}\inner{z}{\sigma}$.
\end{lemma}
\begin{proof}
By von Neumann's trace inequality~\citep[Thm.~7.4.1.1]{horn2012matrix}, $\inner{Z}{H}\le\sum_i\sigma_i(Z)\sigma_i(H)$ with equality when $Z$ shares the singular vectors of $H$. Hence $Z^*=U\diag(z^*)V^\top$ for some non-negative $z^*$ with $\varphi(z^*)\le\tau$, reducing the problem to $\max_{\varphi(z)\le\tau}\inner{z}{\sigma}=\tau\,\varphi^\circ(\sigma)$ by definition of the dual norm.
\end{proof}

\subsection{Spectral, Frobenius, and nuclear norms}

The vector problem $\max_{\varphi(z)\le\tau}\inner{z}{\sigma}$ has closed-form solutions for the three standard norms:

\emph{Spectral norm} ($\varphi(z)=\|z\|_\infty=\max_i|z_i|$).
Each $z_i$ is independently bounded by $\tau$, and $\sigma_i\ge 0$, so $z_i^*=\tau$ for all $i$.
Thus $Z^*=\tau\,UV^\top=\tau\,\Ortho(H)$ and the optimal value is $\tau\sum_i\sigma_i=\tau\|H\|_*$. Entries with $\sigma_i=0$ do not affect the objective and may be set to zero, which gives the thin polar factor used in \Cref{prop:tangent_compat}.
The dual norm is $\varphi^\circ=\|\cdot\|_1$ on singular values, i.e., the nuclear norm.

\emph{Frobenius norm} ($\varphi(z)=\|z\|_2$).
By Cauchy--Schwarz, $z^*=\tau\,\sigma/\|\sigma\|_2$, giving $Z^*=\tau\,H/\|H\|_F$.
The optimal value is $\tau\|H\|_F$.
Its dual norm is $\varphi^\circ=\|\cdot\|_F$, i.e., it is self-dual.

\emph{Nuclear norm} ($\varphi(z)=\|z\|_1=\sum_i|z_i|$).
Under $\sum_i z_i\le\tau$ with $z_i\ge 0$, all of the norm bound goes to the largest $\sigma_i$: $z^*=\tau\,e_1$.
Writing $u_i$ and $v_i$ for the $i$-th columns of $U$ and $V$,
we get $Z^*=\tau\,u_1v_1^\top$ and the optimal value is $\tau\sigma_1=\tau\|H\|_2$. If exactly $k>1$ singular values of $H$ equal $\sigma_1$,
the choice required in \Cref{prop:tangent_compat} is $Z^*=\frac{\tau}{k}\sum_{i=1}^{k}u_iv_i^\top$, which does not depend on the choice of singular vectors for $\sigma_1$.
The dual norm is $\varphi^\circ=\|\cdot\|_\infty$, i.e., the spectral norm.

For a general unitarily invariant norm $\varphi$, norm duality gives  $\max_{\varphi(z)\le\tau}\inner{z}{\sigma}=\tau\,\varphi^\circ(\sigma)$, and the maximizers are the elements of $\tau\,\partial\varphi^\circ(\sigma)$. Computing a maximizer thus amounts to a convex problem in $\min(p,q)$ variables, and \Cref{app:other_norms} gives closed forms for several other norms.

\subsection{Other unitarily invariant norms}\label{app:other_norms}

The spectral and nuclear norms are the two extremes of this family (for $\varphi$ normalized so that $\varphi(e_1)=1$, $\|z\|_\infty\le\varphi(z)\le\|z\|_1$), with the Frobenius norm in between.
Here we note two natural interpolations that also admit closed-form LMOs within our framework. Throughout, $m:=\min(p,q)$ denotes the number of singular values of $H$.

\noindent
\paragraph{\texorpdfstring{Partial polar $k$-norms.}{Partial polar k-norms.}}
Let $\|Z\|_{(k)}:=\sum_{i=1}^k\sigma_i(Z)$ denote the Ky Fan $k$-norm~\citep[\S3.4]{horn2012matrix}, the sum of the $k$ largest singular values. We take as the constraint norm its dual,
\[
\varphi(Z)=\|Z\|_{(k)}^{\circ}=\max\bigl(\|Z\|_2,\;\|Z\|_*/k\bigr),
\qquad k\in\{1,\ldots,m\},
\]
which is unitarily invariant and interpolates between the nuclear norm at $k=1$, where $\|Z\|_*\ge\|Z\|_2$ gives $\varphi=\|\cdot\|_*$, and the spectral norm at $k=m$, where $\|Z\|_*/m\le\|Z\|_2$ gives $\varphi=\|\cdot\|_2$.
The LMO has a clean closed form: the vector problem $\max_{\varphi(z)\le\tau}\inner{z}{\sigma}$ is solved by $z_i^*=\tau$ for $i\le k$ and $z_i^*=0$ for $i>k$, giving
\[
Z^* = \tau\sum_{i=1}^k u_i v_i^\top,
\]
the \emph{rank-$k$ partial polar factor} of $H$. If a tied block $\sigma_j=\cdots=\sigma_{j+\ell-1}$ straddles the cutoff $k$, the symmetric choice required by \Cref{prop:tangent_compat} spreads the remaining budget $(k-j+1)\tau$ equally over the block, $z_i^*=\tau(k-j+1)/\ell$ for $j\le i\le j+\ell-1$, which attains the same optimal value and is independent of the choice of singular vectors within the block. The dual norm is $\varphi^{\circ}=\|\cdot\|_{(k)}$, the Ky Fan $k$-norm itself, and the squared Riemannian radius is $C_\varphi=k$.
Thus \Cref{thm:general} yields rate $O(\sqrt{Lk\Delta_0/T})$ with convergence in $\|H_t\|_{(k)}$.
Setting $k$ between $1$ and $m$ provides a tunable trade-off between aggressiveness and rank: $k=1$ concentrates on the leading singular direction (conservative, $C_\varphi=1$), while $k=m$ equalizes all $m$ directions (aggressive, $C_\varphi=m$).

\noindent
\paragraph{Schatten $s$-norms.}
The Schatten $s$-norm $\|Z\|_{\mathcal{S}_s}:=(\sum_i\sigma_i(Z)^s)^{1/s}$ for $s\in[1,\infty]$ is unitarily invariant and continuously interpolates between the nuclear ($s=1$), Frobenius ($s=2$), and spectral ($s=\infty$) norms.
The dual norm is the Schatten $s^*$-norm with $1/s+1/s^*=1$, so the LMO $\max_{\|z\|_s\le\tau}\inner{z}{\sigma}$ is solved by H\"{o}lder's inequality:
\[
z_i^* = \tau\,\frac{\sigma_i^{s^*-1}}{\|\sigma^{s^*-1}\|_s}, \qquad Z^* = U\diag(z^*)V^\top,
\]
where $\sigma^{s^*-1}$ denotes the vector $(\sigma_1^{s^*-1},\ldots,\sigma_m^{s^*-1})$.
This is a \emph{power-law reweighting} of the singular values. For $s=2$ ($s^*=2$), $z_i^*\propto\sigma_i$ (proportional, recovering Frobenius). For $s\to\infty$ ($s^*\to 1$), $z_i^*\to\tau$ (all equal, recovering spectral). For $s=1$ ($s^*=\infty$), all budget concentrates on the largest $\sigma_i$ (recovering nuclear).
Intermediate values of $s$ give a ``soft Muon'' that partially equalizes singular values without the hard cutoff of the partial polar norm.

The squared Riemannian radius is
\[
C_\varphi = \begin{cases} 1 & s\le 2, \\ m^{1-2/s} & s\ge 2. \end{cases}
\]
For $s\le 2$, the maximizer of $\|z\|_2^2$ subject to $\|z\|_s\le 1$ is a unit vector ($C_\varphi=1$). For $s>2$, it is the uniform vector $z=(m^{-1/s},\ldots,m^{-1/s})$ giving $C_\varphi=m^{1-2/s}$.
Thus \Cref{thm:general} yields rate $O(\sqrt{L\cdot\max(1,m^{1-2/s})\cdot\Delta_0/T})$, smoothly interpolating between the $O(\sqrt{L/T})$ rate of Frobenius/nuclear norms and the $O(\sqrt{Lm/T})$ rate of the spectral norm.
In practice, choosing $s$ slightly above $2$ gives a mild singular value equalization (``softer'' than full Muon) at a reduced effective-dimension cost.

\section{Proofs of Main Results}\label{app:proof_convergence}

\subsection{Proof of \Cref{prop:symmetry}}\label{app:proof_symmetry}

\begin{proof}[Proof of \Cref{prop:symmetry}]
The map $\xi\mapsto P_x\,\xi\,Q_x$ is self-adjoint and positive definite on the ambient matrix space, and it maps $T_x\M$ into itself (it equals $G_x$ there). Hence, it also maps the orthogonal complement of $T_x\M$ into itself. Consequently, its square root $\xi\mapsto P_x^{1/2}\xi\,Q_x^{1/2}$ maps $T_x\M$ into itself and coincides there with $G_x^{1/2}$, i.e., $G_x^{1/2}\xi=P_x^{1/2}\xi\,Q_x^{1/2}$; the same holds at $\tilde{x}$. Let $Z:=G_x^{1/2}\xi$. Then
\begin{equation*}
G_{\tilde{x}}^{1/2}\tilde{\xi}=P_{\tilde{x}}^{1/2}L_h\,\xi\,R_h^\top Q_{\tilde{x}}^{1/2}=U_h\,Z\,V_h^\top,
\qquad U_h:=P_{\tilde{x}}^{1/2}L_hP_x^{-1/2},\quad V_h:=Q_{\tilde{x}}^{1/2}R_hQ_x^{-1/2}.
\end{equation*}
By the transformation law, $U_h^\top U_h=P_x^{-1/2}L_h^\top P_{\tilde{x}}L_hP_x^{-1/2}=P_x^{-1/2}P_xP_x^{-1/2}=I$, and similarly $V_h^\top V_h=I$. Since $U_h$ and $V_h$ are square, they are orthogonal, so $\sigma_i(U_hZV_h^\top)=\sigma_i(Z)$ for all $i$. Hence $\varphi(G_{\tilde{x}}^{1/2}\tilde{\xi})=\varphi(G_x^{1/2}\xi)$ for every unitarily invariant $\varphi$. With $\varphi=\|\cdot\|_F$, this gives $g_{\tilde{x}}(\tilde{\xi},\tilde{\xi})=\|G_{\tilde{x}}^{1/2}\tilde{\xi}\|_F^2=\|G_x^{1/2}\xi\|_F^2=g_x(\xi,\xi)$, so $g$ is $\mathcal{G}$-invariant by polarization. It remains to verify the hypotheses in the three cases.

\emph{Fixed-rank manifold.} The action $(B,A)\mapsto(BN^{-1},NA)$ acts separately on the two factor blocks of $T_x\M_r=\RR^{m\times r}\times\RR^{r\times n}$, which are full matrix spaces. For the $\dot{B}$-block, $\dot{B}\mapsto\dot{B}N^{-1}$ gives $L_h=I_m$ and $R_h=N^{-\top}$, with $P=I_m$ and $Q=AA^\top\mapsto NAA^\top N^\top=R_h^{-\top}Q\,R_h^{-1}$. For the $\dot{A}$-block, $\dot{A}\mapsto N\dot{A}$ gives $L_h=N$ and $R_h=I_n$, with $Q=I_n$ and $P=B^\top B\mapsto N^{-\top}B^\top B\,N^{-1}=L_h^{-\top}P\,L_h^{-1}$. The conclusion therefore holds for each block, and hence for the product norm $\max\{\varphi(\dot{B}(AA^\top)^{1/2}),\varphi((B^\top B)^{1/2}\dot{A})\}$.

\emph{SPD manifold.} $T_XS_{++}^n=\text{Sym}(n)$ is mapped into itself by $\xi\mapsto X^{-1}\xi X^{-1}$. The congruence $X\mapsto NXN^\top$ gives $L_h=R_h=N$, and $P=Q=X^{-1}\mapsto(NXN^\top)^{-1}=N^{-\top}X^{-1}N^{-1}$.

\emph{Grassmann manifold.} On the horizontal space $\{\xi:X^\top\xi=0\}$, the metric is Euclidean, $P=I_m$ and $Q=I_r$, and $O\in O(r)$ acts by $\xi\mapsto\xi O$, so $L_h=I_m$ and $R_h=O^\top$. The law holds because $R_h^{-\top}R_h^{-1}=O^\top O=I_r$.
\end{proof}

\subsection{Proof of \Cref{prop:tangent_compat}}\label{app:proof_reduction}

\begin{proof}[Proof of \Cref{prop:tangent_compat}]
Substitute $Z=G_x^{1/2}\xi$ in~\eqref{eqn:main_lmo}.
Since $g_x(\xi,\text{grad}\,f)=\inner{G_x^{1/2}\xi}{G_x^{1/2}(\text{grad}\,f)}=\inner{Z}{H}$ (by self-adjointness of $G_x^{1/2}$) and $\varphi(G_x^{1/2}\xi)=\varphi(Z)$, the LMO becomes $\max_{Z\in\mathcal{Z}_x,\,\varphi(Z)\le\tau}\inner{Z}{H}$, and its solution $\hat{Z}$ gives $\xi^*=G_x^{-1/2}\hat{Z}$.
Note that $H\in\mathcal{Z}_x$ since $\text{grad}\,f\in T_x\M$.
Write $H=U\Sigma V^\top$ (compact SVD with $\Sigma = \diag(\sigma)$ and $\sigma_1\ge\cdots\ge\sigma_r>0$).
By von Neumann's trace inequality~\citep[Thm.~7.4.1.1]{horn2012matrix}, $\inner{Z}{H}\le\sum_i\sigma_i(Z)\sigma_i(H)$ with equality when $Z$ shares singular vectors with $H$.
Thus the unconstrained problem $\max_{\varphi(Z)\le\tau}\inner{Z}{H}$ has a maximizer of the form $Z^*= U\diag(z^*)V^\top$ (\Cref{lem:vonneumann}), where $z^*$ solves $\max_{\varphi(z)\le \tau }\inner{z}{\sigma}$, with optimal value $\tau \varphi^\circ(\sigma)$ by norm duality.
We choose $z^*$ with $z^*_i=z^*_j$ whenever $\sigma_i=\sigma_j$. Such a maximizer exists: averaging any maximizer over permutations within each group of equal $\sigma_i$ preserves the objective and, by convexity and permutation invariance of $\varphi$, does not increase $\varphi$.

If $\mathcal{Z}_x$ is SV-invariant, then since $H\in\mathcal{Z}_x$, replacing $\Sigma$ with $\diag(z^*)$ preserves membership by \Cref{def:sv_invariant}: $Z^*\in\mathcal{Z}_x$.
Hence the unconstrained maximizer satisfies the $\mathcal{Z}_x$-constraint, so $\hat{Z}=Z^*=P_{\mathcal{Z}_x}(Z^*)$ solves the constrained problem.

If $\mathcal{Z}_x$ is $\varphi$-contractive, then $P_{\mathcal{Z}_x}(Z^*)\in\mathcal{Z}_x$ and $\varphi(P_{\mathcal{Z}_x}(Z^*))\le\varphi(Z^*)\le\tau$, so it is feasible. Since $P_{\mathcal{Z}_x}$ is self-adjoint and $H\in\mathcal{Z}_x$, $\inner{P_{\mathcal{Z}_x}(Z^*)}{H}=\inner{Z^*}{P_{\mathcal{Z}_x}(H)}=\inner{Z^*}{H}$. As the constrained maximum cannot exceed the unconstrained one, $\hat{Z}=P_{\mathcal{Z}_x}(Z^*)$ solves the constrained problem.
\end{proof}

\subsection{Proof of \Cref{cor:fixed_rank}}\label{app:proof_fixed_rank}

\begin{proof}[Proof of \Cref{cor:fixed_rank}]
On $\M_r$ with the coupled metric, $G_x^{1/2}(\dot{B},\dot{A})=(\dot{B}(AA^\top)^{1/2},\,(B^\top B)^{1/2}\dot{A})$ and $T_x\M_r=\RR^{m\times r}\times\RR^{r\times n}$.

\emph{Step 1: Left-right form and symmetry (\Cref{prop:symmetry}).}
The $\dot{B}$-block metric is $g^B(\dot{B}_1,\dot{B}_2)=\tr(\dot{B}_1^\top\dot{B}_2\,AA^\top)$, which has the left-right form with $P=I_m$ and $Q=AA^\top$.
Under $(B,A)\mapsto(BN^{-1},NA)$, $Q=AA^\top\mapsto NAA^\top N^\top$ and $\dot{B}\mapsto\dot{B}N^{-1}$; the metric is invariant since $\tr((\dot{B}N^{-1})^\top(\dot{B}N^{-1})NAA^\top N^\top)=\tr(\dot{B}^\top\dot{B}\,AA^\top)$.
By \Cref{prop:symmetry}, the intrinsic constraint $\varphi(\dot{B}(AA^\top)^{1/2})\le\tau$ is $\text{GL}(r)$-invariant.
The $\dot{A}$-block is symmetric (with $P=B^\top B$, $Q=I_n$).

\emph{Step 2: SV-invariance and reduction (\Cref{prop:tangent_compat}).}
The scaled tangent space decomposes as $\mathcal{Z}_x=\RR^{m\times r}\times\RR^{r\times n}$ (full ambient space on each block), which is trivially singular value invariant (\Cref{app:sv_invariance}).
With the $\ell_\infty$-product norm $\max(\varphi(Z_B),\varphi(Z_A))\le\tau$, the LMO~\eqref{eqn:main_lmo} decouples into independent $\dot{B}$ and $\dot{A}$ subproblems via~\eqref{eqn:whitened_lmo_product}.

\emph{Step 3: Closed form for the $\dot{B}$-block.}
The objective is $g_x(\xi,\text{grad}\,f)=\inner{\dot{B}A}{\nabla_X f}=\inner{\dot{B}}{\nabla_B f}$ where $\nabla_B f=\nabla_X f A^\top$.
Substituting $Z_B=\dot{B}(AA^\top)^{1/2}$ gives objective $\inner{Z_B}{H_B}$ with $H_B=\nabla_B f\,(AA^\top)^{-1/2}$.
By \Cref{prop:tangent_compat} (with SV-invariance from Step~2), the tangent-space constraint is redundant and the spectral-norm LMO reduces to the unconstrained problem $\max_{\|Z_B\|_2\le\tau}\inner{Z_B}{H_B}$.
By \Cref{lem:vonneumann}, $Z_B^*=\tau\,\Ortho(H_B)$, so $\dot{B}^*=\tau\,\Ortho(\nabla_B f\,(AA^\top)^{-1/2})\,(AA^\top)^{-1/2}$.

The $\dot{A}$-block follows by symmetry: $H_A=(B^\top B)^{-1/2}\nabla_A f$ with $\nabla_A f=B^\top\nabla_X f$, giving $\dot{A}^*=\tau\,(B^\top B)^{-1/2}\,\Ortho((B^\top B)^{-1/2}\nabla_A f)$.
\end{proof}

\subsection{Proof of \Cref{lemma:stiefel_spectral}}\label{app:proof_stiefel}

The proof requires showing that on each block of the product-norm decoupling, the tangent-space constraint is redundant.
Both blocks are SV-invariant (\Cref{app:sv_invariance}). For the skew block, the following result additionally shows that $\text{Skew}(r)$ is $\varphi$-contractive, so that any maximizer can be projected onto $\text{Skew}(r)$.

\begin{lemma}[Skew-symmetric LMO compatibility]\label{lem:skew_compat}
Let $S\in\textnormal{Skew}(r)$ and let $\varphi$ be any unitarily invariant norm.
Then
$\max_{A\in\textnormal{Skew}(r),\;\varphi(A)\le\tau}\inner{A}{S}
=\max_{\varphi(Z)\le\tau}\inner{Z}{S}.$
Thus, the $\textnormal{Skew}(r)$ constraint is redundant and the unconstrained maximizer can always be chosen skew-symmetric. In addition, $\textnormal{Skew}(r)$ is $\varphi$-contractive (\Cref{def:contractive}).
\end{lemma}
\begin{proof}
The inequality ``$\le$'' is trivial (smaller feasible set).
For ``$\ge$'': let $Z^*$ be any maximizer of the right-hand side.
Define $\hat{Z}:=\frac{1}{2}(Z^*-(Z^*)^\top)\in\text{Skew}(r)$.

Since $S\in\text{Skew}(r)$, we note that
$\inner{\hat{Z}}{S}
=\tfrac{1}{2}\inner{Z^*}{S}-\tfrac{1}{2}\inner{(Z^*)^\top}{S}
=\tfrac{1}{2}\inner{Z^*}{S}-\tfrac{1}{2}\inner{Z^*}{S^\top}
=\tfrac{1}{2}\inner{Z^*}{S}+\tfrac{1}{2}\inner{Z^*}{S}
=\inner{Z^*}{S}.$ Thus, the objective is preserved when $Z^*$ is replaced by $\hat{Z}$.

Since $\varphi$ is unitarily invariant, $\varphi(M^\top)=\varphi(M)$ and $\varphi(-M)=\varphi(M)$.
By the triangle inequality:
$\varphi(\hat{Z})=\varphi\!\bigl(\tfrac{Z^*-(Z^*)^\top}{2}\bigr)
\le\tfrac{1}{2}\varphi(Z^*)+\tfrac{1}{2}\varphi((Z^*)^\top)=\varphi(Z^*)\le\tau$. Hence, $\hat{Z}$ satisfies the norm constraint, since its $\varphi$-norm does not increase.

Overall, we showed that $\hat{Z}\in\text{Skew}(r)$, $\varphi(\hat{Z})\le\tau$, and $\inner{\hat{Z}}{S}=\inner{Z^*}{S}$. Therefore $\hat{Z}$ is also an optimizer of the left-hand side.

The norm bound above holds with $Z^*$ replaced by any $W\in\RR^{r\times r}$, which shows that $\text{Skew}(r)$ is $\varphi$-contractive.
\end{proof}

\begin{proof}[Proof of \Cref{lemma:stiefel_spectral}]
On $\text{St}(m,r)$ with $G_X=I$, a tangent vector $\xi\in T_X\text{St}$ decomposes as $\xi=XA+X_\perp B$ with $A\in\text{Skew}(r)$ and $B\in\RR^{(m-r)\times r}$.
Since $[X\mid X_\perp]\in O(m)$ and $\varphi$ is unitarily invariant, $\varphi(\xi)=\varphi\bigl(\begin{smallmatrix}A\\B\end{smallmatrix}\bigr)$.

\emph{Step 1: Product-norm decoupling.}
Under the $\ell_\infty$-product norm $\max(\varphi(A),\varphi(B))\le\tau$, the objective separates as
$g_X(\xi,\text{grad}\,f)=\inner{A}{S}+\inner{B}{N}$
where $S=\Skew(X^\top\nabla_X f)\in\text{Skew}(r)$ and $N=X_\perp^\top\nabla_X f\in\RR^{(m-r)\times r}$.
By~\eqref{eqn:whitened_lmo_product}, the LMO decouples into independent subproblems:
\[
A^*=\argmax_{A\in\text{Skew}(r),\;\varphi(A)\le\tau}\inner{A}{S},
\qquad
B^*=\argmax_{B\in\RR^{(m-r)\times r},\;\varphi(B)\le\tau}\inner{B}{N}.
\]

\emph{Step 2: Normal block ($B$).}
The feasible set $\RR^{(m-r)\times r}$ is the full ambient space, hence trivially SV-invariant.
By \Cref{prop:tangent_compat} and \Cref{lem:vonneumann}, $B^*=\tau\,\Ortho(N)=\tau\,\Ortho(X_\perp^\top\nabla_X f)$ for $\varphi=\|\cdot\|_2$.

\emph{Step 3: Skew block ($A$).}
The feasible set $\text{Skew}(r)$ is SV-invariant (\Cref{app:sv_invariance}). By \Cref{prop:tangent_compat} and \Cref{lem:vonneumann}, $A^*=\tau\,\Ortho(S)$ with the thin polar factor, which is skew-symmetric. Since $\text{Skew}(r)$ is also $\varphi$-contractive (\Cref{lem:skew_compat}), $\tau\,\Skew(\Ortho(S))$ is a solution for any choice of polar factor, and it coincides with $\tau\,\Ortho(S)$ for the thin polar factor.

Thus, $\xi^*=XA^*+X_\perp B^*=\tau\bigl(X\,\Skew(\Ortho(\Skew(X^\top\nabla_X f)))+X_\perp\,\Ortho(X_\perp^\top\nabla_X f)\bigr)$.
\end{proof}

\begin{remark}
With the thin polar factor, $\Ortho(S)$ is skew-symmetric for every $r$. If the polar factor is computed from a full SVD and $S$ is singular, which always happens for odd $r$ since $\det S=0$, it fills the null directions and need not be skew-symmetric. In this case, the skew projection $\hat{Z}=\frac{\tau}{2}(\Ortho(S)-\Ortho(S)^\top)$ achieves the same objective as $\tau\,\Ortho(S)$, since the null directions contribute zero to $\inner{\cdot}{S}$, and satisfies $\|\hat{Z}\|_2\le\tau\|\Ortho(S)\|_2=\tau$. Either way, $A^*\in\text{Skew}(r)$ with optimal value $\tau\,\varphi^\circ(S)$.
\end{remark}

\subsection{Proof of \Cref{thm:general}}\label{app:deterministic_proof}

\begin{lemma}[LMO objective value and update norm bound]\label{lem:lmo_dual}
Let $\xi_t^*$ be the solution of the intrinsic LMO~\eqref{eqn:main_lmo} at $x_t$ with norm bound $\tau$, and let $Z_t^*=G_{x_t}^{1/2}\xi_t^*$ and $H_t=G_{x_t}^{1/2}(\textnormal{grad}\,f(x_t))$. Then:
\begin{enumerate}
\item[(i)] \emph{(Objective value)} $g_{x_t}(\xi_t^*,\textnormal{grad}\,f_t) = \tau\,\varphi^\circ(H_t)$.
\item[(ii)] \emph{(Update norm bound)} $\|\xi_t^*\|_{x_t}^2\le C_\varphi\,\tau^2$.
\end{enumerate}
\end{lemma}
\begin{proof}
\emph{Part (i).}
By the change of variables $Z=G_{x_t}^{1/2}\xi$ (\Cref{prop:tangent_compat}), the LMO objective equals $\inner{Z_t^*}{H_t}$.
By \Cref{prop:tangent_compat}, $Z_t^*$ attains the maximum of $\inner{Z}{H_t}$ over all $Z$ with $\varphi(Z)\le\tau$, without the constraint $Z\in\mathcal{Z}_{x_t}$.
By \Cref{lem:vonneumann}, this maximum equals $\tau\,\varphi^\circ(H_t)$.

\emph{Part (ii).}
The Riemannian norm of the update is $\|\xi_t^*\|_{x_t}^2 = \|G_{x_t}^{1/2}\xi_t^*\|_F^2 = \|Z_t^*\|_F^2$.
By \Cref{def:effdim}, $C_\varphi(x_t)=\max\{\|Z\|_F^2:\varphi(Z)\le 1,\,Z\in\mathcal{Z}_{x_t}\}$, so for any $Z\in\mathcal{Z}_{x_t}$ with $\varphi(Z)\le\tau$, we have $\|Z\|_F^2=\tau^2\|Z/\tau\|_F^2\le\tau^2 C_\varphi(x_t)\le C_\varphi\,\tau^2$.
Applying this to $Z_t^*$ gives $\|Z_t^*\|_F^2\le C_\varphi\tau^2$.
\end{proof}

\begin{proof}[Proof of \Cref{thm:general}]
By retraction-$L$-smoothness (\Cref{assum:smooth}), with $x_{t+1}=R_{x_t}(-\eta\xi_t^*)$:
\begin{align}
f(x_{t+1}) &\le f(x_t) - \eta\,g_{x_t}(\xi_t^*,\text{grad}\,f_t) + \frac{L\eta^2}{2}\|\xi_t^*\|_{x_t}^2 \notag\\
&= f(x_t) - \eta\,D_t + \frac{L\eta^2}{2}\|Z_t^*\|_F^2 \notag\\
&\le f(x_t) - \eta\tau\,\varphi^\circ(H_t) + \frac{L\eta^2}{2}\,C_\varphi\,\tau^2. \label{eqn:descent_step}
\end{align}
Here we used $\inner{\xi_t^*}{\nabla f_t}=g_{x_t}(\xi_t^*,\text{grad}\,f_t)=\inner{Z_t^*}{H_t}=D_t$ and $\|Z_t^*\|_F^2\le C_\varphi\tau^2$ (\Cref{lem:lmo_dual}).

Summing~\eqref{eqn:descent_step} over $t=0,\ldots,T-1$:
\[
f(x_T)-f(x_0) \le -\eta\tau\sum_{t=0}^{T-1}\varphi^\circ(H_t) + \frac{L\eta^2 C_\varphi\tau^2 T}{2}.
\]
Rearranging and using $f(x_T)\ge f^*$:
\[
\frac{1}{T}\sum_{t=0}^{T-1}\varphi^\circ(H_t) \le \frac{\Delta_0}{\eta\tau T} + \frac{L\eta C_\varphi\tau}{2},
\]
where $\Delta_0=f(x_0)-f^*$.
Setting $\eta=\sqrt{2\Delta_0/(LC_\varphi\tau^2 T)}$:
\[
\min_{t\le T}\varphi^\circ(H_t) \le \frac{1}{T}\sum_t\varphi^\circ(H_t) \le \sqrt{\frac{2LC_\varphi\Delta_0}{T}}.
\]
\end{proof}

\noindent
\paragraph{Proof for the decaying schedule.}
With $\eta_t=\sqrt{2\Delta_0/(LC_\varphi)}/(\tau\sqrt{t+1})$, write $c=\sqrt{2\Delta_0/(LC_\varphi)}$ so that $\eta_t=c/(\tau\sqrt{t+1})$ and $Lc^2C_\varphi/2=\Delta_0$.
The descent step~\eqref{eqn:descent_step} with time-varying $\eta_t$ gives
\[
f(x_{t+1})\le f(x_t) - \frac{c\,\varphi^\circ(H_t)}{\sqrt{t+1}} + \frac{\Delta_0}{t+1}.
\]
Summing over $t=0,\ldots,T-1$ and using $f(x_T)\ge f^*$:
\[
c\sum_{t=0}^{T-1}\frac{\varphi^\circ(H_t)}{\sqrt{t+1}} \le \Delta_0 + \Delta_0\sum_{t=0}^{T-1}\frac{1}{t+1} \le \Delta_0(2+\ln T).
\]
Since $\min_t\varphi^\circ(H_t)\cdot\sum_t 1/\sqrt{t+1}\le\sum_t\varphi^\circ(H_t)/\sqrt{t+1}$ and $\sum_{t=0}^{T-1}1/\sqrt{t+1}\ge 2(\sqrt{T}-1)$:
\[
\min_{t\le T}\varphi^\circ(H_t) \le \frac{\Delta_0(2+\ln T)}{2c(\sqrt{T}-1)} = O\!\left(\sqrt{\frac{LC_\varphi\Delta_0}{T}}\cdot\ln T\right),
\]
matching the constant-step-size rate up to a logarithmic factor, without requiring the horizon $T$ in advance.

\subsection{Computation of $C_\varphi$ for the fixed-rank manifold}\label{app:proof_lowrank_rate}

We compute $C_\varphi=2r$ for $\M_r$ with $\varphi=\|\cdot\|_2$ (\Cref{cor:lowrank}).
The decoupled LMO produces $Z_{B}^*=\Ortho(H_{B})\in\RR^{m\times r}$ and $Z_{A}^*=\Ortho(H_{A})\in\RR^{r\times n}$, where $H_{A}:=(B^\top B)^{-1/2}\nabla_{A}f$.
Each has rank $\le r$, so $\|Z_{B}^*\|_F^2\le r$ and $\|Z_{A}^*\|_F^2\le r$.
The total Riemannian norm is $\|\xi^*\|_x^2=\|Z_{B}^*\|_F^2+\|Z_{A}^*\|_F^2\le 2r$.
Hence $C_\varphi=2r$, and \Cref{thm:general} gives rate $O(\sqrt{Lr/T})$.

\subsection{Computation of $C_\varphi$ for SPD, Stiefel, and Grassmann}\label{app:proof_others_rate}

\paragraph{SPD ($C_\varphi=n$).}
$Z^*=\sign(H)\in\text{Sym}(n)$ has eigenvalues in $\{-1,0,1\}$, so $\|Z^*\|_F^2\le n$.

\paragraph{Stiefel ($C_\varphi=2\lfloor r/2\rfloor+\min(m{-}r,r)$).}
With the product-norm decoupling,
the skew block gives $Z_S^*=\Skew(\Ortho(S))\in\text{Skew}(r)$ with $\|Z_S^*\|_2\le 1$. Since the singular values of a skew-symmetric matrix come in equal pairs, $\rank(Z_S^*)\le 2\lfloor r/2\rfloor$, so $\|Z_S^*\|_F^2\le 2\lfloor r/2\rfloor$.
The normal block gives $Z_N^*=\Ortho(N)\in\RR^{(m-r)\times r}$ with rank $\le\min(m{-}r,r)$, so $\|Z_N^*\|_F^2\le\min(m{-}r,r)$.
The total is $C_\varphi=2\lfloor r/2\rfloor+\min(m{-}r,r)$.

\paragraph{Grassmann ($C_\varphi=\min(m{-}r,r)$).}
$Z^*=\Ortho(\hat{H})\in\RR^{m\times r}$ with $\hat{H}=(I-XX^\top)\nabla_X f$ having column space in $\text{col}(X)^\perp\cong\RR^{m-r}$.
Thus $\Ortho(\hat{H})$ has rank $\le\min(m{-}r,r)$, giving $\|Z^*\|_F^2\le\min(m{-}r,r)$.

\subsection{Proof of \Cref{thm:stochastic}}\label{app:stochastic_proof}

\begin{proof}[Proof of \Cref{thm:stochastic}]
The LMO applied to the stochastic gradient $\widetilde{\nabla f}_t$ produces $\tilde{\xi}_t^*$ with scaled variable $\tilde{Z}_t^*=G_{x_t}^{1/2}\tilde{\xi}_t^*$ satisfying $\varphi(\tilde{Z}_t^*)\le\tau$.
The per-step descent (\Cref{assum:smooth}) gives
\[
f(x_{t+1}) \le f(x_t) - \eta_t\,g_{x_t}(\tilde{\xi}_t^*,\text{grad}\,f_t) + \frac{L\eta_t^2}{2}\,C_\varphi\,\tau^2.
\]
The inner product decomposes as
\[
g_{x_t}(\tilde{\xi}_t^*,\text{grad}\,f_t)
= \inner{\tilde{Z}_t^*}{H_t}
= \inner{\tilde{Z}_t^*}{\tilde{H}_t} + \inner{\tilde{Z}_t^*}{H_t - \tilde{H}_t}.
\]
The first term equals $\tau\,\varphi^\circ(\tilde{H}_t)$ by LMO optimality (\Cref{lem:lmo_dual}).
For the second,
H\"{o}lder's inequality gives $|\inner{\tilde{Z}_t^*}{H_t-\tilde{H}_t}|\le\varphi(\tilde{Z}_t^*)\,\varphi^\circ(H_t-\tilde{H}_t)\le\tau\,\varphi^\circ(\tilde{H}_t-H_t)$.
By the reverse triangle inequality, $\varphi^\circ(\tilde{H}_t)\ge\varphi^\circ(H_t)-\varphi^\circ(\tilde{H}_t-H_t)$.
Combining:
\begin{equation}\label{eqn:stoch_perstep}
g_{x_t}(\tilde{\xi}_t^*,\text{grad}\,f_t) \ge \tau\bigl[\varphi^\circ(H_t) - 2\,\varphi^\circ(\tilde{H}_t - H_t)\bigr].
\end{equation}
Substituting into the descent with $\eta_t=c/(\tau\sqrt{t+1})$:
\[
f(x_{t+1}) \le f(x_t) - \frac{c\,\varphi^\circ(H_t)}{\sqrt{t+1}} + \frac{2c\,\varphi^\circ(\tilde{H}_t-H_t)}{\sqrt{t+1}} + \frac{Lc^2C_\varphi}{2(t+1)}.
\]
Taking expectations conditional on $x_0,\ldots,x_t$ and using $\mathbb{E}[\varphi^\circ(\tilde{H}_t-H_t)\,|\,x_0,\ldots,x_t]\le\sigma_\varphi$ (by Jensen and \Cref{assum:noise}), then summing over $t=0,\ldots,T-1$:
\[
c\sum_{t=0}^{T-1}\frac{\mathbb{E}[\varphi^\circ(H_t)]}{\sqrt{t+1}}
\;\le\; \Delta_0 + 2c\sigma_\varphi\sum_{t=0}^{T-1}\frac{1}{\sqrt{t+1}} + \frac{Lc^2C_\varphi}{2}\sum_{t=0}^{T-1}\frac{1}{t+1}.
\]
Using $\sum 1/\sqrt{t+1}\le 2\sqrt{T}$ and $\sum 1/(t+1)\le 1+\ln T$, dividing both sides by $c\sum_{t=0}^{T-1}1/\sqrt{t+1}$, and noting that the $\sigma_\varphi$ term reduces exactly to $2\sigma_\varphi$:
\[
\min_{t\le T}\mathbb{E}[\varphi^\circ(H_t)]
\;\le\; \frac{\Delta_0}{c\sum_{t=0}^{T-1}1/\sqrt{t+1}} + 2\sigma_\varphi + \frac{Lc\,C_\varphi(1+\ln T)}{2\sum_{t=0}^{T-1}1/\sqrt{t+1}}.
\]
Using $\sum_{t=0}^{T-1}1/\sqrt{t+1}\ge\sqrt{T}$ in the denominator (each term is at least $1/\sqrt{T}$), which holds for all $T\ge 1$:
\begin{equation}\label{eqn:stoch_final}
\min_{t\le T}\mathbb{E}[\varphi^\circ(H_t)]
\;\le\; \frac{1}{\sqrt{T}}\Bigl(\frac{\Delta_0}{c} + \frac{Lc\,C_\varphi(1+\ln T)}{2}\Bigr) + 2\sigma_\varphi.
\end{equation}
Setting $c=\sqrt{2\Delta_0/(LC_\varphi)}$ yields $O\!\bigl(\sqrt{LC_\varphi\Delta_0}\ln T/\sqrt{T}\bigr)+2\sigma_\varphi$, recovering \Cref{thm:general} (up to log factors) when $\sigma_\varphi=0$.
\end{proof}

\section{Singular Value Invariance of the Scaled Tangent Spaces}\label{app:sv_invariance}

The closed-form solutions derived in \Cref{sec:solutions} all rely on \Cref{prop:tangent_compat},
through its condition that the scaled tangent space $\mathcal{Z}_x=\{G_x^{1/2}\xi:\xi\in T_x\M\}$ be singular value (SV) invariant (\Cref{def:sv_invariant}).
Recall that SV-invariance means: if $U\Sigma V^\top\in\mathcal{Z}_x$ is a compact SVD, then $UDV^\top\in\mathcal{Z}_x$ for every nonnegative diagonal $D$ that takes equal values on equal singular values. For such $D$, the matrix $UDV^\top$ does not depend on the choice of the compact SVD, so it suffices to check one compact SVD of each element.
When this holds, the tangent-space membership constraint $Z\in\mathcal{Z}_x$ in the LMO becomes redundant: the unconstrained maximizer $Z^*$ of $\inner{Z}{H}$ subject to $\varphi(Z)\le\tau$ automatically lands in $\mathcal{Z}_x$ (because $Z^*$ shares singular vectors with $H\in\mathcal{Z}_x$, and SV-invariance allows replacing singular values freely).
This is what enables the clean closed forms $\Ortho(H)$, $H/\|H\|_F$, and $u_1v_1^\top$ without any projection step.

Below we verify SV-invariance for each manifold treated in \Cref{sec:solutions}.
For the fixed-rank and Grassmann manifolds, $\mathcal{Z}_x$ is a standard matrix space and the verification is immediate.
For the SPD manifold, SV-invariance follows from the structure of symmetric matrices.
For the Stiefel manifold, the full tangent space is \emph{not} SV-invariant, which is why \Cref{sec:solutions}
uses the product-norm decoupling into skew and normal blocks. We verify SV-invariance of each block separately.

\noindent
\paragraph{Fixed-rank manifold $\M_r$.}
With the coupled metric, $G_x^{1/2}(\dot{B},\dot{A})=(\dot{B}(AA^\top)^{1/2},(B^\top B)^{1/2}\dot{A})$ and the tangent space is $T_x\M_r=\RR^{m\times r}\times\RR^{r\times n}$.
The scaled tangent space decomposes block-wise as $\mathcal{Z}_x=\RR^{m\times r}\times\RR^{r\times n}$, where each block is the full ambient matrix space of that dimension.
Full matrix spaces are trivially SV-invariant (every $m\times r$ matrix lies in $\RR^{m\times r}$, regardless of its singular values).
This is why the LMO on $\M_r$ (\Cref{cor:fixed_rank}) reduces to independent unconstrained problems on each factor block.

\noindent
\paragraph{SPD manifold $S_{++}^n$.}
With the affine-invariant metric, $G_X^{1/2}\xi=X^{-1/2}\xi X^{-1/2}$ maps $T_XS_{++}^n=\text{Sym}(n)$ to $\mathcal{Z}_X=\text{Sym}(n)$.
To check SV-invariance: a symmetric matrix $M$ has a compact SVD $M=U\Sigma V^\top$ where $V=U\diag(\pm 1)$ (the sign pattern encodes the signs of the eigenvalues).
Replacing $\Sigma$ with any admissible non-negative diagonal $D$ gives $UDV^\top=UD\diag(\pm 1)U^\top$, which is symmetric since $D\diag(\pm 1)$ is diagonal.
Hence $\text{Sym}(n)$ is SV-invariant, and the LMO on $S_{++}^n$ reduces to the unconstrained problem yielding $\sign(H)$, $H/\|H\|_F$, or rank-$1$ truncation.

\noindent
\paragraph{Stiefel manifold $\text{St}(m,r)$ (block-wise).}
With $G_X=I$, the scaled tangent space coincides with the tangent space itself: $\mathcal{Z}_X=T_X\text{St}=\{XA+X_\perp B: A\in\text{Skew}(r),\;B\in\RR^{(m-r)\times r}\}$.
Since $[X\mid X_\perp]\in O(m)$, any $\xi\in T_X\text{St}$ can be written as $[X,X_\perp]\bigl(\begin{smallmatrix}A\\B\end{smallmatrix}\bigr)$, so $\mathcal{Z}_X$ is isometric to $\{\bigl(\begin{smallmatrix}A\\B\end{smallmatrix}\bigr): A\in\text{Skew}(r),\;B\in\RR^{(m-r)\times r}\}$.
This space is \emph{not} jointly SV-invariant: if $\bigl(\begin{smallmatrix}A\\B\end{smallmatrix}\bigr)=U\Sigma V^\top$, replacing $\Sigma$ with a different $D$ can produce a top block that is no longer skew-symmetric.
This is precisely why \Cref{sec:solutions} uses the product-norm $\max(\varphi(A),\varphi(B))\le\tau$ to decouple the two blocks, and SV-invariance is checked for each block independently:
\begin{itemize}
\item \emph{Normal block} $\RR^{(m-r)\times r}$: the full ambient space, trivially SV-invariant.
\item \emph{Skew block} $\text{Skew}(r)$: a skew-symmetric matrix has real Schur form $S=Q\,\text{blkdiag}(\sigma_1 J,\sigma_2 J,\ldots)\,Q^\top$ with $J=\bigl(\begin{smallmatrix}0&1\\-1&0\end{smallmatrix}\bigr)$, so its singular values come in equal pairs $(\sigma_1,\sigma_1,\sigma_2,\sigma_2,\ldots)$.
Replacing $\sigma_i$ with any $d_i\ge 0$ gives $Q\,\text{blkdiag}(d_1 J,d_2 J,\ldots)\,Q^\top$, which remains skew-symmetric.
Since an admissible $D$ takes equal values on each pair of equal singular values, $\text{Skew}(r)$ is SV-invariant.
\end{itemize}
For $\varphi=\|\cdot\|_2$, the skew-block LMO sets all nonzero pairs to $(1,1)$, producing the thin polar factor $\Ortho(S)$. For $\varphi=\|\cdot\|_*$, it concentrates on the leading pair, giving a rank-$2$ skew-symmetric matrix.

\noindent
\paragraph{Grassmann manifold $\text{Gr}(m,r)$.}
With $G_X=I$ and $X^\top X=I_r$, the horizontal space is $\mathcal{Z}_X=T_X\text{Gr}=\{Z\in\RR^{m\times r}:X^\top Z=0\}$.
If $M=U\Sigma V^\top\in\mathcal{Z}_X$, then $X^\top M=0$ implies $\text{col}(M)\subseteq\text{col}(X)^\perp$.
In a compact SVD, the columns of $U$ span $\text{col}(M)\subseteq\text{col}(X)^\perp$.
Replacing $\Sigma$ with any non-negative diagonal $D$ gives $UDV^\top$, whose column space is still in $\text{col}(X)^\perp$, so $X^\top(UDV^\top)=0$.
Hence $T_X\text{Gr}$ is SV-invariant. No product-norm relaxation is needed on Grassmann, and the LMO reduces directly to the unconstrained problem.

\section{Derivations and Implementation Details}\label{app:implementation}

This appendix provides complete derivations of the intrinsic LMO solutions for all three standard norms (spectral, Frobenius, nuclear) on each manifold, along with step-by-step implementation procedures for the spectral instance.
The general pipeline from \Cref{prop:tangent_compat} is: (i)~scale $H=G_x^{1/2}(\text{grad}\,f)$, (ii)~solve $Z^*=\argmax_{\varphi(Z)\le\tau}\inner{Z}{H}$ in closed form via \Cref{app:primal_dual}, (iii)~invert $\xi^*=G_x^{-1/2}Z^*$.
The three standard norms give $Z^*=\tau\,\Ortho(H)$ (spectral), $Z^*=\tau\,H/\|H\|_F$ (Frobenius), and $Z^*=\tau\,u_1v_1^\top$ (nuclear).
Throughout, $\nabla_X f$ denotes the Euclidean gradient of $f$ at $X\in\M$.

\subsection{Summary of closed-form solutions}\label{app:summary_table}

\Cref{tab:solutions_app} collects the closed-form intrinsic LMO solutions on four manifolds under three norms.

\Cref{tab:per_manifold} complements \Cref{tab:solutions_app} by stating, for each manifold, whether the update solves the joint  LMO~\eqref{eqn:main_lmo} or the $\ell_\infty$-product LMO~\eqref{eqn:whitened_lmo_product}, and the resulting relaxation factor where a joint norm exists.

\begin{table}[h]
\caption{Closed-form intrinsic LMO solutions $\xi^*$ (up to the norm bound~$\tau$) on four manifolds.
$H_B=\nabla_B f\,(AA^\top)^{-1/2}$ and $H_A=(B^\top B)^{-1/2}\nabla_A f$ are the scaled factor gradients;
$H=X^{1/2}(\nabla_X f)X^{1/2}$ is the scaled SPD gradient;
$S=\Skew(X^\top\nabla_X f)$ and $N=X_\perp^\top\nabla_X f$ are the Stiefel skew and normal blocks;
$\hat{H}=(I-XX^\top)\nabla_X f$ is the Grassmann horizontal gradient.
}
\label{tab:solutions_app}
\centering
\small
\resizebox{\textwidth}{!}{\begin{tabular}{llccc}
\toprule
\textbf{Manifold} &  $G_x^{1/2}(\xi)$ & $\varphi=\|\cdot\|_2$ (spectral) & $\varphi=\|\cdot\|_F$ (Frobenius) & $\varphi=\|\cdot\|_*$ (nuclear) \\
\midrule
$\M_r$ ($\dot{B}$-block) &
$\dot{B}(AA^\top)^{1/2}$ &
$\Ortho(H_{B})\,(AA^\top)^{-1/2}$ &
$\nabla_{B}f\,(AA^\top)^{-1}/\|H_B\|_F$ &
$u_1 v_1^\top(AA^\top)^{-1/2}$ \\[4pt]
$S_{++}^n$ &
$X^{-1/2}\xi\,X^{-1/2}$ &
$X^{1/2}\sign(H)X^{1/2}$ &
$X(\nabla_X f)X/\|\text{grad}\,f\|_X$ &
$\sign(\lambda_1)\,X^{1/2}q_1 q_1^\top X^{1/2}$ \\[4pt]
$\text{St}(m,r)$ &
$\xi$  &
$X\,\Skew(\Ortho(S)) + X_\perp\,\Ortho(N)$ &
$X(S/\|S\|_F) + X_\perp(N/\|N\|_F)$ &
rank-2 skew $+$ rank-1 \\[4pt]
$\text{Gr}(m,r)$ &
$\xi$ &
$\Ortho(\hat{H})$ &
$\hat{H}/\|\hat{H}\|_F$ &
$u_1 v_1^\top$ of $\hat{H}$ \\
\bottomrule
\end{tabular}}
\end{table}

\begin{table}[h]
\centering\small
\caption{Intrinsic LMO formulation solved on each manifold. For every row the solution
is in closed form for the spectral, Frobenius, nuclear, partial polar, and Schatten norms
(\Cref{app:primal_dual}). Any other unitarily invariant norm reduces to a
$\min(p,q)$-dimensional convex program over singular values (\Cref{prop:tangent_compat}).
The last column gives the relaxation factor of the product constraint relative to the joint
norm when a joint norm exists.}
\label{tab:per_manifold}
\begin{tabular}{@{}l p{5.2cm} p{6.2cm}@{}}
\toprule
Manifold & LMO solved & Relaxation of the joint norm \\
\midrule
$\M_r$ & $\ell_\infty$-product LMO~\eqref{eqn:whitened_lmo_product} over the factor blocks $(\dot B,\dot A)$ & not applicable, since the two blocks have different shapes and no joint unitarily invariant norm is defined on the pair \\
$S_{++}^n$ & joint LMO~\eqref{eqn:main_lmo} & none, the joint LMO is solved exactly \\
$\text{St}(m,r)$ & $\ell_\infty$-product LMO~\eqref{eqn:whitened_lmo_product} over the skew and normal blocks $(A,B)$ & factor $\sqrt{2}$ for the spectral norm, since $\max\{\|A\|_2,\|B\|_2\}\le\|[A;B]\|_2\le\sqrt{2}\max\{\|A\|_2,\|B\|_2\}$ \\
$\text{Gr}(m,r)$ & joint LMO~\eqref{eqn:main_lmo} & none, the joint LMO is solved exactly \\
\bottomrule
\end{tabular}
\end{table}

\subsection{Computing the polar factor}\label{app:ortho}

\noindent
\paragraph{Newton--Schulz iterations.}
All spectral-norm instances require computing the polar factor $\Ortho(M)=UV^\top$ from $M=U\Sigma V^\top$.
Following~\citet{jordan2024muon}, this can be approximated by Newton--Schulz iterations:
$X_0=M/\|M\|_F$, \; $X_{k+1}=\frac{3}{2}X_k-\frac{1}{2}X_k X_k^\top X_k$.
Each iteration costs $O(pr^2)$ for $M\in\RR^{p\times r}$ and converges quadratically when $\|X_0\|_2<\sqrt{3}$, which is ensured by the normalization since $\|X_0\|_2=\|M\|_2/\|M\|_F\le 1$.
In practice, 5--10 iterations suffice.

The choice of algorithm depends on the shape and size of $M$:

\noindent
\emph{Tall-skinny or short-wide matrices} ($M\in\RR^{p\times r}$ with $r\ll p$): dominant case on $\M_r$, $\text{Gr}(m,r)$, and the normal block of $\text{St}(m,r)$.
Thin SVD costs $O(pr^2)$; Newton--Schulz avoids the SVD and is GPU-friendly.

\noindent
\emph{Small square matrices} ($M\in\RR^{r\times r}$): arises for the skew block on $\text{St}(m,r)$.
For skew-symmetric $S$, the real Schur decomposition gives $\Ortho(S)$ directly.

\noindent
\emph{Large square matrices} ($M\in\RR^{n\times n}$): arises on $S_{++}^n$ where the ``polar factor'' is the matrix sign $\sign(H)$.
Newton iterations $X_{k+1}=\frac{1}{2}(X_k+X_k^{-1})$ converge quadratically; eigendecomposition is also standard.

\subsection{Fixed-rank manifold $\M_r$}\label{app:impl_lowrank}

\paragraph{Derivation (all three norms).}
On $\M_r$ with the coupled metric, $G_x^{1/2}(\dot{B},\dot{A})=(\dot{B}(AA^\top)^{1/2},(B^\top B)^{1/2}\dot{A})$.
The LMO decouples into independent $\dot{B}$ and $\dot{A}$ subproblems via~\eqref{eqn:whitened_lmo_product}.
For the $\dot{B}$-block, the scaled gradient is $H_B=\nabla_B f\,(AA^\top)^{-1/2}$ where $\nabla_B f=\nabla_X f A^\top$.
By \Cref{prop:tangent_compat}, $Z_B^*$ is the unconstrained maximizer of $\inner{Z_B}{H_B}$ subject to $\varphi(Z_B)\le\tau$, and $\dot{B}^*=Z_B^*(AA^\top)^{-1/2}$.
The three norms give:
\begin{itemize}
\item \emph{Spectral} ($\varphi=\|\cdot\|_2$): $Z_B^*=\tau\,\Ortho(H_B)$, so $\dot{B}^*=\tau\,\Ortho(\nabla_B f\,(AA^\top)^{-1/2})\,(AA^\top)^{-1/2}$.
\item \emph{Frobenius} ($\varphi=\|\cdot\|_F$): $Z_B^*=\tau\,H_B/\|H_B\|_F$, so $\dot{B}^*=\tau\,\nabla_B f\,(AA^\top)^{-1}/(\|H_B\|_F)$.
This is the Riemannian preconditioned gradient of~\citet{mishra2016riemannian,zhang24riemannianlora}, normalized to unit intrinsic Frobenius norm.
\item \emph{Nuclear} ($\varphi=\|\cdot\|_*$): $Z_B^*=\tau\,u_1v_1^\top$ from the SVD of $H_B$, so $\dot{B}^*=\tau\,u_1v_1^\top(AA^\top)^{-1/2}$ (rank-$1$ update per block).
\end{itemize}
The $\dot{A}$-block follows by symmetry with $H_A=(B^\top B)^{-1/2}\nabla_A f$.

\paragraph{QR-based implementation (spectral norm).}
We show that $\dot{B}=\tau\,\Ortho(H_{B})\,(AA^\top)^{-1/2}$ can be computed using only QR factors.
Let $A^\top=\hat{Q}_{A}R_{A}$ be a thin QR factorization ($\hat{Q}_{A}\in\RR^{n\times r}$ with orthonormal columns, $R_{A}\in\RR^{r\times r}$ upper triangular with positive diagonal); similarly $B=\hat{Q}_{B}R_{B}$.
The scaled factor gradient is the $m\times r$ matrix
\[
H_{B} = \nabla_{B}f\,(AA^\top)^{-1/2} = \nabla_X fA^\top(AA^\top)^{-1/2} = \nabla_X f\hat{Q}_{A}R_{A}(AA^\top)^{-1/2}.
\]
Writing $Q_{A}:=(AA^\top)^{-1/2}A\in\RR^{r\times n}$ (orthonormal rows), we have $H_{B}=\nabla_X f Q_{A}^\top$.
Since $Q_{A}^\top$ and $\hat{Q}_{A}$ span the same column space, $Q_{A}^\top=\hat{Q}_{A}\Omega$ for some $\Omega\in O(r)$.
Thus $\Ortho(H_{B})=\Ortho(\nabla_X f\hat{Q}_{A})\,\Omega$. Moreover, $AA^\top=R_A^\top R_A$, so $R_A=\Omega\,(AA^\top)^{1/2}$ is the polar decomposition of $R_A$ and $\Omega(AA^\top)^{-1/2}=R_A^{-\top}$ exactly.

In practice we compute $\dot{B}=\tau\,\Ortho(\nabla_X f\hat{Q}_{A})\,R_{A}^{-\top}$ and $\dot{A}=\tau\,R_{B}^{-1}\,\Ortho(\hat{Q}_{B}^\top\nabla_X f)$.

\noindent\textbf{Step-by-step procedure} ($\dot{B}$ computation; $\dot{A}$ is symmetric):
\begin{enumerate}
\item \textbf{QR factorization}: $A^\top = \hat{Q}_{A}R_{A}$; similarly $B=\hat{Q}_{B}R_{B}$. \hfill $O((m+n)r^2)$.
\item \textbf{Projected gradient}: $\nabla_X f\,\hat{Q}_{A}=(\nabla_B f)\,R_A^{-1}\in\RR^{m\times r}$. \hfill $O(mr^2)$.
\item \textbf{Polar factor}: $\Ortho(\nabla_X f\,\hat{Q}_{A})$ via Newton--Schulz. \hfill $O(mr^2)$ per iteration.
\item \textbf{Triangular solve}: $\dot{B} = \tau\,\Ortho(\nabla_X f\,\hat{Q}_{A})\,R_{A}^{-\top}$. \hfill $O(mr^2)$.
\item \textbf{Retract}: $B_{+} = B - \eta\,\dot{B}$, $A_{+} = A - \eta\,\dot{A}$. \hfill $O((m+n)r)$.
\end{enumerate}
\noindent\textbf{Total}: $O((m+n)r^2+r^3)$ per step.
No eigendecomposition or matrix square root is needed.

\subsection{SPD manifold $S_{++}^n$}\label{app:impl_spd}

\paragraph{Derivation (all three norms).}
With the affine-invariant metric, $G_X^{1/2}\xi=X^{-1/2}\xi X^{-1/2}$ and $\text{grad}\,f(X)=X(\nabla_X f)X$.
The scaled gradient is $H=G_X^{1/2}(\text{grad}\,f)=X^{1/2}(\nabla_X f)X^{1/2}\in\text{Sym}(n)$.
Since $\mathcal{Z}_X=\text{Sym}(n)$ is SV-invariant, \Cref{prop:tangent_compat} applies.
The three norms give $Z^*$ and $\xi^*=X^{1/2}Z^*X^{1/2}$:
\begin{itemize}
\item \emph{Spectral}: $Z^*=\tau\,\sign(H)=\tau\,Q\diag(\sign(\lambda_i))Q^\top$, so $\xi^*=\tau\,X^{1/2}\sign(H)X^{1/2}$.
\item \emph{Frobenius}: $Z^*=\tau\,H/\|H\|_F$, so $\xi^*=\tau\,\text{grad}\,f/\|\text{grad}\,f\|_X$ (the normalized natural gradient).
\item \emph{Nuclear}: $Z^*=\tau\,\sign(\lambda_1)\,q_1q_1^\top$, so $\xi^*=\tau\,\sign(\lambda_1)\,X^{1/2}q_1q_1^\top X^{1/2}$ (rank-$1$ symmetric).
\end{itemize}
All three are $\text{GL}(n)$-equivariant by \Cref{prop:symmetry}: under $X\mapsto NXN^\top$, the update transforms as $\xi^*\mapsto N\xi^*N^\top$.

\paragraph{Implementation (spectral norm).}
\begin{enumerate}
\item \textbf{Matrix square root}: Compute $X^{1/2}$ and $X^{-1/2}$ via eigendecomposition of $X$. \hfill $O(n^3)$.
\item \textbf{Scaled gradient}: $H = X^{1/2}\,\nabla_X f\,X^{1/2}$. \hfill $O(n^3)$.
\item \textbf{Matrix sign}: $\sign(H) = Q\diag(\sign(\lambda_i))Q^\top$ via eigendecomposition, or Newton iterations $Z_{k+1}=\frac{1}{2}(Z_k+Z_k^{-1})$. \hfill $O(n^3)$.
\item \textbf{Invert}: $\xi^* = \tau\,X^{1/2}\,\sign(H)\,X^{1/2}$. \hfill $O(n^3)$.
\item \textbf{Retract}: $X_{+} = X^{1/2}\exp(-\eta\,X^{-1/2}\xi^* X^{-1/2})\,X^{1/2}$. \hfill $O(n^3)$.
\end{enumerate}
\noindent\textbf{Total}: $O(n^3)$ per step, dominated by eigendecompositions.

\subsection{Stiefel manifold $\text{St}(m,r)$}\label{app:impl_stiefel}

\paragraph{Derivation (all three norms).}
With $G_X=I$, the tangent space decomposes as $\xi=XA+X_\perp B$ with $A\in\text{Skew}(r)$, $B\in\RR^{(m-r)\times r}$.
Under the product norm $\max(\varphi(A),\varphi(B))\le\tau$, the LMO decouples into skew and normal blocks with $S=\Skew(X^\top\nabla_X f)$ and $N=X_\perp^\top\nabla_X f$.
\begin{itemize}
\item \emph{Spectral}: $A^*=\tau\,\Skew(\Ortho(S))$, $B^*=\tau\,\Ortho(N)$.
Result: $\xi^*=\tau(X\,\Skew(\Ortho(S))+X_\perp\,\Ortho(N))$ (\Cref{lemma:stiefel_spectral}).
\item \emph{Frobenius}: $A^*=\tau\,S/\|S\|_F$, $B^*=\tau\,N/\|N\|_F$ (block-wise normalized).
\item \emph{Nuclear}: $A^*=\frac{\tau}{2}(u_1v_1^\top+u_2v_2^\top)$ from the leading pair of singular vectors of $S$, a rank-$2$ skew-symmetric matrix with $\|A^*\|_*=\tau$, and $B^*=\tau\,p_1q_1^\top$ from the leading singular vectors of $N$ (rank-$1$).
\end{itemize}

\paragraph{Implementation (spectral norm).}
$X_\perp\in\RR^{m\times(m-r)}$ is never formed explicitly; instead $X_\perp\,\Ortho(X_\perp^\top\nabla_X f)=\Ortho((I-XX^\top)\nabla_X f)$, where $\Ortho(M)=U_1V_1^\top$ is formed from the nonzero singular values of $M$.
The Newton--Schulz iteration of \Cref{app:ortho} returns this factor directly.
If a thin SVD is used instead, the singular vectors associated with zero singular values are discarded, which is equivalent to forming $X_\perp$ explicitly and using $X_\perp\,\Ortho(X_\perp^\top\nabla_X f)$.
\begin{enumerate}
\item \textbf{Skew gradient}: $P = X^\top\nabla_X f$, then $S = \frac{1}{2}(P - P^\top)$. \hfill $O(mr^2)$.
\item \textbf{Skew polar factor}: $\Skew(\Ortho(S))$ via real Schur decomposition. \hfill $O(r^3)$.
\item \textbf{Normal gradient}: $\hat{N} = \nabla_X f - X\,P$. \hfill $O(mr^2)$.
\item \textbf{Normal polar factor}: $\Ortho(\hat{N})$ via Newton--Schulz. \hfill $O(mr^2)$.
\item \textbf{Assemble}: $\xi^* = \tau[X\,\Skew(\Ortho(S)) + \Ortho(\hat{N})]$. \hfill $O(mr^2)$.
\item \textbf{Retract}: $X_{+} = \qf(X - \eta\,\xi^*)$. \hfill $O(mr^2)$.
\end{enumerate}
\noindent\textbf{Total}: $O(mr^2+r^3)$ per step. Steps 2 and 4 can be parallelized.

\subsection{Grassmann manifold $\text{Gr}(m,r)$}\label{app:impl_grassmann}

\paragraph{Derivation (all three norms).}
With $G_X=I$ and $X^\top X=I_r$, the horizontal space $T_X\text{Gr}=\{\xi:X^\top\xi=0\}$ is SV-invariant.
The horizontal gradient is $\hat{H}=(I-XX^\top)\nabla_X f$.
\begin{itemize}
\item \emph{Spectral}: $\xi^*=\tau\,\Ortho(\hat{H})$,  orthonormal columns in $\text{col}(X)^\perp$.
\item \emph{Frobenius}: $\xi^*=\tau\,\hat{H}/\|\hat{H}\|_F$,  normalized Riemannian gradient on $\text{Gr}(m,r)$~\citep{edelman1998geometry}.
\item \emph{Nuclear}: $\xi^*=\tau\,u_1v_1^\top$ from the SVD of $\hat{H}$,  rank-$1$ horizontal update.
\end{itemize}

\paragraph{Implementation (spectral norm).}
\begin{enumerate}
\item \textbf{Horizontal projection}: $\hat{H} = \nabla_X f - X\,(X^\top\nabla_X f)$. \hfill $O(mr^2)$.
\item \textbf{Polar factor}: $\Ortho(\hat{H})$ via Newton--Schulz. \hfill $O(mr^2)$.
\item \textbf{Assemble}: $\xi^* = \tau\,\Ortho(\hat{H})$. \hfill $O(mr)$.
\item \textbf{Retract}: $X_{+} = \qf(X - \eta\,\xi^*)$. \hfill $O(mr^2)$.
\end{enumerate}
\noindent\textbf{Total}: $O(mr^2)$ per step. The simplest instantiation: $G_X=I$ eliminates scaling and inversion.

\section{Prior Convergence Results as Special Cases}\label{app:connections}

We show how several known rates arise as special cases of Theorems~\ref{thm:general} and~\ref{thm:stochastic}.

\paragraph{Euclidean spectral steepest descent~\citep{carlson2015psd,bernstein2025old}.}
Setting $\M=\RR^{p\times q}$ (unconstrained), $G_x=I$, and $\varphi=\|\cdot\|_2$ recovers the Euclidean Muon update $\xi^*=\tau\,\Ortho(\nabla f)$.
The use of this spectral-norm steepest-descent operator for deep learning was first proposed by \citet{carlson2015psd}, who also combined it with element-wise adaptive preconditioning; \citet[Prop.~5]{bernstein2025old} later formalize the connection to steepest descent.
The squared Riemannian radius is $C_\varphi=\min(p,q)$ (rank of $\Ortho(\nabla f)$), and \Cref{thm:general} gives the nonasymptotic rate $O(\sqrt{L\min(p,q)/T})$.

\paragraph{Standard Riemannian gradient descent~\citep{boumal2023intromanifolds}.}
Setting $\varphi=\|\cdot\|_F$ with the joint LMO~\eqref{eqn:main_lmo} gives $\xi^*\propto\text{grad}\,f/\|\text{grad}\,f\|_x$ and $C_\varphi=1$ (under the $\ell_\infty$-product norm on $\M_r$ and $\text{St}(m,r)$, the direction is normalized block-wise and $C_\varphi=2$).
\Cref{thm:general} yields $\min_t\|\text{grad}\,f(x_t)\|_{x_t}\le\sqrt{2L\Delta_0/T}$, which is the standard retraction-smooth Riemannian GD rate~\citep[Cor.~4.9]{boumal2023intromanifolds} (with the normalized-gradient step replacing the unnormalized one; the rates coincide up to constants when $\eta$ is optimized).

\paragraph{Factor-wise Muon on $\M_r$~\citep{kang2026spectral}.}
\citet{kang2026spectral} analyze factor-wise Muon (SpecGF) on $\M_r$ with $\varphi=\|\cdot\|_2$.
Since their orthogonalization operator satisfies $\|T(\cdot)\|_2=1$, the ambient update norm satisfies $\|\dot{X}_t\|_F^2\le 2\tau^2(\|A_t\|_F^2+\|B_t\|_F^2)$, which depends on factor conditioning.
Accordingly, their main convergence result (Theorem~6.2) is conditional: SpecGF either converges to a global minimum or $\|A_t\|_F^2+\|B_t\|_F^2\to\infty$; unconditional global convergence is established only with $\ell_2$ regularization.
Our quotient formulation sidesteps this issue entirely: the Gram-root scaling yields a rank-only bound $\|\xi^*\|_x^2\le 2r\tau^2$ with $C_\varphi=2r$ (\Cref{cor:lowrank}, computed in \Cref{app:proof_lowrank_rate}), giving a nonasymptotic rate $O(\sqrt{Lr/T})$ without any bounded-factor-norm condition.

\paragraph{NuMuon~\citep{dolatabadi2026numuon}.}
Setting $G_x=I$ and $\varphi=\|\cdot\|_*$ gives $\xi^*=\tau\,u_1v_1^\top$ (rank-$1$ update) with $C_\varphi=1$.
\Cref{thm:general} gives rate $O(\sqrt{L/T})$ with convergence in $\varphi^\circ=\|\cdot\|_2$ (spectral norm of the gradient).
NuMuon~\citep{dolatabadi2026numuon} uses a joint spectral-and-nuclear-norm constraint ($\|O\|_2\le\rho$, $\|O\|_*\le\tau$), yielding a rank-$k$ update ($k=\lfloor\tau/\rho\rfloor$); their Theorem~3.6 gives a rate with an $Lk$-dependent term and convergence in $\|\nabla f\|_*$ (nuclear norm of the gradient).
Our pure nuclear instance ($k=1$) recovers the rank-$1$ special case.
On manifolds with $G_x\neq I$, our quotient extensions add symmetry invariance while preserving $C_\varphi=O(1)$.

\paragraph{General Euclidean LMOs~\citep{pethick2025lmo}.}
\citet{pethick2025lmo} study LMOs constrained by general norms (including spectral, RMS, and sign norms) in Euclidean space ($G_x=I$), and propose Scion, a unified algorithmic family.
Their spectral-norm instance corresponds to our framework with $G_x=I$ on $\RR^{p\times q}$.
Our contribution beyond their setting is the non-trivial metric $G_x\neq I$ (which introduces the metric-scaling and inversion steps and the symmetry proposition) and the per-manifold closed forms.

\paragraph{The role of $C_\varphi$.}
Across all these instances, the squared Riemannian radius $C_\varphi$ is the unifying quantity:
it is $\min(p,q)$ for Euclidean Muon, $O(1)$ for Frobenius/nuclear norms, $2r$ for quotient-Muon on $\M_r$,
and also $2r$ for factor-wise Muon (Euclidean metric on the factors), whose dependence on $\|A_t\|_F^2+\|B_t\|_F^2$ enters instead through the ambient update norm and hence the smoothness constant.
The spectral norm always gives the largest $C_\varphi$ but the strongest convergence criterion ($\varphi^\circ=\|\cdot\|_*$). The nuclear norm gives the smallest $C_\varphi$ but the weakest criterion ($\varphi^\circ=\|\cdot\|_2$).

\section{Detailed Comparisons with Prior Methods}\label{app:comparisons}

In the main text (\Cref{sec:solutions}), we briefly compare our framework to prior methods on the fixed-rank manifold.
Here we provide detailed calculations and side-by-side comparisons on all four manifolds. \Cref{tab:convergence_comparison_cr} summarizes convergence guarantees for selected prior methods and iMuon. The entries use the convergence criteria and assumptions of the cited results.
\begin{table}[t]
\centering
\caption{Convergence guarantees and key assumptions of various prior works and iMuon.}
\label{tab:convergence_comparison_cr}
\small
\setlength{\tabcolsep}{4pt}
\renewcommand{\arraystretch}{1.12}
\begin{tabular}{@{}>{\raggedright\arraybackslash}p{0.30\linewidth}>{\raggedright\arraybackslash}p{0.29\linewidth}>{\raggedright\arraybackslash}p{0.36\linewidth}@{}}
\toprule
Method & Convergence guarantee & Key assumptions \\
\midrule
Riemannian AMSGrad / AdamNC
\citep[Thms.~1--2]{becigneul2018riemannian}
& Average regret:  $\widetilde O(T^{-1/2})$
& Geodesic convexity, compact domain, bounded gradients \\[4pt]
Scaled GD
\citep[Thm.~7.4]{zhang24riemannianlora}
& Local linear, condition-number independent
& Least squares, RIP, spectral initialization \\[4pt]
Spectron
\citep{janson2026stabilizing}
& $\times$
& N/A \\[4pt]
Riemannion
\citep{bogachev2026riemannion}
& $\times$
& N/A \\[4pt]
SpecGF
\citep[Thm.~6.2]{kang2026spectral}
& Global minimum or factor divergence
& Least-squares factorization, almost every initialization \\[4pt]
SpecGF
\citep[Prop.~6.7]{kang2026spectral}
& Conditional exponential loss decay
& Convergence to a global minimum,  distinct target singular values \\[4pt]
iMuon, deterministic,
\Cref{thm:general}
& Stationarity:  $O(T^{-1/2})$
& Retraction smoothness \\[4pt]
iMuon, stochastic,
\Cref{thm:stochastic}
& Expected stationarity:  $\widetilde O(T^{-1/2})+2\sigma_\varphi$
& Retraction smoothness; bounded scaled noise \\
\bottomrule
\end{tabular}
\par\smallskip
\begin{minipage}{0.98\linewidth}
\footnotesize
$\times$: no optimization convergence rate established in the cited work.
$\widetilde O$ suppresses logarithmic factors.
The Riemannian adaptive bounds depend on curvature and domain diameter.
SpecGF is a continuous-time flow (with a smoothed orthogonalization map) for least-squares factorization.
For iMuon, rates concern best-iterate stationarity in $\varphi^\circ(H_t)$, with constants scaling as $\sqrt{LC_\varphi\Delta_0}$; $C_\varphi=2r$ for the fixed-rank spectral product formulation.
Full assumptions and step-size conditions are given in the cited works.
\end{minipage}
\end{table}
\subsection{Fixed-rank manifold $\M_r$: iMuon vs.\ Spectron, Riemannion, and factor-wise Muon}\label{app:compare_lowrank}

We compare three approaches to spectral-norm optimization on the rank-$r$ manifold $\M_r$ with factorization $X=BA$.
Throughout, $\nabla_X f$ denotes the Euclidean gradient.

\noindent
\paragraph{Prior methods.}
\emph{Factor-wise Muon} applies the standard Euclidean Muon \citep{jordan2024muon,bernstein2025modular,bernstein2025old} polar factor $\Ortho(\cdot)$ to each factor's gradient independently: $\dot{B}=\tau\,\Ortho(\nabla_X fA^\top)$ and $\dot{A}=\tau\,\Ortho(B^\top\nabla_X f)$.
This is simple and requires no Riemannian structure, but the resulting update $\dot{X}=\dot{B}A+B\dot{A}$ has $\|\dot{X}\|_F^2\le 2\tau^2(\|A\|_F^2+\|B\|_F^2)$, which depends on factor conditioning. The convergence analysis of \citet{kang2026spectral} for this update is discussed in \Cref{app:connections} and \Cref{tab:convergence_comparison_cr}.

\emph{Spectron}~\citep{janson2026stabilizing} addresses the same factor-growth pathology from a practical angle.
Like factor-wise Muon, Spectron orthogonalizes each factor's momentum via Newton--Schulz iterations, producing $\Ortho(M_A)$ and $\Ortho(M_B)$.
The key innovation is a \emph{spectral renormalization} step that adaptively scales the updates to control the composite update norm.
Since the ambient update satisfies $\|\dot{X}\|_2\le\rho(\|A\|_2+\|B\|_2+1)$ (via submultiplicativity and the triangle inequality), Spectron sets the per-factor constraint radius to
$\rho=\eta/(\|A\|_2+\|B\|_2+1)$,
yielding factor updates $\dot{A}=\rho\,\Ortho(M_A)$ and $\dot{B}=\rho\,\Ortho(M_B)$, which guarantee $\|\dot{X}\|_2\le\eta$.
The spectral norms $\|A\|_2,\|B\|_2$ are estimated at each iteration via a single power iteration step.
This makes training stable and enables native low-rank pretraining that matches dense-model performance.

Spectron and our method address the same instability but through fundamentally different mechanisms.
Spectron is an \emph{extrinsic} fix: it applies standard (Euclidean) orthogonalization to each factor and then rescales the result by an iterate-dependent quantity $1/(\|A\|_2+\|B\|_2+1)$ to bound the ambient update.
The rescaling shrinks as the factors grow, which is effective in practice but means the effective step size is not a fixed hyperparameter. It depends on the current factor norms.
Our method is \emph{intrinsic}: the Gram-root preconditioning $(AA^\top)^{-1/2}$ from the coupled metric absorbs the factor scaling \emph{before} orthogonalization, so the resulting update is $\text{GL}(r)$-equivariant by construction (\Cref{prop:symmetry}), its induced update of $X=BA$ does not depend on the factorization, and its squared Riemannian norm is bounded by $2r\tau^2$ regardless of factor magnitudes (\Cref{app:proof_lowrank_rate}), with no power iteration or adaptive rescaling needed.

\emph{Riemannion}~\citep{bogachev2026riemannion} takes an \emph{embedded-manifold} approach.
It optimizes $X$ directly on $\M_r=\{X\in\RR^{m\times n}:\rank(X)=r\}$, parameterized via a compact SVD $X=A_L\Sigma_r B_R^\top$ with $A_L\in\text{St}(m,r)$, $B_R\in\text{St}(n,r)$, and $\Sigma_r\in\RR^{r\times r}$ diagonal positive. At a point $X\in\M_r$, the tangent space $T_X\M_r$ consists of rank-$\le 2r$ matrices; in SVD coordinates, $\dot{X}=A_L M B_R^\top + U_\perp B_R^\top + A_L V_\perp^\top$ for $M\in\RR^{r\times r}$, $U_\perp\perp A_L$, and $V_\perp\perp B_R$. Riemannion maintains a momentum $M_t\in T_X\M_r$ (rank $\le 2r$) via Riemannian heavy-ball updates with vector transport.
To generalize Muon, it applies a rank-$2r$ orthogonalization $\Ortho_r(M_t)$, replacing only the first $2r$ singular values of $M_t$ with $1$ while setting the rest to $0$, and then projects back onto the tangent space:
$\tilde{M}_t=P_{T_X\M_r}(\Ortho_r(M_t))$.
Because $\Ortho_r$ preserves the column and row spaces of $M_t\in T_X\M_r$, the result is a good approximation (the authors report singular values in $(0.9,1.1)$ in practice), but not exact.
The iterate is then updated via retraction: $X_{t+1}=R_X(-\eta\,\tilde{M}_t)$.

We can equivalently describe the entire procedure in the $X=BA$ notation used throughout this section.
The embedded tangent space is $T_X\M_r=\{\dot{B}A+B\dot{A}:\dot{B}\in\RR^{m\times r},\;\dot{A}\in\RR^{r\times n}\}$, a subspace of rank-$\le 2r$ matrices in $\RR^{m\times n}$.
The orthogonal projection of any $Z\in\RR^{m\times n}$ onto $T_X\M_r$ (under the Euclidean metric $G_X=I$) is
$P_{T_X\M_r}(Z)=P_B\,Z + Z\,P_A - P_B\,Z\,P_A$,
where $P_B=B(B^\top B)^{-1}B^\top$ projects onto $\text{col}(B)$ and $P_A=A^\top(AA^\top)^{-1}A$ projects onto $\text{row}(A)$.
Riemannion's update in $BA$ language is then: (i)~maintain momentum $M_t\in T_X\M_r$ (rank $\le 2r$); (ii)~orthogonalize $\Ortho_r(M_t)$ (set all $2r$ nonzero singular values to $1$); (iii)~project back via $\tilde{M}_t = P_B\,\Ortho_r(M_t) + \Ortho_r(M_t)\,P_A - P_B\,\Ortho_r(M_t)\,P_A$; (iv)~retract $X_{t+1}=R_X(-\eta\,\tilde{M}_t)$.
The projection in step~(iii) couples $B$ and $A$ through $P_B$ and $P_A$: the result depends on both factors simultaneously, and there is no block-diagonal decoupling into independent $\dot{B}$ and $\dot{A}$ subproblems.
Moreover, the projection perturbs the singular values away from $1$, so the spectral-norm constraint $\|\tilde{M}_t\|_2\le 1$ is only approximately satisfied.

Because Riemannion works on $\M_r$ directly (not on factors $(A,B)$), it is reparameterization-invariant: different factorizations of the same $X$ yield the same update.
However, it uses the Euclidean metric ($G_X=I$) on the embedded manifold, which provides no preconditioning.
To obtain an exact spectral-norm solution, Riemannion would need to solve the tangent-constrained LMO $\max_{S\in T_X\M_r,\,\|S\|_2\le 1}\inner{M_t}{S}$,
for which \Cref{prop:tangent_compat} does not apply, since $T_X\M_r$ in the embedded viewpoint is neither SV-invariant nor spectral-norm contractive (\Cref{app:lowrank_embedded}).

\emph{Our method (iMuon)} resolves both issues by working on the \emph{quotient} $\M_r=(\RR^{m\times r}_*\times\RR^{r\times n}_*)/\text{GL}(r)$ with the coupled metric.
The intrinsic scaling $G_x^{1/2}$ converts the factor gradients into Euclidean coordinates via the Gram-root preconditioning $(AA^\top)^{-1/2}$, and the product-norm decoupling yields independent closed-form LMOs on each factor block.
Unlike Riemannion, the spectral-norm constraint is \emph{exactly} satisfied on each block (no approximation), and the coupled metric provides automatic preconditioning that makes the convergence rate independent of factor conditioning ($C_\varphi=2r$).
Unlike Spectron, no power iteration or adaptive rescaling is needed: the $\text{GL}(r)$-invariance of the intrinsic norm handles factor-scale dependence algebraically rather than numerically.

\noindent
\paragraph{Side-by-side update rules.}

\begin{center}
\small
\begin{tabular}{lll}
\toprule
\textbf{Method} & $\dot{B}$ & $\dot{A}$ \\
\midrule
Factor-wise Muon & $\tau\,\Ortho(\nabla_X fA^\top)$ & $\tau\,\Ortho(B^\top\nabla_X f)$ \\[3pt]
Spectron & $\rho\,\Ortho(M_B)$ & $\rho\,\Ortho(M_A)$ \\[3pt]
Riemannion & \multicolumn{2}{c}{$\dot{X}=P_{T_{X}\M_r}(\Ortho_r(M_t))$, embedded manifold with $G_X=I$} \\[3pt]
Ours (iMuon) & $\tau\,\Ortho(\nabla_X f\hat{Q}_{A})\,R_{A}^{-\top}$ & $\tau\,R_{B}^{-1}\,\Ortho(\hat{Q}_{B}^\top\nabla_X f)$ \\
\bottomrule
\end{tabular}
\end{center}
where $A^\top=\hat{Q}_{A}R_{A}$ and $B=\hat{Q}_{B}R_{B}$ are thin QR factorizations, $M_A,M_B$ are per-factor momentum buffers, and $\rho=\eta/(\|A\|_2+\|B\|_2+1)$ is Spectron's adaptive constraint radius.
The key differences are:
\begin{itemize}
\item \emph{Ortho input:} factor-wise Muon and Spectron both apply $\Ortho$ to raw factor gradients (or momenta).
Our method applies $\Ortho$ to the \emph{scaled} gradient $\nabla_X f\hat{Q}_{A}$, which projects $\nabla_X f$ onto the row space of $A$ with orthonormalized rows, a different polar direction in general.
\item \emph{Norm control mechanism:} factor-wise Muon uses a fixed $\tau$; Spectron uses the adaptive radius $\rho$ that shrinks as factors grow; Riemannion uses tangent-space projection after rank-$2r$ orthogonalization (approximate).
Our method uses the Gram-root preconditioning $R_A^{-\top}$ (equivalently $(AA^\top)^{-1/2}$), which absorbs the factor scale \emph{before} orthogonalization.
\item \emph{Reparameterization invariance:} our update is $\text{GL}(r)$-equivariant by \Cref{prop:symmetry}, i.e., $(\dot{B}^*,\dot{A}^*)\mapsto(\dot{B}^*N^{-1},N\dot{A}^*)$, so the induced update of $X=BA$ does not depend on the chosen factorization.
Riemannion is also reparameterization-invariant (it works on $\M_r$ directly), but violates the spectral-norm constraint after projection.
Factor-wise Muon and Spectron are not reparameterization-invariant: different factorizations of the same $X$ yield different updates.
Spectron compensates via the adaptive $\rho$, which bounds the \emph{ambient} update norm but does not make the \emph{direction} invariant. For example, without momentum, the rescaling $(B,A)\mapsto(cB,A/c)$ with $c>0$ turns Spectron's ambient update $\rho\,[\Ortho(\nabla_X f A^\top)A+B\,\Ortho(B^\top\nabla_X f)]$ into $\rho'\,[c^{-1}\Ortho(\nabla_X f A^\top)A+c\,B\,\Ortho(B^\top\nabla_X f)]$, so the two terms are reweighted by $1/c$ and $c$, which the common scalar $\rho'$ cannot undo.
\end{itemize}

\noindent
\paragraph{Ambient update norm comparison.}
The ambient update $\dot{X}=\dot{B}A+B\dot{A}$ enters the convergence rate through the descent lemma.
We compute $\|\dot{X}\|_F^2$ for each method:

\emph{Our method.}
Define $Q_{A}=(AA^\top)^{-1/2}A\in\RR^{r\times n}$, which has orthonormal rows ($Q_{A}Q_{A}^\top=I_r$).
Then $\dot{B}A=\tau\,\Ortho(H_{B})\,(AA^\top)^{-1/2}A=\tau\,\Ortho(H_{B})\,Q_{A}$.
Since $\Ortho(H_{B})\in\RR^{m\times r}$ has orthonormal columns:
\[
\|\dot{B}A\|_F^2 = \tau^2\,\tr(Q_{A}^\top\,I_r\,Q_{A}) = \tau^2\,\tr(I_r) = \tau^2 r.
\]
Similarly $\|B\dot{A}\|_F^2=\tau^2 r$.
By the triangle inequality: $\|\dot{X}\|_F\le 2\tau\sqrt{r}$, so $\|\dot{X}\|_F^2\le 4r\tau^2$.
This bound depends only on the rank $r$, not on the factor norms. The Gram root converts $A$ into the partial isometry $Q_{A}$, exactly canceling the factor norm dependence.

\emph{Factor-wise Muon.}
$\dot{B}A=\tau\,\Ortho(\nabla_X fA^\top)A$.
Since $\Ortho(\nabla_X fA^\top)$ has orthonormal columns, $\|\dot{B}A\|_F^2=\tau^2\|A\|_F^2$.
Similarly $\|B\dot{A}\|_F^2=\tau^2\|B\|_F^2$.
Thus $\|\dot{X}\|_F^2\le 2\tau^2(\|A\|_F^2+\|B\|_F^2)$, \emph{grows with factor norms}.

\emph{Riemannion.}
$\dot{X}=P_{T_{X}\M_r}(\Ortho_r(M_t))$, where $\Ortho_r$ replaces the first $2r$ singular values of $M_t$ with $1$.
The tangent-space projection after rank-$2r$ orthogonalization perturbs singular values (they lie in $(0.9,1.1)$ rather than exactly $1$), so the spectral-norm constraint $\|\dot{X}\|_2\le\tau$ is only approximately satisfied.

\noindent
\paragraph{Spectral-norm control.}
Our method also controls the spectral norm of the ambient update: $\|\dot{X}\|_2\le\|\dot{B}A\|_2+\|B\dot{A}\|_2\le 2\tau$, since $\|\dot{B}A\|_2=\tau\|\Ortho(H_{B})\,Q_{A}\|_2\le\tau\cdot 1\cdot 1=\tau$ (orthonormal columns times orthonormal rows).
Factor-wise Muon gives $\|\dot{X}\|_2\le\tau(\|A\|_2+\|B\|_2)$, which can be arbitrarily large.

\noindent
\paragraph{Nuclear norm ($\varphi=\|\cdot\|_*$).}
The nuclear instance replaces $\Ortho(H_{B})$ with the leading left/right singular vectors $u_1v_1^\top$ of $H_{B}$, giving a rank-$1$ update per factor.
The ambient update $\dot{X}=\dot{B}A+B\dot{A}$ has rank $\le 2$ and $\|\dot{X}\|_F^2\le 4\tau^2$ (each rank-$1$ block contributes $\tau^2$, independent of factor norms).
This is the intrinsic analog of NuMuon~\citep{dolatabadi2026numuon}: NuMuon constrains the update with both a spectral-norm cap $\rho$ and a nuclear-norm budget $\tau$, yielding a rank-$k$ update ($k=\lfloor\tau/\rho\rfloor$) from the top-$k$ singular vectors of the momentum. In the rank-$1$ special case ($k=1$), i.e., $\tau = \rho$, NuMuon produces $u_1v_1^\top$, while our nuclear instance applies the rank-$1$ truncation to the \emph{scaled} gradient $H_{B}=\nabla_{B}f\,(AA^\top)^{-1/2}$, inheriting the same $\text{GL}(r)$-invariance and condition-number-free bound as the spectral instance.

\subsection{SPD manifold $S_{++}^n$: spectral LMO vs.\ natural gradient}\label{app:compare_spd}

On the SPD manifold with the affine-invariant metric, the standard Riemannian optimization method is the \emph{natural gradient}~\citep{amari1998natural}, which corresponds to the Frobenius instance ($\varphi=\|\cdot\|_F$) of our framework: $\xi^*\propto \text{grad}\,f/\|\text{grad}\,f\|_X$ where $\text{grad}\,f(X)=X(\nabla_X f)X$. The intrinsic spectral norm $\|X^{-1/2}\xi X^{-1/2}\|_2$ is the Finsler norm of the Thompson metric on $S_{++}^n$~\citep{thompson1963certain,mostajeran2024differential}, so the spectral instance is the steepest-descent direction of this Finsler geometry.
The natural gradient preserves the relative magnitudes of the scaled gradient's eigenvalues, directing more step budget toward larger-eigenvalue directions.

The spectral instance ($\varphi=\|\cdot\|_2$) of our framework produces a qualitatively different update: the matrix-sign $\xi^*=\tau\,X^{1/2}\sign(H)X^{1/2}$, where $H=X^{1/2}(\nabla_X f)X^{1/2}$.
This equalizes all eigenvalues to $\pm 1$, redistributing the step uniformly across eigendirections, the SPD analog of how Euclidean Muon replaces the gradient with its polar factor.

The difference is analogous to Muon vs.\ SGD in Euclidean space:
\begin{itemize}
\item The natural gradient preserves the magnitude profile of $H$'s eigenvalues, so large eigenvalues of $H$ receive proportionally larger step components.
\item The matrix-sign update equalizes all eigenvalues to $\pm 1$, redistributing the step uniformly across eigendirections.
This is more aggressive (higher squared Riemannian radius $C_\varphi=n$ vs.\ $1$) but drives all eigenvalues of the scaled gradient to zero simultaneously.
\end{itemize}

The naive alternative, applying $\Ortho$ or $\sign$ to the Euclidean gradient $\nabla_X f$ directly (without the scaling $X^{1/2}(\cdot)X^{1/2}$), would not be $\text{GL}(n)$-equivariant.
For example, under $X\mapsto NXN^\top$ the Euclidean gradient becomes $N^{-\top}(\nabla_X f)N^{-1}$, and $\sign(N^{-\top}(\nabla_X f)N^{-1})\neq N\,\sign(\nabla_X f)\,N^\top$ in general, whereas $\sign(H)=\sign(X^{1/2}(\nabla_X f)X^{1/2})$ transforms covariantly because the scaling absorbs the change of basis.

\noindent
\paragraph{Nuclear norm ($\varphi=\|\cdot\|_*$).}
The nuclear instance selects the leading eigenpair $(q_1,\lambda_1)$ of $H=X^{1/2}(\nabla_X f)X^{1/2}$ and produces a rank-$1$ symmetric update $\xi^*=\tau\,\sign(\lambda_1)\,X^{1/2}q_1q_1^\top X^{1/2}$.
This concentrates the entire norm bound on the eigendirection of $H$ with the largest eigenvalue magnitude.
Compared to the spectral instance ($C_\varphi=n$, all eigendirections active), the nuclear instance has $C_\varphi=1$ and converges in the spectral norm of the gradient ($\varphi^\circ=\|\cdot\|_2$), which is weaker but cheaper.

\subsection{Stiefel manifold $\text{St}(m,r)$: product norm vs.\ joint norm}\label{app:compare_stiefel}

On the Stiefel manifold with $G_X=I$, the natural formulation is the \emph{joint-norm} LMO: constrain $\varphi(\xi)\le\tau$ over the full tangent space $T_X\text{St}\subset\RR^{m\times r}$. As discussed in \Cref{app:stiefel}, this tangent space is neither SV-invariant nor spectral-norm contractive, so \Cref{prop:tangent_compat} does not apply directly. For $\varphi=\|\cdot\|_2$, this problem has been studied in the context of manifold Muon. \citet{bernstein2025manifolds} solves it via dual ascent on a symmetric Lagrange multiplier, and \citet{cesista2025stiefel,su2025stiefel} propose other iterative procedures. A concurrent work~\citep{solonko2026skewon} shows that the joint-norm problem admits an exact closed-form solution, which we discuss at the end of this section.

Our framework instead uses the tangent space decomposition $\xi=XA+X_\perp B$ with the \emph{product norm} $\max\{\varphi(A),\varphi(B)\}\le\tau$. This relaxation decouples the LMO into closed-form subproblems on the two blocks, and the skew block $A$, which rotates within the current column space, and the normal block $B$, which moves orthogonally to it, each receive their own budget. For $\varphi=\|\cdot\|_2$, since $\max\{\|A\|_2,\|B\|_2\}\le\|[A;B]\|_2\le\sqrt{2}\max\{\|A\|_2,\|B\|_2\}$, the product constraint permits tangent vectors with $\|\xi\|_2$ up to $\sqrt{2}\tau$. For $\varphi=\|\cdot\|_F$, let $S=\Skew(X^\top\nabla_X f)$ and $N=X_\perp^\top\nabla_X f$ be the gradient blocks paired with $A$ and $B$. The product-norm solution gives $\|A^*\|_F=\|B^*\|_F=\tau$, whereas the joint-norm solution, which is the normalized Riemannian gradient, keeps the ratio $\|A^*\|_F/\|B^*\|_F=\|S\|_F/\|N\|_F$. The two solutions are parallel only when $\|S\|_F=\|N\|_F$, and even then the product-norm solution is $\sqrt{2}$ times the joint-norm solution.

The relaxation also enlarges the squared Riemannian radius $C_\varphi$ (\Cref{def:effdim}), which enters the rate constant of \Cref{thm:general}. Since the product-norm ball contains the joint-norm ball, $C_\varphi$ under the product norm is at least that under the joint norm. For $\varphi=\|\cdot\|_2$, the product norm gives $C_\varphi=2\lfloor r/2\rfloor+\min(m-r,r)$, whereas the joint norm gives at most $r$, since every $\xi\in\RR^{m\times r}$ with $\|\xi\|_2\le 1$ satisfies $\|\xi\|_F^2\le r$. For example, when $r$ is even and $m\ge 2r$, the two radii are $2r$ and $r$, respectively. For $\varphi=\|\cdot\|_F$ and $\varphi=\|\cdot\|_*$, the joint norm gives $C_\varphi=1$, whereas the product norm gives $C_\varphi=2$ and $C_\varphi=3/2$, respectively, when $r\ge 2$ and $m>r$.

\noindent
\paragraph{Nuclear norm ($\varphi=\|\cdot\|_*$).}
The nuclear instance concentrates on the leading singular direction of each block. Since the singular values of $S$ come in equal pairs, the skew update is the rank-$2$ skew-symmetric matrix $A^*=\frac{\tau}{2}(u_1v_1^\top+u_2v_2^\top)$, where $u_1,u_2$ and $v_1,v_2$ are the left and right singular vectors of $S$ for its leading pair, so that $\|A^*\|_*=\tau$ and $\|A^*\|_F^2=\tau^2/2$. The normal update is the rank-$1$ matrix $B^*=\tau\,p_1q_1^\top$, where $p_1$ and $q_1$ are the leading left and right singular vectors of $N$. We note that the joint-norm construction of \citet{solonko2026skewon} does not extend to the nuclear norm (see the end of this section), so the product norm provides the closed-form nuclear instance on $\text{St}(m,r)$.

\noindent
\paragraph{Comparison with MCSD/SPEL~\citep{yang2026mcsd}.}
MCSD (Manifold Constrained Steepest Descent) is a recent single-loop framework for norm-constrained optimization on embedded submanifolds. {Its Stiefel specialization, SPEL, computes the spectral-norm LMO of the Riemannian gradient $\nabla_\M f(X_t)=P_{T_{X_t}\text{St}}(\nabla f(X_t))$ in the \emph{ambient} Euclidean space, $d_t=-\Ortho(\nabla_\M f(X_t))$, and then projects back to $\textnormal{St}(m,r)$ via the polar map: $X_{t+1}=\Ortho(X_t+\eta_t\, d_t)=\Ortho\bigl(X_t-\eta_t\,\Ortho(\nabla_\M f(X_t))\bigr)$.
A defining design choice of MCSD is that the update direction $d_t$ \emph{need not} lie in $T_X\,\textnormal{St}$; manifold feasibility is restored solely by the polar projection step, which is what makes MCSD single-loop.
}

Both SPEL and our Stiefel LMO (\Cref{lemma:stiefel_spectral})
avoid the joint-norm tangent space problem
$\max_{\xi\in T_X\text{St},\,\|\xi\|_2\le\tau}\inner{\xi}{\nabla f}$, but they use different relaxation strategies:
\begin{itemize}
\item \emph{Our approach} decomposes the tangent space into skew and normal blocks with a product norm $\max(\|A\|_2,\|B\|_2)\le\tau$, solving each block's LMO in closed form.
The update remains in $T_X\text{St}$ by construction, and each block receives independent spectral control.
\item \emph{SPEL} {drops the tangent-space membership constraint, solving the LMO over all of $\RR^{m\times r}$, and compensates with the polar manifold projection. This yields a simpler per-iteration formula (one $\Ortho$ of an $m\times r$ matrix to obtain the LMO direction, plus one $\Ortho$ for the polar projection) but couples the skew (rotation within $\mathrm{col}(X)$) and normal (extension into $\mathrm{col}(X)^\perp$) contributions into a single ambient direction. The MCSD framework also extends beyond Stiefel to other embedded submanifolds (Grassmann, sphere) but does not natively handle quotient geometries such as $\M_r$ with the coupled metric.}
\end{itemize}

\noindent
\paragraph{Joint-norm closed form~\citep{solonko2026skewon}.}
For $\varphi=\|\cdot\|_2$, a concurrent work~\citep{solonko2026skewon} solves the joint-norm LMO in closed form by lifting the tangent space to $m\times m$ skew-symmetric matrices. The map $Y\mapsto YX$ sends skew-symmetric matrices onto $T_X\text{St}$, and a norm-preserving dilation argument~\citep{davis1982norm} shows that it maps the spectral-norm ball of skew-symmetric matrices onto the spectral-norm ball of $T_X\text{St}$. Since $\inner{YX}{\nabla_X f}=\inner{Y}{\Skew(\nabla_X f\,X^\top)}$ for skew-symmetric $Y$, the joint-norm LMO reduces to a spectral-norm LMO over skew-symmetric matrices, which is solved as in \Cref{lem:skew_compat}. This gives
\[
\xi^*=\tau\,\Skew\bigl(\Ortho(\Skew(\nabla_X f\,X^\top))\bigr)\,X.
\]
This construction relies on the norm-preserving dilation property of the spectral norm, and for other unitarily invariant norms the lifted problem differs from the intrinsic LMO~\eqref{eqn:main_lmo}. For example, for $\varphi=\|\cdot\|_F$ the lift yields a ball of the canonical metric, and for $\varphi=\|\cdot\|_*$ it doubles the norm of $\xi=X_\perp B$.

\subsection{Grassmann manifold $\text{Gr}(m,r)$: connections to subspace tracking}\label{app:compare_grassmann}

The Grassmann manifold $\text{Gr}(m,r)$ represents $r$-dimensional subspaces of $\RR^m$.
With $G_X=I$ and $X^\top X=I_r$, the horizontal space $\{\xi:X^\top\xi=0\}$ is SV-invariant, so \Cref{prop:tangent_compat} applies directly and no product-norm relaxation is needed.
The three norm instances of our framework produce geometrically distinct updates that connect to known methods in subspace estimation:
\begin{itemize}
\item \emph{Spectral} ($\varphi=\|\cdot\|_2$): $\xi^*=\tau\,\Ortho((I-XX^\top)\nabla_X f)$, an $m\times r$ matrix with orthonormal columns in $\text{colspan}(X)^\perp$, updating all $\min(m{-}r,r)$ subspace directions simultaneously.
\item \emph{Frobenius} ($\varphi=\|\cdot\|_F$): $\xi^*=\tau\,(I-XX^\top)\nabla_X f/\|(I-XX^\top)\nabla_X f\|_F$, the normalized Riemannian gradient on $\text{Gr}(m,r)$~\citep{edelman1998geometry}, whose discrete updates are closely related to Oja's subspace rule.
\item \emph{Nuclear} ($\varphi=\|\cdot\|_*$): $\xi^*=\tau\,u_1v_1^\top$ from the SVD of $(I-XX^\top)\nabla_X f$, a rank-$1$ horizontal update ($C_\varphi=1$).
\end{itemize}
The nuclear instance is closely related to GROUSE (Grassmannian Rank-One Update Subspace Estimation)~\citep{balzano2010grouse}, which performs online subspace tracking via rank-$1$ updates on $\text{Gr}(m,r)$.
Both produce a rank-$1$ tangent vector in the horizontal space; the difference is that GROUSE derives its direction from a \emph{single observation's residual} (streaming setting with one sample at a time), whereas our LMO uses the \emph{leading singular vectors of the full horizontal gradient} $(I-XX^\top)\nabla_X f$ (batch or mini-batch).
Geometrically, each rank-$1$ update adds one new direction $u_1\perp\text{colspan}(X)$ while downweighting the direction $Xv_1$ of the current subspace.
Since $G_X=I$ on $\text{Gr}(m,r)$, no scaling is needed, the Euclidean and intrinsic nuclear norms coincide, and the nuclear iMuon on $\text{Gr}(m,r)$ is simply NuMuon restricted to the horizontal space.

\section{Background on Riemannian Optimization}\label{app:background}

This appendix provides a more thorough discussion of the background material summarized in \Cref{sec:setup}, aimed at readers less familiar with Riemannian optimization on matrix manifolds.

\subsection{Riemannian optimization: key concepts}

Riemannian optimization~\citep{absil2008optimization,boumal2023intromanifolds} generalizes gradient-based methods to curved spaces (manifolds) by replacing Euclidean gradients with Riemannian gradients and straight-line updates with retractions.
The key ingredients are:
\begin{itemize}
\item A \emph{tangent space} $T_x\M$: the set of valid update directions at $x\in\M$.
\item A \emph{Riemannian metric} $g_x$: a smoothly varying inner product on $T_x\M$ that measures lengths and angles.
The Riemannian gradient $\text{grad}\,f(x)$ is the unique tangent vector satisfying $g_x(\text{grad}\,f,\zeta)=Df(x)[\zeta]$ for all $\zeta\in T_x\M$.
\item A \emph{retraction} $R_x:T_x\M\to\M$: a map that moves from $x$ along a tangent direction back onto $\M$, generalizing the Euclidean update $x+\eta\,\xi$.
\end{itemize}
Standard convergence analyses~\citep{boumal2023intromanifolds} assume \emph{retraction-$L$-smoothness}: $f(R_x(\xi))\le f(x)+Df(x)[\xi]+\frac{L}{2}\|\xi\|_x^2$.
Under this condition, Riemannian gradient descent achieves an $O(\sqrt{L/T})$ rate to an approximate first-order stationary point, matching the Euclidean rate.

\subsection{Quotient manifolds and scale invariance}\label{app:quotient_geometry}

Many structured parameterizations involve redundancies.
For example, a rank-$r$ matrix $X=BA$ can be factored as $(B,A)$ or $(BN^{-1},NA)$ for any invertible $N\in\text{GL}(r)$. These are different points in parameter space representing the same $X$.
The mathematically clean way to handle this is the \emph{quotient manifold}: define an equivalence relation $(B,A)\sim(BN^{-1},NA)$ and work on the quotient space $\M_r=(\RR^{m\times r}_*\times\RR^{r\times n}_*)/\text{GL}(r)$, whose points are equivalence classes $[(B,A)]$.
The tangent space of the total space $\overline{\M}$ at a representative $x$ splits into a \emph{vertical space} $\mathcal{V}_x$ (directions along the symmetry orbit, which do not change the equivalence class) and a complementary \emph{horizontal space} $\mathcal{H}_x$ (directions that represent genuine motion on the quotient); only horizontal tangent vectors carry information about the quotient geometry.

For optimization on a quotient manifold to be well-defined, the Riemannian metric must be \emph{invariant} under the equivalence group: $g_{(BN^{-1},NA)}(\cdot,\cdot)=g_{(B,A)}(\cdot,\cdot)$ (after appropriate pushforward of tangent vectors).
If the metric is not invariant, the Riemannian gradient depends on which representative $(B,A)$ of the equivalence class is used, a mathematical inconsistency that manifests as training instability in practice~\citep{janson2026stabilizing}.

The \emph{coupled metric} of~\citet{mishra2014fixed,mishra2016riemannian} achieves invariance by weighting each factor by the \emph{other} factor's Gram matrix: $g(\dot{B},\dot{A})=\tr(\dot{B}^\top \dot{B}\,AA^\top)+\tr(B^\top B\,\dot{A}\dot{A}^\top)$.
Under $(B,A)\mapsto(BN^{-1},NA)$, the cross-Gram weighting transforms as $AA^\top\mapsto NAA^\top N^\top$, which exactly compensates the change $\dot{B}\mapsto\dot{B}N^{-1}$, preserving lengths.
The resulting Riemannian gradient has the form $\nabla_{B}f\,(AA^\top)^{-1}$ (for the $B$-factor), which can be recognized as a preconditioned gradient that automatically equalizes the contribution of large and small singular directions of $A$.

\section{Extended Related Work}\label{app:related_work_extended}

\subsection{Riemannian Frank--Wolfe methods}

\citet{weber2023riemannian} study Riemannian Frank--Wolfe methods for optimization over geodesically convex constraint sets on manifolds.
Their LMO operates over a \emph{feasible set on the manifold} to keep iterates within a constraint (e.g., a geodesic ball or a convex subset of the manifold).
Our LMO is structurally different: it constrains a norm in the \emph{tangent space} to select the step direction, then retracts onto the manifold.
This is a trust-region-like norm bound on the update magnitude rather than feasibility of the iterates.
The two approaches are complementary; ours targets unconstrained manifold optimization with spectral control of the step.

\subsection{Euclidean spectral optimizers and LMO}

The idea of using spectral-norm (Schatten-$\infty$) steepest descent for matrix-valued parameters in deep learning predates the recent Muon line of work.
\citet{carlson2015psd} observe that the objective function in feedforward neural networks admits a tighter majorization bound under the Schatten-$\infty$ norm than under the Frobenius norm, and derive the corresponding steepest-descent operator, which is precisely the orthogonal polar factor $\Ortho(\cdot)$ used by Muon.
They further combine this non-Euclidean gradient with element-wise adaptive learning rates (RMSprop, ADAgrad), yielding preconditioned spectral descent algorithms (RMSspectral, ADAspectral) that demonstrate speedups on RBMs, feedforward nets, and CNNs.
Our framework can be seen as replacing their element-wise diagonal preconditioning with a principled Riemannian metric $G_x$ that is determined by the manifold geometry: where Carlson et~al.\ scale each matrix entry independently, the intrinsic norm $\varphi(G_x^{1/2}\xi)$ scales the entire tangent vector via the metric operator, simultaneously achieving preconditioning, symmetry invariance on quotient manifolds, and closed-form solutions.

More recently, \citet{bernstein2025old} formalize the connection between the polar factor and steepest descent under the spectral norm, and~\citet{pethick2025lmo} {(Scion) generalize the LMO viewpoint to a broad family of induced \emph{operator} norms in Euclidean space, including the spectral norm, sign / $\ell_\infty$, ColNorm, and RMS-to-RMS norms, which goes beyond the unitarily invariant family used in our framework.}
Our intrinsic LMO~\eqref{eqn:main_lmo} recovers all of these {unitarily invariant instances} as the special case $G_x=I$ on $\RR^{p\times q}$, and extends them to non-trivial metrics and manifold-structured parameters.

\subsection{Muon extensions for low-rank/LoRA fine-tuning}

{A first wave of works ports Muon-like spectral control to LoRA-style and natively low-rank training. Riemannion~\citep{bogachev2026riemannion} works on the embedded fixed-rank manifold $\M_r\subset\RR^{m\times n}$ with the ambient Frobenius metric, projects the Euclidean gradient to $T_X\M_r$, and approximates a tangent space spectral LMO by replacing all $\le 2r$ singular values of the tangent space momentum with $1$ via a tangent space `$\Ortho_r$' operator and re-projecting; the spectral-norm constraint is satisfied only approximately (singular values of the resulting tangent space step lie in $(0.9,1.1)$ in their experiments). \citet{kang2026spectral} prove a continuous-time convergence dichotomy for SpecGF on $\|BA-Y\|_F^2$ (see \Cref{app:connections}). \citet{janson2026stabilizing} (Spectron) renormalizes the orthogonalized step of each factor by $\rho=\eta/(\|A\|_2+\|B\|_2+1)$, with $\|A\|_2,\|B\|_2$ estimated by a single power iteration per step. \citet{park2025stiefel} (Stiefel-LoRA) restrict LoRA's $B$ factor to the Stiefel manifold with Riemannian SGD/AdamW. NuMuon~\citep{dolatabadi2026numuon} adds a nuclear-norm budget to Muon's spectral LMO, yielding rank-$\lfloor\tau/\rho\rfloor$ updates (nuclear budget $\tau$, spectral cap $\rho$) suitable for compressible LLMs. \citet{wang2026taming} factorize the EMA momentum buffer to compress optimizer-state memory in both Adam-style and Muon-style optimizers. \citet{schotthoefer2025geometric} provide a geometric framework for momentum-based low-rank training. None of these works simultaneously (i) preserves quotient symmetries on $\M_r$, (ii) gives an exact closed-form tangent space spectral LMO, and (iii) attains a condition-number-free rate, the three properties of our framework specialized to $\M_r$.
}

\subsection{Convergence theory of Muon}

{Discrete-time convergence guarantees for Muon and its variants have appeared in a sequence of works: \citet{li2025note} provide a first convergence note; \citet{sato2025convergence} establish convergence and a critical batch-size analysis under both Nesterov and weight-decay variants; \citet{shen2025muon,kim2026convergence,chen2026muon} treat Muon convergence under various smoothness assumptions, and \citet{ma2026preconditioning} analyze preconditioning benefits of spectral orthogonalization. \citet{riabinin2025gluon} (Gluon) bridge theory and practice in LMO-based optimizers. \citet{amsel2026polar} study optimal polynomial approximations of the matrix sign function used inside Muon. Our \Cref{thm:general,thm:stochastic} cover both deterministic and stochastic settings on Riemannian matrix manifolds, with the squared Riemannian radius $C_\varphi$ as the unifying complexity parameter, and recover the Euclidean Muon rate of \citet{bernstein2025old,shen2025muon} as the $G_x=I$ special case (\Cref{app:connections}).}

\subsection{Riemannian and manifold-aware optimizers for LLMs}

{A complementary line of work injects manifold-derived primitives into LLM-scale pretraining without strictly constraining iterates to a manifold. Mano~\citep{gu2026mano} {projects momentum onto the tangent space of the Oblique manifold and applies alternating column-wise and row-wise normalization across odd/even iterations, which imposes only a soft manifold constraint on the update direction (no retraction is applied to the iterate), so the parameter trajectory itself remains Euclidean.} SSO~\citep{xie2026sso} constrains weights to the spectral sphere $\mathcal{S}_R=\{X:\|X\|_2=R\}$ and runs steepest descent under the spectral norm via bisection on a single Lagrange multiplier; we discuss this case (where SV-invariance fails) in \Cref{app:spectral_sphere}. \citet{yang2026spectral,yang2022mup} provide the spectral / $\mu$P framework that motivates much of this line. These works are largely orthogonal to our setting, which targets problems where the parameters genuinely live on a Riemannian manifold (LoRA-style fixed-rank factorizations, SPD covariance descriptors, subspaces). Within that setting, \citet{tong2021scaledgd,zhang24riemannianlora,mishra2014fixed,mishra2016riemannian} are the closest direct precursors of our Frobenius-norm instance on $\M_r$.
}

\section{Experimental Details for LoRA Fine-Tuning}\label{app:experimental_details}

\paragraph{Compared methods.}
Each LoRA adapter parameterizes the trainable update as $X=BA$, where
$B\in\RR^{m\times r}$ and $A\in\RR^{r\times n}$, while the pretrained
weight is kept fixed. This representation exposes the fixed-rank structure used by the geometry-aware methods. We compare Euclidean,
Riemannian-preconditioned, and spectral-update methods under this common
LoRA parameterization.

SGD and AdamW~\citep{zhang24riemannianlora} update the LoRA factors directly as ordinary Euclidean parameters. Scaled GD and Scaled AdamW
\citep{mishra2012geometry,tong2021scaledgd,zhang24riemannianlora} use the
$\textnormal{GL}(r)$-invariant preconditioning induced by the quotient geometry
of $\M_r$: Scaled GD corresponds to the Frobenius-norm instance of the
intrinsic LMO, which recovers the Riemannian gradient direction, while
Scaled AdamW uses the same geometry-aware preconditioning inside an
AdamW-style update. Factor-wise Muon~\citep{jordan2024muon,bernstein2025old, kang2026spectral,janson2026stabilizing} applies the Euclidean spectral LMO factor-wise as in~\eqref{eqn:factorwise_muon}. Spectron~\citep{janson2026stabilizing} applies the same factor-wise spectral LMO and rescales the step by $\rho=\eta/(\|A\|_2+\|B\|_2+1)$.
In contrast, {iMuon} is the spectral-norm instance of our intrinsic
LMO on $\M_r$ (\Cref{cor:fixed_rank}). We also include Riemannion~\citep{bogachev2026riemannion}, reproduced from the authors'
released code, as an embedded fixed-rank Muon baseline with approximate
tangent-space orthogonalization.

\paragraph{Momentum implementation.}\label{app:momentum}
For the Adam family methods, we use standard Adam first- and second-moment updates. For iMuon and factor-wise Muon, we use the \textit{Nesterov} SGD style momentum update. For Spectron, we follow the momentum recurrence of \citet{janson2026stabilizing} with $\beta=0.95$. For a LoRA pair $(B_t,A_t)$ with $X_t=B_tA_t$ at iteration $t$, let
\[
G_{B,t}=\nabla_X f(X_t)A_t^\top,
\qquad
G_{A,t}=B_t^\top\nabla_X f(X_t)
\]
denote the Euclidean gradients of the LoRA factors. For factor-wise Muon, the momentum buffers are
\[
M_{B,t}=\beta M_{B,t-1}+G_{B,t},
\qquad
M_{A,t}=\beta M_{A,t-1}+G_{A,t},
\] with $\beta=0.95$ and the quantities passed to its spectral step are
\[
\widetilde{G}_{B,t}=G_{B,t}+\beta M_{B,t},
\qquad
\widetilde{G}_{A,t}=G_{A,t}+\beta M_{A,t}.
\]

Factor-wise Muon applies the spectral LMO separately to
these two factor directions. iMuon instead uses projection-transported momentum. To describe its construction, consider full-rank factors and define
\[
P_t=B_t^\top B_t,\qquad Q_t=A_tA_t^\top.
\]
We transport the previous momentum pair $Z=M_{t-1}$ by projecting it into the horizontal space at the current point $(B_t,A_t)$:
\[
\mathcal{T}_t(Z)
=\bigl(Z_B-B_t\Omega,\;Z_A+\Omega A_t\bigr),
\]
where
\[
\Omega=\frac12\left(
P_t^{-1}B_t^\top Z_B-Z_AA_t^\top Q_t^{-1}
\right).
\]
With the metric-preconditioned gradient pair
\[
R_t=\bigl(G_{B,t}Q_t^{-1},\;P_t^{-1}G_{A,t}\bigr),
\]
we update the momentum buffer and form the Nesterov direction as
\[
M_t=\beta\mathcal{T}_t(M_{t-1})+(1-\beta)R_t,
\qquad
D_t=\beta M_t+(1-\beta)R_t,
\]
with $M_0=(0,0)$ and $\beta=0.95$. Here $P_t$ and $Q_t$ are invertible since the LoRA factors have full rank $r$. The map $\mathcal{T}_t$ is the metric-orthogonal projection onto the horizontal space at $(B_t,A_t)$ and $R_t$ is the Riemannian gradient $G^{-1}\nabla f$, so $D_t$ transforms as a tangent vector under $(B,A)\mapsto(BN^{-1},NA)$ and the resulting update is representative-invariant. The intrinsic LMO of \Cref{cor:fixed_rank} is then applied to $D_t$ in place of the gradient pair.

\paragraph{Compute Resource.}\label{app:llm-compute}

All the experiments are run on Nvidia GPU clusters using L40, A40, A100, RTX 6000 GPUs depending on availability of the resources. The E2E~NLG experiments use single-GPU training. For GLUE, all tasks except MNLI and QQP use single-GPU training, while MNLI and QQP use two-GPU DDP. No hardware specific tuning is performed and the same training code path and optimizer settings are used across GPU types.

We additionally compare the per-batch runtime of factor-wise Muon, iMuon, and Riemannion (Adam)
over one E2E NLG epoch at batch size 8. We profile each batch into three
stages: forward pass, backward pass, and optimizer step. The optimizer-step
overhead motivates a CholeskyQR variant of iMuon, which replaces the thin QR
by a Cholesky-based factorization of $AA^\top$.  \Cref{tab:e2e-runtime-profile}
reports per-stage means and standard deviations for factor-wise Muon, iMuon
(PyTorch thin-QR), iMuon (CholeskyQR), and Riemannion (Adam).

\begin{table}[t]
\centering
\caption{Per-batch runtime profile on E2E NLG for one training epoch with
batch size \(8\). Timings are reported in milliseconds as mean \(\pm\)
standard deviation over training batches.}
\label{tab:e2e-runtime-profile}
\small
\setlength{\tabcolsep}{6pt}
\begin{tabular}{lcccc}
\toprule
Stage & Factor-wise Muon & iMuon (QR) & iMuon (CholeskyQR) & Riemannion (Adam) \\
\midrule
Forward pass
& \(135.97 \pm 1.24\)
& \(133.45 \pm 1.30\)
& \(135.48 \pm 1.42\)
& \(127.43 \pm 0.32\)\\
Backward pass
& \(158.00 \pm 0.34\)
& \(157.41 \pm 0.31\)
& \(157.45 \pm 0.28\)
& \(155.72 \pm 0.68\)\\
Optimizer step
& \(3.22 \pm 0.09\)
& \(13.08 \pm 0.14\)
& \(5.49 \pm 0.13\)
& \(54.16 \pm 0.30\)\\
Total per batch
& \(297.20 \pm 1.21\)
& \(303.95 \pm 1.26\)
& \(298.41 \pm 1.38\)
& \(337.31 \pm 1.00\)\\
\bottomrule
\end{tabular}
\end{table}

\subsection{E2E NLG }

\paragraph{Setup}
We fine-tune GPT-2 Medium (355M parameters) on the E2E NLG
Challenge~\citep{novikova2017e2e}, a table-to-text generation task with
approximately 42K training and 4.7K test examples in the restaurant domain.
LoRA adapters are applied to the query and value projections within the
merged \texttt{c\_attn} linear layer, following LoRA~\citep{hu2022lora} and Riemannian LoRA~\citep{zhang24riemannianlora}.
The LoRA configuration uses rank $r=4$, scaling $\alpha=32$, and dropout $0.1$. Training runs for 5 epochs with batch size 8, a linear learning-rate schedule with 500 warmup steps, weight decay $0.01$, label smoothing $0.1$, and no gradient clipping. Generation uses beam search with beam size 10, length penalty $0.8$, and no-repeat $4$-gram decoding. All single-seed comparisons use seed 110, which is also the seed used by~\citet{zhang24riemannianlora}. We compute BLEU, NIST, METEOR, ROUGE-L, and CIDEr using the standard \texttt{e2e-metrics} package.

Learning rates are tuned separately for each method. Following the E2E reporting protocol, we report the best test-set BLEU run from the sweep, with the remaining metrics computed from the same run. The full sweeps are reported in \Cref{tab:e2e-lr-nomom,tab:e2e-lr-mom}. The SGD, Scaled GD, AdamW, and Scaled AdamW entries in \Cref{tab:e2e-main} are taken from \citet{zhang24riemannianlora}.

\paragraph{Results.}
\Cref{tab:e2e-main} reports, for each method, the run with the best test-set BLEU, with the remaining metrics computed from the same run. Tables~\ref{tab:e2e-lr-nomom} and~\ref{tab:e2e-lr-mom} report the full learning-rate sweeps from which these runs are selected.
\Cref{tab:seed_sensitivity} reports the seed sensitivity of iMuon with and without momentum. For each configuration, the learning rate selected in \Cref{tab:e2e-main} is fixed and the run is repeated with seeds $0$, $42$, and $110$. Entries are the mean and sample standard deviation across the three seeds.
\Cref{tab:e2e-higher-rank} reports results without momentum at LoRA rank $32$ for factor-wise Muon, Riemannion, Spectron, and iMuon, with the learning rate selected by a sweep at this rank.

\begin{table}[t]
\centering
\caption{Seed sensitivity of iMuon on E2E NLG with LoRA rank
$r=4$. Entries report mean $\pm$ sample standard deviation
across seeds $0,42,110$. }
\label{tab:seed_sensitivity}
\small
\begin{tabular}{lccccc}
\toprule
Configuration & BLEU & NIST & METEOR & ROUGE-L & CIDEr \\
\midrule
Without momentum & $70.06\pm0.22$ & $8.83\pm0.02$ & $46.61\pm0.05$ & $71.54\pm0.46$ & $2.524\pm0.007$ \\
With momentum & $69.95\pm0.49$ & $8.79\pm0.04$ & $46.80\pm0.10$
& $71.62\pm0.31$ & $2.517\pm0.012$ \\
\bottomrule
\end{tabular}
\end{table}

\begin{table}[t]
\centering
\small
\caption{E2E NLG learning-rate sweep without momentum. For Riemannion (SGD), low-LR points
$\{5{\times}10^{-5},10^{-4}\}$ are omitted because pilot runs showed the
useful range was at substantially larger learning rates.}
\label{tab:e2e-lr-nomom}
\begin{tabular}{llccccc}
\toprule
Method & LR & BLEU & NIST & METEOR & ROUGE-L & CIDEr \\
\midrule
\multirow{4}{*}{Factor-wise Muon}
& $5{\times}10^{-5}$ & 64.22 & 8.33 & 42.46 & 67.49 & 2.21 \\
& $10^{-4}$          & 65.31 & 8.42 & 43.40 & 67.85 & 2.30 \\
& $5{\times}10^{-4}$ & 69.09 & 8.73 & 46.20 & 71.17 & 2.48 \\
& $10^{-3}$ & 70.02 & 8.81 & 46.77 & 71.68 & 2.53 \\
\midrule
\multirow{5}{*}{Riemannion (SGD)}
& $5{\times}10^{-4}$ & 65.05 & 8.40 & 43.30 & 68.05 & 2.29 \\
& $10^{-3}$          & 66.08 & 8.48 & 43.88 & 68.07 & 2.36 \\
& $5{\times}10^{-3}$ & 68.60 & 8.70 & 45.84 & 70.49 & 2.45 \\
& $10^{-2}$          & 69.40 & 8.75 & 46.28 & 71.18 & 2.47 \\
& $5{\times}10^{-2}$ & 70.02 & 8.78 & 46.79 & 71.99 & 2.52 \\
\midrule
\multirow{7}{*}{Spectron}
& $5\times10^{-5}$ & 53.00 & 3.67 & 35.04 & 64.71 & 1.34 \\
& $10^{-4}$ & 57.63 & 5.95 & 37.58 & 65.52 & 1.61 \\
& $5\times10^{-4}$ & 64.91 & 8.31 & 41.86 & 67.39 & 2.13 \\
& $10^{-3}$ & 65.49 & 8.40 & 43.44 & 68.33 & 2.30 \\
& $5\times10^{-3}$ & 68.17 & 8.62 & 45.78 & 70.13 & 2.43 \\
& $10^{-2}$ & 68.45 & 8.66 & 46.03 & 70.79 & 2.43 \\
& $5\times10^{-2}$ & 70.25 & 8.81 & 46.85 & 71.81 & 2.51 \\
\midrule
\multirow{7}{*}{iMuon}
& $5\times10^{-5}$ & 66.00 & 8.40 & 44.67 & 69.09 & 2.35 \\
& $10^{-4}$ & 66.00 & 8.41 & 45.22 & 69.32 & 2.40 \\
& $5\times10^{-4}$ & 69.85 & 8.80 & 46.49 & 71.41 & 2.52 \\
& $10^{-3}$ & 70.25 & 8.86 & 46.55 & 71.42 & 2.53 \\
& $5\times10^{-3}$ & 69.64 & 8.78 & 46.43 & 71.38 & 2.50 \\
& $10^{-2}$ & 69.98 & 8.79 & 46.70 & 71.48 & 2.51 \\
& $5\times10^{-2}$ & 70.30 & 8.85 & 46.62 & 72.03 & 2.53 \\
\bottomrule
\end{tabular}
\end{table}

\begin{table}[t]
\centering
\small
\caption{E2E NLG learning-rate sweep with momentum. Results of the no-momentum LR sweep were used to narrow the LR sweep of the momentum-based runs. }
\label{tab:e2e-lr-mom}
\begin{tabular}{llccccc}
\toprule
Method & LR & BLEU & NIST & METEOR & ROUGE-L & CIDEr \\
\midrule
\multirow{5}{*}{Factor-wise Muon}
& $5{\times}10^{-5}$ & 68.53 & 8.67 & 45.78 & 70.80 & 2.43 \\
& $10^{-4}$          & 69.57 & 8.77 & 46.40 & 71.53 & 2.50 \\
& $5{\times}10^{-4}$ & 69.97 & 8.81 & 46.70 & 71.44 & 2.53 \\
& ${7{\times}10^{-4}}$ & 70.26 & 8.83 & 46.76 & 71.75 & 2.52 \\
& $10^{-3}$          & 70.22 & 8.82 & 46.61 & 71.93 & 2.53 \\
\midrule
\multirow{7}{*}{Riemannion (Adam)}
& $5{\times}10^{-5}$ & 61.80 & 7.57 & 39.83 & 65.24 & 1.95 \\
& $10^{-4}$          & 64.59 & 8.16 & 41.17 & 67.11 & 2.10 \\
& $5{\times}10^{-4}$ & 65.23 & 8.38 & 43.87 & 67.82 & 2.31 \\
& $10^{-3}$          & 66.46 & 8.54 & 44.50 & 67.90 & 2.34 \\
& $5{\times}10^{-3}$ & 69.27 & 8.79 & 46.03 & 70.12 & 2.47 \\
& $10^{-2}$          & 69.58 & 8.79 & 46.57 & 70.85 & 2.52 \\
& $5{\times}10^{-2}$ & 70.21 & 8.83 & 46.55 & 71.87 & 2.53 \\
\midrule
\multirow{3}{*}{Spectron}
& $5\times10^{-3}$ & 69.28 & 8.74 & 46.31 & 71.10 & 2.51 \\
& $10^{-2}$ & 69.63 & 8.76 & 46.67 & 71.40 & 2.51 \\
& $5\times10^{-2}$ & 69.84 & 8.77 & 46.98 & 71.84 & 2.50 \\
\midrule
\multirow{4}{*}{iMuon}
& $5\times10^{-4}$ & 69.66 & 8.78 & 46.41 & 71.41 & 2.51 \\
& $10^{-3}$ & 70.16 & 8.80 & 46.57 & 71.87 & 2.53 \\
& $5\times10^{-3}$ & 70.19 & 8.83 & 46.83 & 71.84 & 2.54 \\
& $10^{-2}$ & 70.41 & 8.83 & 46.84 & 71.96 & 2.52 \\
\bottomrule
\end{tabular}
\end{table}

\begin{table}[t]
\centering
\caption{Higher-rank E2E NLG results without momentum at LoRA rank $32$.}
\label{tab:e2e-higher-rank}
\small
\setlength{\tabcolsep}{4pt}
\begin{tabular}{lccccc}
\toprule
Method & BLEU & NIST & METEOR & ROUGE-L & CIDEr \\
\midrule
\multicolumn{6}{c}{\textit{Rank \(32\)}} \\
Factor-wise Muon & 70.05 & 8.84 & 46.53 & 71.76 & 2.52 \\
Riemannion (SGD) & 70.27 & 8.83 & 46.60 & 71.91 & 2.52 \\
Spectron        & 69.95 & 8.81 & 46.70 & 71.83 & 2.52 \\
iMuon           & 70.05 & 8.83 & 46.41 & 71.56 & 2.51 \\
\bottomrule
\end{tabular}
\end{table}

\subsection{GLUE}

\paragraph{Setup}
We evaluate LoRA fine-tuning on the GLUE benchmark, following the Mistral-7B setup of \citet{zhang24riemannianlora}. We fine-tune a 4-bit quantized Mistral-7B-v0.1 model whose frozen base weights are stored in NF4 and use bfloat16 compute. LoRA adapters use rank $r=16$, scaling $\alpha=16$, dropout $0.05$, and no bias, and are applied to five target modules: \texttt{q\_proj}, \texttt{k\_proj}, \texttt{v\_proj}, \texttt{o\_proj}, and \texttt{gate\_proj}. Riemannion uses its doubled-rank PEFT parameterization with rank $r=32$ and $\alpha=32$, corresponding to effective rank $16$. The doubling is used for its tangent-space implementation rather than as a rank-$32$ forward adapter.

Training uses per-device batch size $8$ for up to $5$ epochs, with two-GPU DDP for MNLI and QQP and single-GPU training for the remaining seven tasks. We use a linear learning-rate schedule with no warmup, weight decay $0.01$, and maximum gradient norm $1$. AdamW-family methods use $\beta_1=0.9$, $\beta_2=0.999$, and $\epsilon=10^{-6}$. Early stopping is enabled with patience $10$, evaluation is performed every $500$ steps, and the best checkpoint is selected by the validation metric. Single-seed results use seed $42$. The seed-sensitivity experiments use seeds $0$, $42$, and $110$.

We evaluate on the nine GLUE tasks: MNLI, SST-2, MRPC, CoLA, QNLI, QQP, RTE, STS-B, and WNLI. Following the GLUE convention, the metric is accuracy for classification tasks, Matthews correlation for CoLA, and Pearson correlation for STS-B, and this metric is used both for checkpoint selection and for reporting. The reported average is the unweighted average across the nine tasks. The SGD, Scaled GD, AdamW, and Scaled AdamW entries in \Cref{tab:glue-main} are taken from \citet{zhang24riemannianlora}.

\paragraph{Results.}
\Cref{tab:glue-main} reports the GLUE results for 4-bit Mistral-7B fine-tuning with LoRA rank $r=16$. The learning rate is selected separately for each method and task by validation performance, since this setting is sensitive to the learning rate across both tasks and optimizers. This follows the task-specific selection of \citet{zhang24riemannianlora} for the same model. \Cref{tab:glue-selected-lr} lists the candidate set searched for each method and the learning rate selected for each task.

\Cref{tab:glue-seed-sensitivity} reports the seed sensitivity of iMuon with and without momentum. For each task, the learning rate selected with seed $42$ is fixed and the run is repeated with seeds $0$ and $110$. Entries are the mean and sample standard deviation across the three seeds.

\begin{table*}[t]
\centering
\caption{Learning-rate candidates and selected learning rates for the GLUE entries
evaluated in this work. Candidate sets summarize method-specific
searches and can differ across tasks; the task columns report the learning rate selected by validation performance. The SGD,
Scaled GD, AdamW, and Scaled AdamW rows in Table~\ref{tab:glue-main} are taken
from \citet{zhang24riemannianlora} and are therefore not included here.}
\label{tab:glue-selected-lr}
\scriptsize
\setlength{\tabcolsep}{3pt}
\renewcommand{\arraystretch}{1.18}
\resizebox{\textwidth}{!}{\begin{tabular}{llccccccccc}
\toprule
Method & LR candidates & MNLI & SST-2 & MRPC & CoLA & QNLI & QQP & RTE & STS-B & WNLI \\
\midrule
\multicolumn{11}{c}{\textit{Without momentum}}\\
Factor-wise Muon
& $\{5{\times}10^{-5},\,10^{-4}\}$
& $5{\times}10^{-5}$ & $5{\times}10^{-5}$ & $10^{-4}$ & $10^{-4}$
& $5{\times}10^{-5}$ & $5{\times}10^{-5}$ & $10^{-4}$ & $10^{-4}$ & $10^{-4}$ \\

Riemannion (SGD)
& \begin{tabular}[c]{@{}l@{}}
$\{5{\times}10^{-5},\,10^{-4},\,5{\times}10^{-4},\,10^{-3},$\\
$2{\times}10^{-3},\,5{\times}10^{-3},\,10^{-2},\,5{\times}10^{-2}\}$
\end{tabular}
& $5{\times}10^{-3}$ & $5{\times}10^{-3}$ & $5{\times}10^{-3}$ & $10^{-2}$
& $5{\times}10^{-3}$ & $10^{-2}$ & $5{\times}10^{-3}$ & $2{\times}10^{-3}$ & $10^{-2}$ \\

Spectron
& $\{10^{-4},10^{-3},10^{-2},5\times10^{-2}\}$
& $10^{-3}$ & $10^{-3}$ & $10^{-2}$ & $10^{-3}$
& $10^{-3}$ & $10^{-4}$ & $10^{-2}$ & $10^{-2}$ & $10^{-2}$ \\
iMuon
& $\{1,2,3,5,10\}\times10^{-5}$
& $2\times10^{-5}$ & $3\times10^{-5}$ & $10^{-4}$ & $5\times10^{-5}$
& $2\times10^{-5}$ & $10^{-5}$ & $10^{-4}$ & $10^{-4}$ & $10^{-4}$ \\
\midrule
\multicolumn{11}{c}{\textit{With momentum}}\\
Factor-wise Muon
& $\{5{\times}10^{-5},\,10^{-4}\}$
& $5{\times}10^{-5}$ & $5{\times}10^{-5}$ & $10^{-4}$ & $10^{-4}$
& $5{\times}10^{-5}$ & $5{\times}10^{-5}$ & $10^{-4}$ & $10^{-4}$ & $10^{-4}$ \\
Riemannion (Adam)
& \begin{tabular}[c]{@{}l@{}}
$\{10^{-5},\,2{\times}10^{-5},\,3{\times}10^{-5},\,4{\times}10^{-5},$\\
$5{\times}10^{-5},\,10^{-4},\,5{\times}10^{-4}\}$
\end{tabular}
& $10^{-4}$ & $4{\times}10^{-5}$ & $10^{-4}$ & $10^{-4}$
& $5{\times}10^{-5}$ & $10^{-4}$ & $10^{-4}$ & $5{\times}10^{-5}$ & $5{\times}10^{-4}$ \\
Spectron
& $\{10^{-4},10^{-3},10^{-2},5\times10^{-2}\}$
& $10^{-4}$ & $10^{-3}$ & $10^{-3}$ & $10^{-3}$
& $10^{-4}$ & $10^{-3}$ & $10^{-3}$ & $10^{-3}$ & $10^{-3}$ \\
iMuon
& $\{5\times10^{-5},10^{-4}\}$
& $5\times10^{-5}$ & $10^{-4}$ & $10^{-4}$ & $5\times10^{-5}$
& $5\times10^{-5}$ & $10^{-4}$ & $5\times10^{-5}$ & $10^{-4}$ & $10^{-4}$ \\
\bottomrule
\end{tabular}}
\end{table*}

\begin{table*}[t]
\centering
\caption{Seed sensitivity of iMuon on GLUE with 4-bit Mistral-7B
and LoRA rank $r=16$. Entries report mean $\pm$ sample standard
deviation across seeds $0,42,110$. Avg.\ is the unweighted nine-task average; its standard deviation is computed across the three per-seed averages.}
\label{tab:glue-seed-sensitivity}
\small
\setlength{\tabcolsep}{3pt}
\resizebox{\textwidth}{!}{\begin{tabular}{lcccccccccc}
\toprule
Configuration & MNLI & SST-2 & MRPC & CoLA & QNLI
& QQP & RTE & STS-B & WNLI & Avg. \\
\midrule
Without momentum & $89.46\pm0.51$ & $96.67\pm0.23$ & $90.28\pm0.62$ & $71.20\pm1.29$ & $94.40\pm0.25$ & $88.46\pm0.13$ & $90.13\pm1.10$ & $92.40\pm0.35$
& $81.22\pm2.93$ & $88.25\pm0.64$ \\
With momentum & $91.65\pm0.14$ & $97.52\pm0.24$ & $88.97\pm0.42$ & $69.86\pm1.74$
& $94.79\pm0.12$ & $91.95\pm0.29$ & $87.97\pm0.21$
& $92.34\pm0.36$ & $83.10\pm0.00$
& $88.68\pm0.18$ \\
\bottomrule
\end{tabular}}
\end{table*}

\section{Learning on Fixed-Rank Manifold}\label{app:fixed_rank_experiments}
\paragraph{Common LMO step-size convention.}
All LMO-based methods use constraint radius \(\tau=1\). The reported learning rate scales the resulting unit-radius direction. When Spectron is included, we use its adaptive extrinsic normalization separately.

\subsection{Matrix completion: synthetic experiments}
\label{app:fixed_rank_noisy_large_scale}

\paragraph{Setup.}
We study synthetic fixed-rank matrix completion at large scale, following the standard sampled-entry low-rank recovery setup~\citep{candes2012exact,keshavan2010matrix,recht2011simpler} and the quotient fixed-rank factorization used in Riemannian matrix completion~\citep{vandereycken2013low,mishra2012geometry,mishra2014fixed,mishra2016riemannian}. For each trial, we generate a rank-\(r\) ground-truth matrix \(X^*\in\RR^{m\times n}\), with \(m=n=5000\) and \(r=10\). The nonzero singular values are ordered as \(\sigma_1\ge\cdots\ge\sigma_r>0\), and we control the target conditioning by setting \(\kappa=\sigma_1(X^*)/\sigma_r(X^*)\). Given an observed index set \(\Omega\), we optimize
\[
    \min_{X=BA}
    f(X)
    =
    \frac{1}{2|\Omega|}
    \left\|\mathcal P_\Omega(X-Y)\right\|_F^2,
\]
where \(B\in\RR^{m\times r}\), \(A\in\RR^{r\times n}\), and \(Y\) contains the observed entries. In the noiseless case, \(Y_{ij}=X^*_{ij}\) for \((i,j)\in\Omega\). In the noisy case, \(Y_{ij}=X^*_{ij}+\varepsilon_{ij}\), with Gaussian observation noise scaled as described below. We use an oversampling ratio \(s\), meaning that \(|\Omega|=sr(m+n)\), and all reported runs use \(s=10\).
All methods are initialized from the rank-\(r\) SVD of the zero-filled observed matrix \(\mathcal P_\Omega(Y)\). The step size follows \(\eta_t=\eta_0/\sqrt{t}\), where \(\eta_0\) is selected separately for each method and each \((\kappa,\rho)\) setting.

\paragraph{Methods.}
We compare norm-matched Euclidean and intrinsic updates:
\begin{itemize}[leftmargin=0.3in]
    \item \textbf{Frobenius norm:} The intrinsic LMO recovers standard (normalized) Riemannian gradient descent, denoted by \textbf{RGD}, compared with (normalized) Euclidean gradient descent on the chosen representative, denoted by \textbf{EGD}.
    \item \textbf{Spectral norm:} \textbf{iMuon} denotes the intrinsic spectral norm LMO induced by the Riemannian metric, while factor-wise \textbf{Muon} denotes the corresponding Euclidean spectral norm LMO applied to the same representative as in (\ref{eqn:factorwise_muon}).
    \item \textbf{Nuclear norm:} \textbf{iMuon-Nu} denotes the intrinsic nuclear norm LMO. \textbf{NuMuon} is its Euclidean nuclear norm counterpart, following the nuclear norm constrained Muon terminology of~\citet{dolatabadi2026numuon}. In this section, we use only the nuclear norm constraint so that the comparison is norm matched to iMuon-Nu.
\end{itemize}
Each comparison is made within the same norm family, isolating the effect of using the intrinsic metric induced geometry rather than the ambient Euclidean geometry.

\paragraph{Dataset.}
We vary the target condition number over \(\kappa\in\{1,10\}\). To match the matrix-completion convention of reporting noise by its variance or standard deviation, we parameterize the observation noise by the relative entrywise scale
\[
    \rho=\sigma_\varepsilon/\mathrm{rms}_\Omega(X^*),
\]
where \(\varepsilon_{ij}\sim\mathcal N(0,\sigma_\varepsilon^2)\) is added only on observed entries and \(\mathrm{rms}_\Omega(X^*)\) is the root mean square magnitude of the sampled clean entries. We report the noiseless case \(\rho=0\), a mild noise setting \(\rho=0.05\), and two additional noise-stress settings \(\rho=0.10\) and \(\rho=0.50\). For each \((\kappa,\rho)\) setting and each seed, we generate a fresh rank-\(10\) target matrix \(X^*\), sample the observation set \(\Omega\), and add observation noise when applicable. Results are averaged over seeds \(\{0,1,2\}\).

\begin{figure}[t]
    \centering
    \begin{subfigure}{0.9\linewidth}
        \centering
        \includegraphics[width=0.9\linewidth,height=0.24\textheight,keepaspectratio]{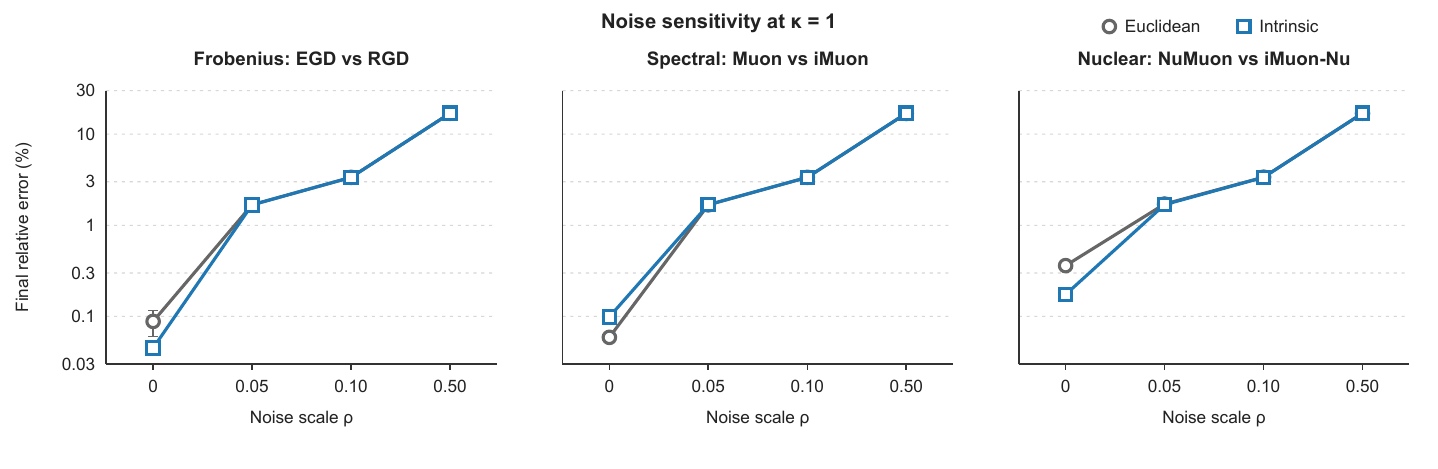}
        \caption{\(\kappa=1\)}
        \label{fig:app_fixed_rank_noise_kappa1}
    \end{subfigure}

    \medskip
    \begin{subfigure}{0.9\linewidth}
        \centering
        \includegraphics[width=0.9\linewidth,height=0.24\textheight,keepaspectratio]{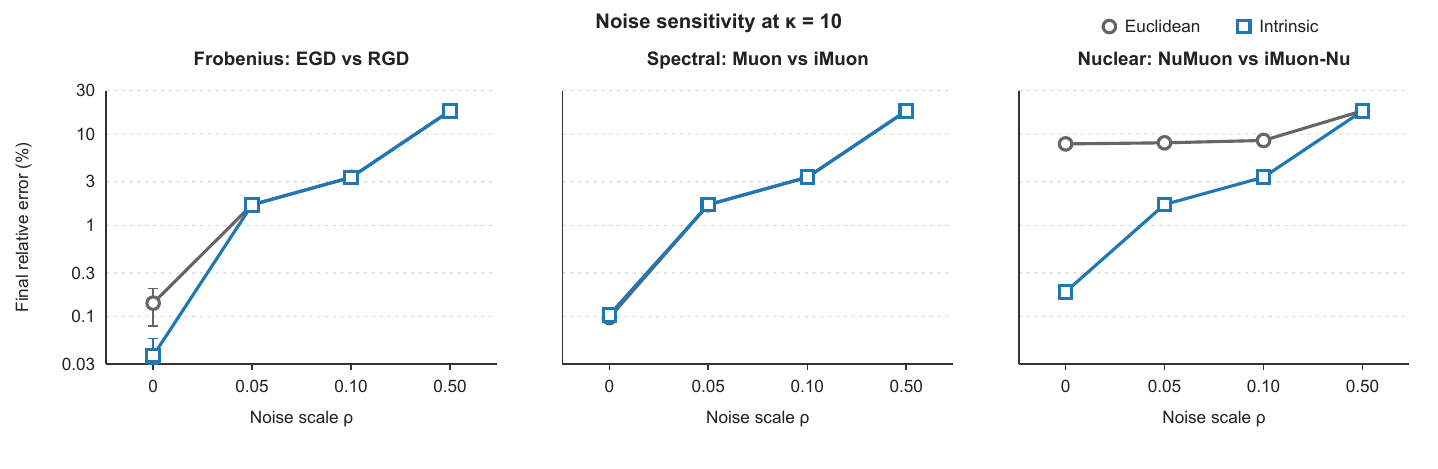}
        \caption{\(\kappa=10\)}
        \label{fig:app_fixed_rank_noise_kappa10}
    \end{subfigure}

    \caption{Noise sensitivity in large-scale synthetic fixed-rank matrix completion. Each row fixes the condition number and plots final relative recovery error as the relative observed-entry noise scale \(\rho\) varies. The three panels in each row compare the Frobenius, spectral, and nuclear norm pairs. The y-axis is logarithmic, and lower is better.}
    \label{fig:app_fixed_rank_noise_sensitivity}
\end{figure}

\paragraph{Results.}
\Cref{fig:app_fixed_rank_noise_sensitivity} plots final recovery error against the relative noise level for each condition number. For both condition numbers, the Frobenius and nuclear intrinsic updates improve over their Euclidean counterparts for low and moderate noise. The spectral pair is close at all noise levels, with factor-wise Muon slightly ahead in the noiseless case at \(\kappa=1\). At \(\rho=0.50\), observation noise dominates and the ordering becomes mixed.

\subsection{Fixed-rank CIFAR-100 classification heads}
\label{app:fixed_rank_cifar_heads}

The task is to assign each CIFAR-100 image to one of 100 semantic categories using frozen ResNet-18 features~\citep{krizhevsky2009learning,he2016deep}. A low-rank linear head maps these features to class scores while making the class weight vectors share a small set of latent visual directions, so the learned classifier is the fixed-rank parameter \(X=BA\).  We consider two regimes: a representative-sensitivity diagnostic and a balanced spectral comparison with Spectron~\citep{janson2026stabilizing}.

For both regimes, given frozen features \(h_i\in\RR^{512}\) and labels \(y_i\), we train a rank-\(r\) classifier \(X=BA\) and bias \(b\in\RR^{100}\) by minimizing
\[
\min_{B,A,b}\ \frac{1}{N}\sum_{i=1}^N
\ell_{\mathrm{CE}}\bigl(h_i^\top BA+b,y_i\bigr)
+\frac{\lambda}{2}\|BA\|_F^2+\frac{\lambda_b}{2}\|b\|_2^2 .
\]
We use \(\lambda=\lambda_b=10^{-4}\). The rank \(r\), initialization, and learning-rate grids are specified below.
\subsubsection{Representative-sensitivity diagnostic}
\label{app:fixed_rank_cifar_representative}

\paragraph{Setup.}
The representative-sensitivity diagnostic uses a rank-$10$ head,
\[
X=BA,\qquad B\in\RR^{512\times 10},\quad A\in\RR^{10\times 100}.
\]
To test representative sensitivity, we keep the initial product fixed and rescale only the factors,
\[
(B_0,A_0)=(\alpha\widetilde{B}_0,\alpha^{-1}\widetilde{A}_0),\qquad \alpha=10^3.
\]
Thus \(B_0A_0\) is unchanged up to numerical precision, while the two factors are imbalanced. This is a real supervised task with a deliberately chosen representative of the same initial classifier.

\paragraph{Methods.}
We compare the three norm-matched Euclidean/intrinsic pairs used above. The Euclidean methods use the grid
\[
\{10^{-5},\,10^{-4},\,10^{-3},\,10^{-2},\,10^{-1}\}.
\]
For the intrinsic methods, we use local grids,
\[
\text{RGD}:\{0.1,0.3,1.0\},\qquad
\text{iMuon}:\{0.03,0.1,0.3\},\qquad
\text{iMuon-Nu}:\{0.3,1.0,3.0\}.
\]
The learning rate is selected by validation accuracy for each method.

\paragraph{Dataset.}
We use the standard CIFAR-100 train/test split, reserve \(10\%\) of the training set for validation, extract frozen ResNet-18 features, and apply feature \(\ell_2\)-normalization. We train for \(50\) epochs with batch size \(32\) and average over seeds \(\{0,1,2\}\).
\paragraph{Results.}
\Cref{tab:app_fixed_rank_cifar_representative} reports validation-selected test accuracy. Under the same rescaled initial product, the intrinsic methods reach \(33.9\%\) to \(34.3\%\), while their Euclidean counterparts range from \(6.8\%\) to \(10.0\%\).

\begin{table}[t]
\centering
\caption{Representative-sensitivity diagnostic on a frozen-feature CIFAR-100 rank-$10$ classification head with \(\alpha=10^3\). We report mean test accuracy $\pm$ standard deviation over three seeds. Boldface marks the higher-accuracy method within each Euclidean/intrinsic pair.}
\label{tab:app_fixed_rank_cifar_representative}
{\small
\setlength{\tabcolsep}{3.5pt}
\begin{tabular}{lllcc}
\toprule
Norm & Euclidean & Intrinsic & Euclidean acc. & Intrinsic acc.\\
\midrule
Frobenius
& EGD ($10^{-5}$)
& RGD ($0.1$)
& 0.075 $\pm$ 0.007
& \textbf{0.342 $\pm$ 0.004} \\
Spectral
& Muon ($10^{-5}$)
& iMuon ($0.03$)
& 0.068 $\pm$ 0.005
& \textbf{0.339 $\pm$ 0.003}\\
Nuclear
& NuMuon ($10^{-4}$)
& iMuon-Nu ($0.3$)
& 0.100 $\pm$ 0.003
& \textbf{0.343 $\pm$ 0.004} \\
\bottomrule
\end{tabular}
}
\end{table}

\subsubsection{Balanced spectral comparison with Spectron}
\label{app:fixed_rank_cifar_spectron}

\paragraph{Setup.}
The balanced spectral comparison uses a rank-$40$ head:
\[
    X=BA,\qquad B\in\RR^{512\times 40},\quad A\in\RR^{40\times 100},
\]
with \(\alpha=1\). For each seed, all methods use the same train/validation split and the same rank-$40$ class-mean-SVD initialization of \(B_0A_0\). This regime compares spectral fixed-rank updates without artificial factor rescaling.

\paragraph{Methods.}
We compare Muon, Spectron~\citep{janson2026stabilizing}, and iMuon. Muon applies factor-wise Euclidean spectral LMO updates to \(B\) and \(A\)~\citep{jordan2024muon,bernstein2025old}. Spectron orthogonalizes the factor gradients and rescales the update by \(1/(\|B\|_2+\|A\|_2+1)\). We run Spectron without momentum to match the optimizer protocol used here. iMuon uses the coupled fixed-rank metric and the intrinsic spectral LMO from \Cref{cor:fixed_rank}. We train for $50$ epochs with batch size $2048$, use feature \(\ell_2\)-normalization, and select the learning rate by validation accuracy. The learning-rate grids are
\[
\begin{aligned}
\text{Muon}:&\ \{0.03,0.05,0.07,0.10,0.15\},\\
\text{Spectron}:&\ \{0.5,0.7,1.0,1.5,2.0\},\\
\text{iMuon}:&\ \{0.05,0.07,0.10,0.15,0.18,0.20,0.25\}.
\end{aligned}
\]
\paragraph{Dataset.}
We use the standard CIFAR-100 train/test split and reserve \(10\%\) of the training split for validation. Results are averaged over seeds \(\{0,1,2\}\).

\paragraph{Results.}
\Cref{tab:app_fixed_rank_cifar_spectron} reports validation-selected test accuracy. iMuon obtains the highest mean accuracy, \(53.93\%\), compared with \(53.82\%\) for factor-wise Muon and \(53.92\%\) for Spectron. The gaps are small, so we view this as a balanced comparison with fixed-rank spectral baselines rather than a large separation. \Cref{fig:app_fixed_rank_cifar_spectron_train_test} shows the corresponding training trajectories. The three methods reach similar training objectives, while iMuon remains competitive in test accuracy throughout training.

\begin{table}[htbp!]
\centering
\caption{Frozen-feature CIFAR-100 classification with a rank-$40$ fixed-rank head. We report validation-selected test accuracy over three seeds. Higher is better.}
\label{tab:app_fixed_rank_cifar_spectron}
{\small
\setlength{\tabcolsep}{4.5pt}
\begin{tabular}{lcc}
\toprule
Method & Selected LR & Test acc. (\%) \\
\midrule
Factor-wise Muon & \(0.05\) & \(53.82 \pm 0.02\) \\
Spectron & \(0.5\) & \(53.92 \pm 0.09\) \\
iMuon & \(0.18\) & \(\mathbf{53.93 \pm 0.09}\) \\
\bottomrule
\end{tabular}
}
\end{table}

\begin{figure}[t]
    \centering
    \includegraphics[scale=0.55]{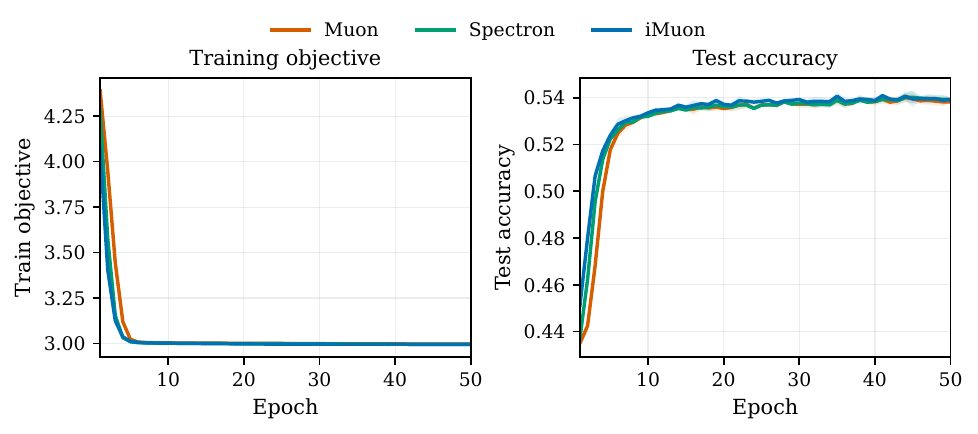}
    \caption{Training trajectories for the balanced fixed-rank CIFAR-100 rank-head comparison. Curves show means over three seeds with standard-deviation bands. The three methods reach similar training objectives, while iMuon remains competitive in test accuracy.}
    \label{fig:app_fixed_rank_cifar_spectron_train_test}
\end{figure}

\subsection{MovieLens-1M}
\label{app:fixed_rank_movielens}

\paragraph{Setup.}
We also test MovieLens-1M~\citep{harper2015movielens}, a standard real-data matrix completion benchmark. The data contain $6040$ users, $3952$ movies, and $1{,}000{,}209$ observed ratings. Following standard recommender preprocessing~\citep{koren2009matrix}, we subtract a global mean, user biases, and movie biases. We then fit the residual ratings with a rank-$5$ fixed-rank factorization \( X = BA\).
On the training entries $\Omega_{\mathrm{train}}$, the objective is
\[
    f(B,A)
    =
    \frac{1}{2|\Omega_{\mathrm{train}}|}
    \sum_{(u,i)\in\Omega_{\mathrm{train}}}
    \bigl((BA)_{ui}-R_{ui}\bigr)^2
    +
    \frac{\lambda}{2}\left(\|B\|_F^2+\|A\|_F^2\right),
\]
where \(R\) is the residual rating matrix after bias correction. All methods start from the rank-$5$ SVD of the zero filled observed residual matrix. We use $\lambda=1$.

\paragraph{Methods.}
We compare the same norm-matched Euclidean/intrinsic pairs as above and add Spectron~\citep{janson2026stabilizing} to the spectral group. Spectron is run without momentum, matching the optimizer protocol used in this section. Each method uses the decaying schedule \(\eta_t=\eta_0/\sqrt{t}\), with \(\eta_0\) selected by validation RMSE from the same grid
\[
\{0.3,0.5,1,1.5,2,3,5,7,10,15,20,30,50,75,100,150,200,300,500,700,1000\}.
\]
The values of \(\eta_0\) are not directly comparable across methods because the update directions have different normalizations.
\paragraph{Dataset.}
We use a random \(20\%\) test split and hold out \(10\%\) of the remaining observations for validation. Validation RMSE is used for model selection. Results are averaged over seeds $\{0,1,2\}$.

\paragraph{Results.}
\Cref{tab:app_fixed_rank_movielens} reports validation-selected test RMSE. RGD and iMuon-Nu improve over their Euclidean counterparts. In the spectral group, Spectron gives the lowest mean RMSE, while iMuon remains close and improves over factor-wise Muon. But these margins are small.

\begin{table}[t]
\centering
\caption{MovieLens-1M fixed-rank matrix completion with observed SVD initialization, rank $5$, factor $\ell_2$ regularization $\lambda=1$, and $\eta_t=\eta_0/\sqrt{t}$. We report mean test RMSE $\pm$ standard deviation over three seeds. Lower is better. Boldface marks the lowest RMSE within each norm group.}
\label{tab:app_fixed_rank_movielens}
{\small
\setlength{\tabcolsep}{4.5pt}
\begin{tabular}{llcc}
\toprule
Norm & Method & Selected \(\eta_0\) & Test RMSE \\
\midrule
Frobenius & EGD & $5$ & 0.8694 $\pm$ 0.0023 \\
           & RGD & $500$ & \textbf{0.8691 $\pm$ 0.0021} \\
\addlinespace
Spectral & Muon & $20$ & 0.8693 $\pm$ 0.0019 \\
         & Spectron & $700$ & \textbf{0.8691 $\pm$ 0.0022} \\
         & iMuon & $300$ & 0.8692 $\pm$ 0.0018 \\
\addlinespace
Nuclear & NuMuon & $7$ & 0.8729 $\pm$ 0.0026 \\
        & iMuon-Nu & $100$ & \textbf{0.8715 $\pm$ 0.0023} \\
\bottomrule
\end{tabular}
}
\end{table}

\section{SPD Prototype Learning for Covariance-Based Image Classification}
\label{app:spd_experiments}

\Cref{tab:app_spd_lmo_directions} gives the SPD update directions for the six methods: EGD, Muon, and NuMuon use the Euclidean column, while RGD, iMuon, and iMuon-Nu use the intrinsic column. Let \(G=\nabla_X f\) and \(H=X^{1/2}GX^{1/2}\). The table reports the LMO maximizer \(\xi^*\). The optimizer retracts along \(-\eta\xi^*\).

\begin{table}[t]
\centering
\caption{Euclidean and intrinsic LMO directions for the SPD prototype experiment. Here \((\mu_1,u_1)\) and \((\lambda_1,q_1)\) are eigenpairs of \(G\) and \(H\), respectively, associated with the largest absolute eigenvalue, and \(\operatorname{sign}(\cdot)\) denotes the matrix sign function.}
\label{tab:app_spd_lmo_directions}
{\small
\setlength{\tabcolsep}{4pt}
{\begin{tabular}{lll}
\toprule
Norm & Euclidean \(\xi^*\) & Intrinsic \(\xi^*\) \\
\midrule
Frobenius
& \(\displaystyle \,G/\|G\|_F\)
& \(\displaystyle \,XGX/\|H\|_F\) \\
Spectral
& \(\displaystyle \,\operatorname{sign}(G)\)
& \(\displaystyle \,X^{1/2}\operatorname{sign}(H)X^{1/2}\) \\
Nuclear
& \(\displaystyle \,\operatorname{sign}(\mu_1)u_1u_1^\top\)
& \(\displaystyle \,\operatorname{sign}(\lambda_1)X^{1/2}q_1q_1^\top X^{1/2}\) \\
\bottomrule
\end{tabular}}
}
\end{table}

\paragraph{Setup.}
We consider supervised prototype classification from SPD covariance descriptors. Each example is a covariance matrix \(C_i\in S_{++}^{32}\) with coarse label \(y_i\in\{1,\ldots,K\}\), where \(K=20\) for CIFAR-100 coarse labels. We learn one SPD prototype \(X_c\in S_{++}^{32}\) per class, so the optimization variable is \(X=(X_1,\ldots,X_K)\in(S_{++}^{32})^K\).

Given prototypes \(X_1,\ldots,X_K\), we form logits from negative squared AI distance
\[
 z_{ic}=-\beta\,d_{\mathrm{AI}}(C_i,X_c)^2, \qquad
 d_{\mathrm{AI}}(C,X)^2=\left\|\log\!\left(X^{-1/2} C X^{-1/2}\right)\right\|_F^2 .
\]
We minimize the regularized cross-entropy objective
\[
f(X)
=
\frac{1}{N}\sum_{i=1}^N\ell_{\mathrm{CE}}(z_i,y_i)
+
\lambda\sum_{c=1}^K d_{\mathrm{AI}}(X_c,\bar X_c)^2 ,
\]
where \(\bar{X}_c\) is the log-Euclidean class-mean initialization for class \(c\). We use the AI metric from the main text, initialize prototypes by log-Euclidean class means, set the logit scale to \(\beta=8.0\), and use \(\lambda=10^{-3}\). We train for \(20\) epochs with training batch size \(64\) and evaluation batch size \(512\), and report results over seeds \(\{0,1,2\}\).

\paragraph{Methods.}
We use the same norm-matched comparison structure: RGD vs.\ EGD for the Frobenius norm, iMuon vs.\ Muon for the spectral norm, and iMuon-Nu vs.\ NuMuon for the nuclear norm. All methods are tuned over the shared learning-rate grid
\[
\{10^{-3},\,3\!\cdot\!10^{-3},\,10^{-2},\,3\!\cdot\!10^{-2},\,10^{-1}\}.
\]
The selected values are \(0.01\) for RGD and iMuon, and \(0.001\) for iMuon-Nu, EGD, Muon, and NuMuon. Model selection uses validation accuracy.

\paragraph{Dataset.}
We use CIFAR-100 coarse labels (\(20\) classes)~\citep{krizhevsky2009learning} together with frozen ResNet-18 visual features~\citep{he2016deep}. We extract the \texttt{layer3} feature map for each image, keep the backbone fixed, and project each spatial token from the backbone channel dimension to a $32$-dimensional space using a fixed random projector with seed $17$. We then build a shrinkage regularized covariance descriptor \(C_i\in S_{++}^{32}\)~\citep{ledoit2004well,tuzel2008pedestrian,huang2017riemannian}
for each image using covariance shrinkage $0.1$ and diagonal stabilization $\varepsilon=10^{-4}$. We use the standard CIFAR-100 train/test split and reserve $10\%$ of the training set as a validation split for learning-rate selection.

\paragraph{Results.}
\Cref{tab:app_spd_cifar_classification} reports validation-selected test accuracy for the three norm-matched Euclidean/intrinsic pairs. Each intrinsic method improves over its Euclidean counterpart. The Frobenius gap is positive (RGD \(0.548\pm0.002\) vs.\ EGD \(0.523\pm0.026\)), the spectral gap is positive (iMuon \(0.475\pm0.009\) vs.\ Muon \(0.407\pm0.016\)), and the largest endpoint gap is for the nuclear pair (iMuon-Nu \(0.546\pm0.000\) vs.\ NuMuon \(0.422\pm0.032\), where the reported zero is a standard deviation below \(0.0005\)).

\begin{table}[t]
\centering
\caption{SPD covariance prototype classification on CIFAR-100 coarse labels. We report mean test accuracy \(\pm\) standard deviation over three seeds. Boldface marks the higher accuracy method within each pair.}
\label{tab:app_spd_cifar_classification}
{\small
\begin{tabular}{lcc}
\toprule
Method & Selected LR & Mean test acc. \\
\midrule
RGD & 0.01 & \textbf{0.548 $\pm$ 0.002} \\
EGD & 0.001 & 0.523 $\pm$ 0.026 \\
iMuon & 0.01 & \textbf{0.475 $\pm$ 0.009} \\
Muon & 0.001 & 0.407 $\pm$ 0.016 \\
iMuon-Nu & 0.001 & \textbf{0.546 $\pm$ 0.000} \\
NuMuon & 0.001 & 0.422 $\pm$ 0.032 \\
\bottomrule
\end{tabular}
}
\end{table}

\Cref{fig:app_spd_cifar_objective_diagnostic} reports a companion convergence plot for the classification term. It uses the same covariance features and norm-matched pairs, removes the prototype anchoring term, and uses a smaller CPU split.
In the Frobenius and spectral pairs, the Euclidean method reaches a lower training cross-entropy, but its test accuracy declines after the first few epochs while that of the intrinsic method continues to rise. In the nuclear pair, iMuon-Nu is ahead on both. Endpoint comparisons use \Cref{tab:app_spd_cifar_classification}.

\begin{figure}[t]
    \centering
    \includegraphics[width=\linewidth]{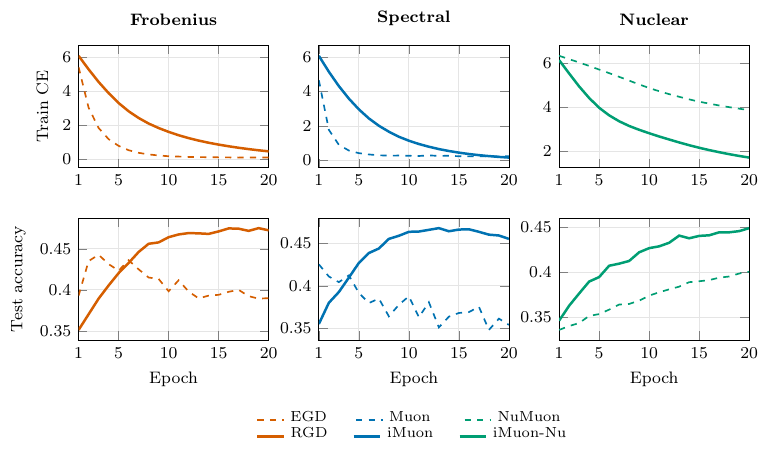}
   \caption{SPD convergence plot on frozen covariance features. The top row reports training cross-entropy and the bottom row reports test accuracy. Columns compare the Frobenius, spectral, and nuclear norm pairs. This variant omits the prototype anchoring term, so the accuracy values are not directly comparable with \Cref{tab:app_spd_cifar_classification}.}
    \label{fig:app_spd_cifar_objective_diagnostic}
\end{figure}

\section{Grassmann Fr\'echet Subspace Learning for Video Face Identification
}
\label{app:grassmann_experiments}

\paragraph{Setup.}
We consider video-based face identification on YouTube Faces~\citep{wolf2011face}. Each video clip is encoded as a \(k\)-dimensional subspace \(S_i\in \textnormal{Gr}(64,k)\), and each identity class has a learned Grassmann prototype \(X_c\in\textnormal{Gr}(64,k)\). For each class \(c\), we optimize the Fr\'echet prototype objective
\[
    \min_{X_c\in\textnormal{Gr}(64,k)}
    \frac{1}{|\mathcal{I}_c|}
    \sum_{i\in\mathcal{I}_c}
    d_{\textnormal{Gr}}(X_c,S_i)^2,
\]
where \(\mathcal{I}_c\) indexes the training clips of class \(c\), and \(d_{\textnormal{Gr}}\) is the geodesic distance induced by principal angles. A test clip is classified by the nearest learned prototype under \(d_{\textnormal{Gr}}\).

We use the canonical Grassmann metric. With a basis \(X\in\RR^{64\times k}\) for a subspace, tangent vectors lie in the horizontal space \(\{\xi:X^\top\xi=0\}\), where \(G_X=I\). Hence \(\varphi(G_X^{1/2}\xi)\le \tau\) reduces to the Euclidean constraint \(\varphi(\xi)\le \tau\). This boundary case checks that the intrinsic LMO recovers the projected Euclidean update when the metric scaling is identity.

\paragraph{Methods.}
We optimize the three norm-matched pairs for \(100\) iterations, using the squared-geodesic loss with scale \(c_{\textnormal{Gr}}=0.25\) and \(\varepsilon=10^{-6}\). We also report GDA~\citep{hamm2008grassmann} and PML~\citep{huang2015projection} as reference methods. GDA learns a discriminative Grassmann projection, while PML learns a projection metric for subspaces.

\paragraph{Dataset.}
We use YouTube Faces clips~\citep{wolf2011face} represented as \(k\)-dimensional subspaces of a \(64\)-dimensional feature space, with \(k\in\{3,5,8,10\}\). We use the standard Grassmann geometry~\citep{edelman1998geometry,absil2008optimization} and evaluate over \(5\) random train/test splits with training fraction \(0.7\).

\paragraph{Results.}
\Cref{tab:app_grassmann_orthonormal} reports classification accuracy on YouTube Faces. Across all \(k\), each intrinsic method matches its Euclidean counterpart within the displayed precision. This is expected: with \(G_X=I\), the intrinsic LMO and the Euclidean LMO applied to the horizontal gradient \((I-XX^\top)\nabla_Xf\) are the same problem, so the paired columns are the same algorithm. The informative comparison in the table is therefore across norms, and against GDA and PML.

\begin{table}[t]
\centering
\caption{Grassmann Fr\'echet prototype learning on YouTube Faces. We report mean classification accuracy \(\pm\) standard deviation over \(5\) splits. Boldface marks the highest accuracy in each row. Since \(G_X=I\) on the horizontal space, each intrinsic method coincides with its Euclidean projected counterpart.}

\label{tab:app_grassmann_orthonormal}
{\footnotesize
\setlength{\tabcolsep}{4pt}
\resizebox{\linewidth}{!}{\begin{tabular}{lcccccccc}
\toprule
$k$ & RGD & iMuon & iMuon-Nu & EGD & Muon & NuMuon & GDA & PML \\
\midrule
3 & 0.136 $\pm$ 0.017 & 0.147 $\pm$ 0.012 & 0.138 $\pm$ 0.018 & 0.136 $\pm$ 0.017 & 0.147 $\pm$ 0.012 & 0.138 $\pm$ 0.018 & 0.087 $\pm$ 0.012 & \textbf{0.160 $\pm$ 0.046} \\
5 & \textbf{0.219 $\pm$ 0.037} & 0.206 $\pm$ 0.050 & 0.211 $\pm$ 0.038 & \textbf{0.219 $\pm$ 0.037} & 0.206 $\pm$ 0.050 & 0.211 $\pm$ 0.038 & 0.094 $\pm$ 0.028 & 0.179 $\pm$ 0.031 \\
8 & 0.232 $\pm$ 0.016 & \textbf{0.238 $\pm$ 0.042} & \textbf{0.238 $\pm$ 0.034} & 0.232 $\pm$ 0.016 & \textbf{0.238 $\pm$ 0.042} & \textbf{0.238 $\pm$ 0.034} & 0.166 $\pm$ 0.022 & 0.194 $\pm$ 0.031 \\
10 & 0.262 $\pm$ 0.016 & 0.264 $\pm$ 0.038 & \textbf{0.266 $\pm$ 0.034} & 0.262 $\pm$ 0.016 & 0.264 $\pm$ 0.038 & \textbf{0.266 $\pm$ 0.034} & 0.164 $\pm$ 0.032 & 0.213 $\pm$ 0.011 \\
\bottomrule
\end{tabular}
}
}
\end{table}

\section{Stiefel Sub-Center Prototype Learning for Frozen-Feature Image Classification}
\label{app:stiefel_cifar_subcenter}

\paragraph{Setup.}
We consider supervised classification with class prototypes constrained to the Stiefel manifold~\citep{edelman1998geometry,absil2008optimization}. Given frozen image features \(h_i\in\RR^{512}\) and labels \(y_i\in\{1,\ldots,C\}\), with \(C=100\) for CIFAR-100, we learn \(q=4\) prototypes per class, giving \(Q=Cq\) prototypes in total. The prototypes are the columns of
\[
    X = [x_{1,1},\ldots,x_{1,q},\ldots,x_{C,1},\ldots,x_{C,q}]
    \in \textnormal{St}(512,Q),
    \qquad X^\top X=I_Q.
\]
For class \(c\), the classifier score is the best matching prototype score
\[
    s_c(h_i,X)=\max_{1\le j\le q} h_i^\top x_{c,j}.
\]
We use an additive margin softmax objective~\citep{wang2018additive}. For the target class \(y_i\), the score is shifted to \(s_{y_i}(h_i,X)-m\) before multiplying all class scores by a logit scale \(\gamma\). The optimization problem is
\[
    \min_{X\in\textnormal{St}(512,Q)}
    \frac{1}{N}\sum_{i=1}^N
    \ell_{\mathrm{CE}}\!\left(
    \gamma\,\bigl(s_1(h_i,X),\ldots,s_{y_i}(h_i,X)-m,\ldots,s_C(h_i,X)\bigr),
    y_i
    \right).
\]
We use margin \(m=0.5\) and logit scale \(\gamma=64\). The Stiefel constraint enforces global orthogonality among all class prototypes. This gives a real supervised task for comparing intrinsic and extrinsic spectral updates on \(\textnormal{St}(512,Q)\), \(Q=400\).

\paragraph{Methods.}
We compare RGD, iMuon, and SPEL. RGD is normalized Riemannian gradient descent on \(\textnormal{St}(512,Q)\). iMuon is the spectral norm instance of our intrinsic LMO, using the closed form tangent decomposition from the main text. SPEL is the Stiefel specialization of manifold constrained steepest descent~\citep{yang2026mcsd}. It forms a spectral steepest descent direction from the Riemannian gradient using extrinsic normalization, then retracts to the manifold. The updates are discussed in \Cref{app:compare_stiefel}. All methods use polar retraction. Learning rates are selected by validation accuracy from \(\{0.02,\,0.03,\,0.05,\,0.1\}\).

\paragraph{Dataset.}
We use CIFAR-100~\citep{krizhevsky2009learning} with frozen ResNet-18 features~\citep{he2016deep}. Features are \(\ell_2\)-normalized before training the Stiefel classifier. Each seed uses \(50\) training examples, \(10\) validation examples, and \(30\) test examples per class. We train for \(30\) epochs with batch size \(64\), evaluation batch size \(4096\), and select the learning rate by validation accuracy.

\paragraph{Results.}
\Cref{tab:app_stiefel_cifar_subcenter} reports test accuracy and test cross entropy (CE), the average negative log likelihood under the additive margin softmax classifier. iMuon has the best mean accuracy and the lowest CE. It improves over RGD by \(1.1\) percentage points and is slightly above SPEL. \Cref{fig:app_stiefel_cifar_train_test_curves} shows that iMuon and SPEL reduce the training objective faster than RGD, with iMuon slightly higher in mean test accuracy near the end of training.

\begin{table}[t]
\centering
\caption{Stiefel sub-center prototype classification on frozen CIFAR-100 features. We train an additive-margin classifier with \(q=4\) prototypes per class on \(\textnormal{St}(512,400)\), select the learning rate by validation accuracy for each seed, and report mean test accuracy and test cross entropy. CE denotes average negative log likelihood. Higher accuracy and lower CE are better.}
\label{tab:app_stiefel_cifar_subcenter}
{\small
\setlength{\tabcolsep}{7pt}
\begin{tabular}{lccc}
\toprule
Method & Selected LRs & Test acc. & Test CE \\
\midrule
RGD & \(0.1,0.1,0.1\) & \(0.542 \pm 0.013\) & \(1.853 \pm 0.026\) \\
iMuon & \(0.03,0.02,0.02\) & \(\mathbf{0.553 \pm 0.006}\) & \(\mathbf{1.839 \pm 0.024}\) \\
SPEL & \(0.03,0.03,0.05\) & \(0.545 \pm 0.012\) & \(1.869 \pm 0.066\) \\
\bottomrule
\end{tabular}
}
\end{table}

\begin{figure}[t]
    \centering   \includegraphics[scale = 0.45]{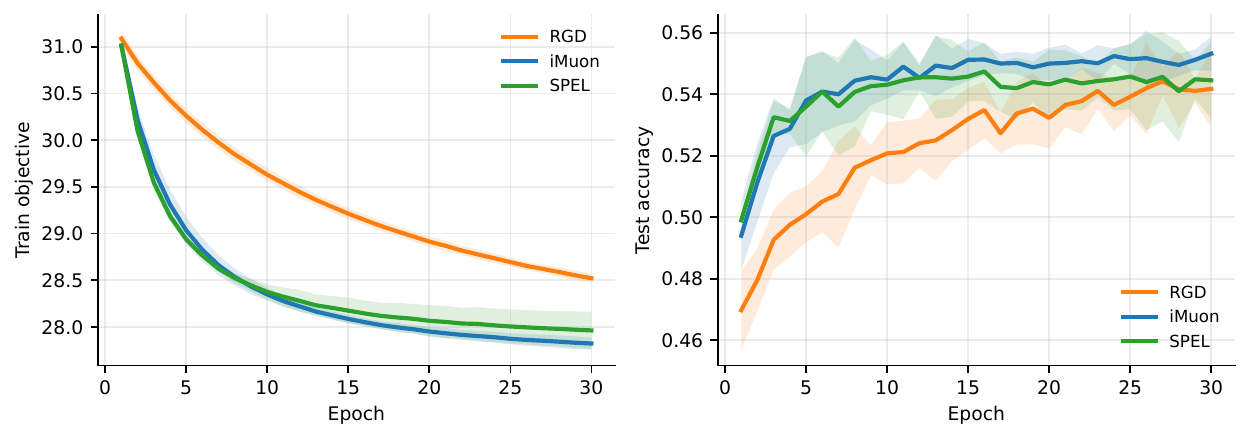}
    \caption{Stiefel sub-center prototype classification trajectories. The left panel reports the training objective and the right panel reports test accuracy over epochs. Curves show means over three seeds with standard deviation bands.}
    \label{fig:app_stiefel_cifar_train_test_curves}
\end{figure}

\section{Additional Manifolds and Metrics}\label{app:other_manifolds}

The LMO framework of \Cref{sec:principle} is stated for general Riemannian matrix manifolds.
In the main text we instantiate it on four manifolds: two quotient manifolds ($\M_r$, $\text{Gr}(m,r)$), the SPD manifold $S_{++}^n$ (with $\text{GL}(n)$ acting by isometries), and the Stiefel manifold $\text{St}(m,r)$.
Here we discuss broader applicability and several instructive boundary cases:
(i)~general conditions for closed-form solutions,
(ii)~the fixed-rank manifold from the \emph{embedded} viewpoint, why the joint spectral-norm constraint on $\dot{X}$ does not decouple, and how the quotient formulation resolves this,
(iii)~the Stiefel manifold from the \emph{embedded submanifold} viewpoint, why the naive approach fails and how the block decomposition of \Cref{subsec:other_manifolds} resolves it,
(iv)~the spectral sphere $\{X:\|X\|_2=R\}$, another non-SV-invariant example, recently used in the SSO optimizer~\citep{xie2026sso},
and (v)~alternative SPD metrics (Log-Euclidean and Bures--Wasserstein).

\subsection{General applicability}

The LMO~\eqref{eqn:main_lmo} is well-defined on \emph{any} Riemannian manifold $\M$ with a metric operator $G_x$: it is simply a convex optimization over the tangent space.
The two structural properties that enable closed-form solutions and symmetry preservation are:
\begin{enumerate}
\item \textbf{Left-right structure} (\Cref{prop:symmetry}): the scaled map $\rho_h = G_{\tilde{x}}^{1/2} \circ D_xh \circ G_x^{-1/2}$ acts as $Z \mapsto U_h Z V_h^\top$ for orthogonal $U_h, V_h$, which holds when the factors of a left-right metric transform as in \Cref{prop:symmetry}.
This ensures $\varphi(\rho_h Z) = \varphi(Z)$ for unitarily invariant~$\varphi$.
\item \textbf{Singular value invariance {or $\varphi$-contractivity}} (\Cref{prop:tangent_compat}): the scaled tangent space $\mathcal{Z}_x$ is SV-invariant (\Cref{def:sv_invariant}) {or $\varphi$-contractive (\Cref{def:contractive})}, so
the tangent-space constraint is redundant and the LMO reduces to a standard Euclidean problem with closed-form solutions.
\end{enumerate}

\noindent For \emph{embedded submanifolds} (no quotient symmetry), the symmetry condition (1) is not required.
Condition (2) alone suffices for closed-form solutions.
If $\mathcal{Z}_x$ satisfies neither condition, the LMO remains a convex problem, but its solution is in general no longer given by the closed forms of \Cref{prop:tangent_compat}, such as $\Ortho(H)$, and may require an iterative solver (see, e.g., \Cref{app:spectral_sphere}).

\subsection{Fixed-rank manifold: embedded viewpoint}\label{app:lowrank_embedded}

In \Cref{subsec:lowrank} we work on the quotient $\M_r=(\RR^{m\times r}_*\times\RR^{r\times n}_*)/\text{GL}(r)$ with the coupled metric, which gives $G_x^{1/2}\neq I$ and decouples the LMO into independent $\dot{B}$ and $\dot{A}$ subproblems.
Here we discuss the alternative \emph{embedded} viewpoint, working directly on $\{X\in\RR^{m\times n}:\text{rank}(X)=r\}$ with $G_x=I$, and explain why it does not yield the same clean decoupled solutions.

\noindent
\paragraph{The embedded LMO.}
On the embedded rank-$r$ manifold, the tangent space at $X= BA$ is $T_{X}\M_r=\{\dot{B}A+B\dot{A}:\dot{B}\in\RR^{m\times r},\;\dot{A}\in\RR^{r\times n}\}$.
With $G_x=I$, the LMO becomes
\[
\dot{X}^* = \argmax_{\dot{X}\in T_{X}\M_r,\;\varphi(\dot{X})\le\tau}\;\inner{\dot{X}}{\nabla_X f}.
\]
For $\varphi=\|\cdot\|_2$, the unconstrained maximizer over all of $\RR^{m\times n}$ is $\Ortho(\nabla_X f)$, which has rank $\min(m,n)$ and generally does not lie in $T_{X}\M_r$ (a subspace of rank $\le 2r$).
Riemannion~\citep{bogachev2026riemannion} addresses this by first computing a rank-$2r$ orthogonalization $\Ortho_r(M_t)$ of the momentum (which lies in $T_X\M_r$), and then projecting the result back onto $T_X\M_r$.
This tangent-space projection changes the singular values, so the spectral-norm constraint is only approximately satisfied.

\noindent
\paragraph{Failure of SV-invariance.}
The scaled tangent space $\mathcal{Z}_{X}=T_{X}\M_r=\{\dot{B}A+B\dot{A}\}\subset\RR^{m\times n}$ is \emph{not} singular value invariant: given $\dot{X}=U\Sigma V^\top\in T_{X}\M_r$, replacing $\Sigma$ with a different non-negative diagonal $D$ yields $UDV^\top$ which generally does not decompose as $\dot{B}A+B\dot{A}$ for any $(\dot{B},\dot{A})$.
The tangent space is not spectral-norm contractive either (simple counter-examples exist).
Thus \Cref{prop:tangent_compat} does not apply, and the constrained LMO requires solving $\max_{\dot{X}\in T_{X}\M_r,\,\|\dot{X}\|_2\le\tau}\inner{\dot{X}}{\nabla_X f}$, a constrained problem for which we are not aware of a closed form.

\noindent
\paragraph{Why the quotient formulation resolves this.}
On the quotient, each factor block lives in a full ambient space ($\RR^{m\times r}$ for $\dot{B}$, $\RR^{r\times n}$ for $\dot{A}$), which is trivially SV-invariant.
The product-norm decoupling $\max(\varphi(Z_{B}),\varphi(Z_{A}))\le\tau$ then gives two independent LMOs, each with a clean closed form.
The price is twofold: (a)~the product norm is a relaxation of the joint norm $\varphi(\dot{X})\le\tau$ (specifically, $\|\dot{X}\|_2\le 2\tau$ rather than $\tau$), and (b)~the non-trivial metric $G_x^{1/2}\neq I$ must be inverted (the Gram-root inversion).
But this price buys symmetry invariance (\Cref{prop:symmetry}) and a condition-number-free bound $\|\xi^*\|_x^2\le 2r\tau^2$ (\Cref{app:proof_lowrank_rate}).

\noindent
\paragraph{Frobenius norm.}
For $\varphi=\|\cdot\|_F$, the embedded LMO does have a closed form: $\dot{X}^*\propto P_{T_{X}}(\nabla_X f)$, the tangent-space projection of the gradient, which is the Riemannian gradient (up to normalization).

\subsection{Stiefel manifold: embedded viewpoint}\label{app:stiefel}

In \Cref{subsec:other_manifolds}, we derived closed-form LMO solutions on $\text{St}(m,r)$ using the decomposition $T_X\text{St}=X\cdot\text{Skew}(r)\oplus X_\perp\cdot\RR^{(m-r)\times r}$ and the product norm. Here, we discuss the alternative \emph{embedded} viewpoint, which treats $T_X\text{St}$ as a single subspace of $\RR^{m\times r}$, and explain why \Cref{prop:tangent_compat} does not apply to it.

\paragraph{The embedded LMO.}
Since $G_X=I$, the LMO~\eqref{eqn:main_lmo} on $\text{St}(m,r)$ is
\[
\xi^* = \argmax_{\xi\in T_X\text{St},\;\varphi(\xi)\le\tau}\;\inner{\xi}{\nabla_X f}.
\]
We write $S=\Skew(X^\top\nabla_X f)$ and $N=X_\perp^\top\nabla_X f$ for the gradient blocks paired with $A$ and $B$ in $\xi=XA+X_\perp B$, so that $\inner{\xi}{\nabla_X f}=\inner{A}{S}+\inner{B}{N}$. For $\varphi=\|\cdot\|_F$, the solution is the normalized projected gradient $\xi^*=\tau\,(XS+X_\perp N)/(\|S\|_F^2+\|N\|_F^2)^{1/2}$. For $\varphi=\|\cdot\|_2$, the unconstrained maximizer $\tau\,\Ortho(\nabla_X f)$ generally does not lie in $T_X\text{St}$, and projecting it onto $T_X\text{St}$ changes its singular values, so the spectral-norm constraint may be violated.

\paragraph{Failure of the conditions of \Cref{prop:tangent_compat}.}
The tangent space $T_X\text{St}$ is not SV-invariant (\Cref{def:sv_invariant}), since replacing the singular values of $\xi\in T_X\text{St}$ generally violates $X^\top\xi+\xi^\top X=0$. It is not spectral-norm contractive either. For $X=[e_1,e_2]\in\text{St}(3,2)$ and $W=[e_3,e_1]$, we have $\|W\|_2=1$ but $\|P_{T_X\text{St}}(W)\|_2=\sqrt{5}/2>1$. Hence, for $\varphi=\|\cdot\|_2$, the LMO above does not reduce to the unconstrained Euclidean LMO. In block form, the joint constraint $\|[A;B]\|_2\le\tau$ couples $A$ and $B$, and the unconstrained maximizer $\Ortho([S;N])$ generally has a top block that is not skew-symmetric.

\paragraph{Resolution via the product norm.}
The product norm $\max\{\varphi(A),\varphi(B)\}\le\tau$ of \Cref{subsec:other_manifolds} decouples the LMO into a skew and a normal subproblem. The normal block $\RR^{(m-r)\times r}$ is trivially SV-invariant, and the skew block $\text{Skew}(r)$ is SV-invariant because the singular values of a skew-symmetric matrix come in equal pairs (\Cref{app:sv_invariance}). Hence, \Cref{prop:tangent_compat} applies to each block and yields the closed forms of \Cref{lemma:stiefel_spectral}. For $\varphi=\|\cdot\|_F$, the product-norm solution normalizes the two blocks separately and is therefore parallel to the normalized projected gradient above only when $\|S\|_F=\|N\|_F$ (in which case it is $\sqrt{2}$ times larger) (\Cref{app:compare_stiefel}).

\paragraph{Conditioning of the product norm when the skew block vanishes.}
Since the product norm constrains the two blocks independently, the spectral-norm skew update satisfies $\|A^*\|_2=\tau$ whenever $S\neq 0$, irrespective of the magnitude of $\|S\|_F$. As $S\to 0$, the polar factor is computed from a matrix whose entries lie at the level of rounding error, its singular vectors become arbitrary, and the update retains a rotation of full magnitude within $\text{col}(X)$ that carries no gradient information. This is the Grassmann limit. Objectives invariant under $X\mapsto XO$ for $O\in O(r)$, such as $f(X)=-\frac{1}{2}\tr(X^\top CX)$ with $C$ symmetric, belong to this case, since $X^\top\nabla_X f$ is then symmetric, so $S$ vanishes in exact arithmetic and lies at the level of rounding error in floating point. Such objectives depend only on $\text{col}(X)$, so the rotation leaves the function value unchanged. When the $O(r)$ symmetry is broken, however, the rotation contributes no descent and the step is used inefficiently. This behavior is a consequence of orthogonalizing the two blocks separately. A simple safeguard is to set $A^*=0$ once $\|S\|_F$ falls below a prescribed tolerance relative to $\|N\|_F$.

\paragraph{Relation to the joint-norm formulation.}
For $\varphi=\|\cdot\|_2$, the joint-norm LMO above has been solved via dual ascent~\citep{bernstein2025manifolds} and other iterative procedures~\citep{cesista2025stiefel,su2025stiefel}, and in closed form in a concurrent work~\citep{solonko2026skewon}. We discuss these approaches in \Cref{app:compare_stiefel}.

\subsection{Spectral sphere}\label{app:spectral_sphere}

The spectral sphere $\mathcal{S}_R:=\{X\in\RR^{m\times n}:\|X\|_2=R\}$ constrains only the \emph{largest} singular value of $X$ to equal $R$. The remaining singular values are free to take any values in $[0,R]$.
When the top singular value is simple (a generic condition), the manifold is smooth with tangent space
\[
T_X\mathcal{S}_R = \{\xi\in\RR^{m\times n} : \langle u_1 v_1^\top, \xi\rangle = 0\},
\]
where $(u_1,v_1)$ are the leading left and right singular vectors of $X$, and the gradient of the spectral norm is $\nabla_X\|X\|_2 = u_1 v_1^\top$.

\noindent
\paragraph{Failure of SV-invariance.}
The tangent space is the hyperplane $\{\xi:\langle u_1 v_1^\top, \xi\rangle=0\}$, defined by an inner product with a \emph{data-dependent} rank-$1$ matrix.
This is not SV-invariant: given $\xi=U\Sigma V^\top\in T_X\mathcal{S}_R$, replacing $\Sigma$ with a different non-negative diagonal $D$ yields $UDV^\top$ whose inner product $\langle u_1 v_1^\top, UDV^\top\rangle=\sum_i d_i(u_1^\top a_i)(v_1^\top b_i)$ (with $a_i,b_i$ the columns of $U,V$) generally differs from zero.
The tangent space is not spectral-norm contractive either.
Thus \Cref{prop:tangent_compat} does not apply, and the spectral-norm LMO $\max\{\langle \nabla_X f,\xi\rangle: \xi\in T_X\mathcal{S}_R,\,\|\xi\|_2\le\tau\}$  is not covered by the closed forms of \Cref{prop:tangent_compat}.

\noindent
\paragraph{Connection to SSO.}
Although SV-invariance fails, the LMO on $\mathcal{S}_R$ remains a well-posed convex problem, namely the maximization of a linear functional over the intersection of a spectral-norm ball and a hyperplane.
This is precisely the problem solved by the Spectral Sphere Optimizer (SSO)~\citep{xie2026sso}, which performs steepest descent on $\mathcal{S}_R$ under the spectral norm for $\mu$P-aligned~\citep{yang2026spectral} LLM training.
In other words, SSO is an instance of our framework~\eqref{eqn:main_lmo} with $G_X=I$ and $\varphi=\|\cdot\|_2$ on the manifold $\mathcal{S}_R$. It solves exactly the same intrinsic LMO, but without the benefit of a closed-form solution.
The single hyperplane constraint requires one Lagrange multiplier $\lambda$, and SSO solves for it via bisection on the monotone function
\[
h(\lambda) := \langle u_1 v_1^\top,\, \Ortho(\nabla_X f + \lambda\, u_1 v_1^\top)\rangle = 0,
\qquad
\xi^*(\lambda) = \Ortho(\nabla_X f + \lambda\, u_1 v_1^\top).
\]
Each bisection step requires one $\Ortho$ evaluation and approximately 5--7 bisection steps are needed in practice.
If the tangent constraint $\langle u_1 v_1^\top, \xi\rangle = 0$ were dropped (i.e., no weight-norm preservation), the solution would reduce to $\xi^* = \Ortho(\nabla_X f)$, standard Euclidean Muon.
Thus the bisection cost is the computational price of the non-SV-invariant tangent constraint, which arises from simultaneously constraining both the update norm and the weight norm.

\noindent
\paragraph{Takeaway.}
The spectral sphere illustrates that our framework applies beyond manifolds with SV-invariant tangent spaces: the LMO~\eqref{eqn:main_lmo} is always well-posed, and SV-invariance determines only whether the solution is available in closed form or requires a solver.
On the four manifolds of \Cref{sec:solutions}, SV-invariance holds (blockwise on $\text{St}(m,r)$ under the product norm) and the LMO reduces to $\Ortho$, $\sign$, or rank-$1$ truncation.
On the spectral sphere, a single linear constraint breaks SV-invariance, but a one-dimensional bisection restores tractability at the cost of a few additional $\Ortho$ evaluations per step.
The Frobenius instance ($\varphi=\|\cdot\|_F$) remains closed-form in all cases, as the LMO reduces to projecting the gradient onto the tangent space and normalizing.

\subsection{Alternative SPD metrics: Log-Euclidean and Bures--Wasserstein}\label{app:spd_metrics}

In \Cref{subsec:other_manifolds} we use the affine-invariant (AI) metric on $S_{++}^n$.
Here we analyze how two other common SPD metrics interact with the two structural properties of the framework: left-right structure (\Cref{prop:symmetry}, enabling symmetry invariance) and SV-invariance of the scaled tangent space (\Cref{prop:tangent_compat}, enabling closed-form reduction).
All three metrics share the same tangent space $T_X S_{++}^n = \text{Sym}(n)$, which is SV-invariant (\Cref{app:sv_invariance}).
Thus \Cref{prop:tangent_compat} applies to all three, and the LMO always reduces to an unconstrained Euclidean problem with the standard closed forms ($\sign(H)$, $H/\|H\|_F$, or rank-$1$ truncation).
The metrics differ only in how the two propositions' \emph{other} aspects, symmetry group size and the explicitness of the scale/invert steps, play out.

\noindent
\paragraph{Affine-invariant (AI) metric (recap).}
$G_X^{AI}(\xi)=X^{-1}\xi X^{-1}$, so $G_X^{1/2}(\xi)=X^{-1/2}\xi X^{-1/2}$ (left-right form with $P=Q=X^{-1}$).
\begin{itemize}
\item \emph{\Cref{prop:symmetry} (left-right structure):} holds with $\text{GL}(n)$-invariance.
\item \emph{\Cref{prop:tangent_compat} (SV-invariance):} $\mathcal{Z}_X=\text{Sym}(n)$, which is SV-invariant.
\item \emph{Scale/invert:} coordinate-free sandwich formulas $H=X^{1/2}(\nabla_X f)X^{1/2}$ and $\xi^*=X^{1/2}Z^*X^{1/2}$.
\end{itemize}
This is the ideal case: both propositions apply with the maximal symmetry group, and all three pipeline steps (scale, solve, invert) have explicit matrix-coordinate formulas.

\noindent
\paragraph{Log-Euclidean metric.}
The Log-Euclidean metric~\citep{arsigny2007geometric} uses $\Phi:X\mapsto\log X$ as an isometry from $(S_{++}^n,g^{LE})$ to $(\text{Sym}(n),\|\cdot\|_F)$, with $g_X^{LE}(\xi,\eta)=\tr\bigl(D_X\log(\xi)^\top D_X\log(\eta)\bigr)$.
In the log-coordinates $Y=\log X$, the metric is flat Euclidean ($G_Y=I$).
\begin{itemize}
\item \emph{\Cref{prop:symmetry}:} trivially satisfied ($G_Y=I$), but invariant only under similarities $X\mapsto cQXQ^\top$ with $c>0$ and $Q\in O(n)$, not $\text{GL}(n)$-invariant.
\item \emph{\Cref{prop:tangent_compat}:} $\mathcal{Z}_Y=\text{Sym}(n)$, SV-invariant. Since $G_Y=I$, the reduction is trivial: no scaling or inversion needed.
\item \emph{Scale/invert:} identity (nothing to compute).
\end{itemize}
The spectral-norm LMO gives $\Delta Y^*=\tau\,\Ortho(\nabla_Y f)$ in log-space, followed by $X_{+}=\exp(Y-\eta\,\Delta Y^*)$.
This is simply standard Euclidean Muon applied to $\log X$; the framework adds nothing beyond reparameterization.
Both propositions hold, but only vacuously. No non-trivial preconditioning or symmetry preservation emerges.

\noindent
\paragraph{Bures--Wasserstein (BW) metric.}
The BW metric~\citep{bhatia2019bures,malago2018wasserstein,han2021riemannian} is
$g_X^{BW}(\xi,\zeta)=\tfrac{1}{2}\tr(L_\xi\zeta)$,
where $L_\xi$ is the unique symmetric solution of the Lyapunov equation $XL_\xi+L_\xi X=\xi$.
The metric operator is $G_X^{BW}=\tfrac{1}{2}\mathcal{L}_X$, where $\mathcal{L}_X(\xi):=L_\xi$.
In the eigenbasis $X=U\Lambda U^\top$ with eigenvalues $\lambda_1,\ldots,\lambda_n$, entry $(i,j)$ of $G_X^{BW}(\xi)$ scales as $\xi_{ij}/(\lambda_i+\lambda_j)$.

\begin{itemize}
\item \emph{\Cref{prop:symmetry} (left-right structure):} the left-right assumption \textbf{fails} because the BW scaling $1/(\lambda_i+\lambda_j)$ does not factor as $f(\lambda_i)\cdot g(\lambda_j)$, so $G_X^{1/2}$ cannot be written as $P\xi Q^\top$ and the proof of \Cref{prop:symmetry} does not apply directly.
However, the \emph{conclusion} of \Cref{prop:symmetry} still holds for the symmetry group $O(n)$ via a direct argument: under $X\mapsto QXQ^\top$ with $Q\in O(n)$, the eigenbasis rotates $U\mapsto QU$ while eigenvalues $\lambda_i$ are unchanged, so $G_{\tilde{X}}^{1/2}(Q\xi Q^\top)=Q\,G_X^{1/2}(\xi)\,Q^\top$.
Since left- and right-multiplication by the same orthogonal $Q$ preserves all unitarily invariant norms, $\varphi(G_{\tilde{X}}^{1/2}\tilde{\xi})=\varphi(G_X^{1/2}\xi)$.
Thus the intrinsic constraint is $O(n)$-invariant, but not $\text{GL}(n)$-invariant (for example, $N=cI$ with $|c|\neq 1$ gives $\varphi(G_{\tilde{X}}^{1/2}(N\xi N^\top))=|c|\,\varphi(G_X^{1/2}\xi)$, because $N\xi N^\top=c^2\xi$ while the Hadamard weights $1/\sqrt{\lambda_i+\lambda_j}$ scale by $1/|c|$).
\item \emph{\Cref{prop:tangent_compat} (SV-invariance):} \textbf{holds.}
The scaled tangent space $\mathcal{Z}_X=\{(G_X^{BW})^{1/2}\xi:\xi\in\text{Sym}(n)\}=\text{Sym}(n)$ (since $(G_X^{BW})^{1/2}$ is an invertible self-adjoint operator on $\text{Sym}(n)$), which is SV-invariant.
Thus the LMO reduces to the unconstrained problem $Z^*=\tau\,\sign(H)$ for the spectral norm, etc.
\item \emph{Scale/invert:} closed-form but requires the eigenbasis.
In the eigenbasis $X=U\Lambda U^\top$, the operators $(G_X^{BW})^{1/2}$ and $(G_X^{BW})^{-1/2}$ act entrywise on $\hat{\xi}:=U^\top\xi U$ as
\[
\bigl[U^\top(G_X^{BW})^{1/2}(\xi)\,U\bigr]_{ij} = \frac{\hat{\xi}_{ij}}{\sqrt{2(\lambda_i+\lambda_j)}},
\qquad
\bigl[U^\top(G_X^{BW})^{-1/2}(Z)\,U\bigr]_{ij} = \sqrt{2(\lambda_i+\lambda_j)}\;\hat{Z}_{ij},
\]
where $\hat{Z}=U^\top Z U$.
Thus the full pipeline is:
(i)~diagonalize $X=U\Lambda U^\top$ once ($O(n^3)$);
(ii)~scale: $H_{ij}=\hat{G}_{ij}/\sqrt{2(\lambda_i+\lambda_j)}$ where $\hat{G}=U^\top(\text{grad}\,f)U$;
(iii)~solve: $Z^*=\tau\,\sign(H)$ via eigendecomposition of $H$;
(iv)~invert: $\hat{\xi}^*_{ij}=\sqrt{2(\lambda_i+\lambda_j)}\,Z^*_{ij}$, then $\xi^*=U\hat{\xi}^*U^\top$.
All steps are $O(n^3)$, the same asymptotic cost as the AI case.
\end{itemize}

The key difference from AI is therefore not computational cost but geometric:
(a)~the inversion formula is a Hadamard (entrywise) rescaling in the eigenbasis rather than a coordinate-free matrix sandwich $X^{1/2}(\cdot)X^{1/2}$, and
(b)~the symmetry group is $O(n)$ rather than $\text{GL}(n)$, so the intrinsic constraint is invariant under a smaller group.

\noindent
\paragraph{Summary.}
We summarize the status of the two key structural properties across the three metrics:

\begin{center}
\small
\begin{tabular}{lccc}
\toprule
& \textbf{Affine-invariant} & \textbf{Log-Euclidean} & \textbf{Bures--Wasserstein} \\
\midrule
Symmetry group & $\text{GL}(n)$ & $O(n)$ and scalings & $O(n)$ \\
\Cref{prop:symmetry} (left-right) & \checkmark & trivial ($G=I$) & \texttimes\ (but $O(n)$-inv.\ holds) \\
\Cref{prop:tangent_compat} (SV-inv.) & \checkmark & \checkmark & \checkmark \\
Scale/invert formula & $X^{1/2}(\cdot)X^{1/2}$ & identity & eigenbasis Hadamard \\
Cost & $O(n^3)$ & $O(n^3)$ & $O(n^3)$ \\
\bottomrule
\end{tabular}
\end{center}

\end{document}